\documentclass{article} 
\usepackage{iclr2026_conference,times}

\usepackage[nospace,noadjust]{cite}
\usepackage{amsthm,amsmath,amssymb,amsfonts,times}
\usepackage{graphicx}
\usepackage{textcomp}
\usepackage{xcolor}
\usepackage{epsfig}
\usepackage{multirow}
\usepackage{caption}
\usepackage{subcaption}
\usepackage{algorithm}
\usepackage{algorithmicx}
\usepackage{algpseudocode}
\usepackage{array}
\usepackage{epstopdf}
\usepackage{color}
\usepackage{diagbox}
\usepackage{bm}
\usepackage{bbm}
\usepackage{mathrsfs}
\usepackage{tcolorbox}
\usepackage{pdfsync}

\usepackage{url}
\usepackage[toc, page]{appendix}

\usepackage{fancyhdr}
\usepackage{cite}
\usepackage{cleveref}

\newtheorem{lemma}{Lemma}
\newtheorem{cor}{Corollary}
\newtheorem{definition}{Definition}
\newtheorem{assum}{Assumption}
\newtheorem{theorem}{Theorem}

\newcommand{\R}{\mathbb{R}}

\newcommand{\e}{\begin{equation}}
\newcommand{\ee}{\end{equation}}
\newcommand{\en}{\begin{equation*}}
\newcommand{\een}{\end{equation*}}
\newcommand{\eqn}{\begin{eqnarray}}
\newcommand{\eeqn}{\end{eqnarray}}
\newcommand{\bmat}{\begin{bmatrix}}
\newcommand{\emat}{\end{bmatrix}}

\DeclareMathAlphabet\mathbfcal{OMS}{cmsy}{b}{n}

\newcommand{\E}{\operatorname{\mathbb{E}}}

\newcommand{\vct}[1]{\boldsymbol{#1}}
\newcommand{\mtx}[1]{\boldsymbol{#1}}

\newcommand{\<}{\langle}
\renewcommand{\>}{\rangle}

\newcommand{\trace}{\operatorname{trace}}

\def \vec       {\operatorname*{vec}}
\newcommand{\wh}{\widehat}

\newcommand{\wt}{\widetilde}
\newcommand{\ol}{\overline}

\newcommand{\calD}{\mathcal{D}}

\newcommand{\calF}{\mathcal{F}}

\newcommand{\calM}{\mathcal{M}}
\newcommand{\calN}{\mathcal{N}}

\newcommand{\calX}{\mathcal{X}}

\newcommand{\ve}{\vct{e}}

\newcommand{\vk}{\vct{k}}

\newcommand{\vo}{\vct{o}}

\newcommand{\vq}{\vct{q}}

\newcommand{\vu}{\vct{u}}
\newcommand{\vv}{\vct{v}}
\newcommand{\vw}{\vct{w}}
\newcommand{\vx}{\vct{x}}
\newcommand{\vy}{\vct{y}}
\newcommand{\vz}{\vct{z}}

\newcommand{\vtheta}{\vct{\theta}}

\newcommand{\mA}{\mtx{A}}
\newcommand{\mB}{\mtx{B}}

\newcommand{\mH}{\mtx{H}}

\newcommand{\mS}{\mtx{S}}
\newcommand{\mT}{\mtx{T}}
\newcommand{\mU}{\mtx{U}}

\newcommand{\mW}{\mtx{W}}
\newcommand{\mX}{\mtx{X}}
\newcommand{\mY}{\mtx{Y}}
\newcommand{\mZ}{\mtx{Z}}

\newcommand{\mLambda}{\mtx{\Lambda}}

\newcommand{\mTheta}{\mtx{\Theta}}

\newcommand{\mId}{{\bf I}}

\graphicspath{{./figs/}}

\newlength{\imgwidth}
\newboolean{twoColVersion}
\setboolean{twoColVersion}{false}
\newcommand{\twoCol}[2]{\ifthenelse{\boolean{twoColVersion}} {#1} {#2} }

\usepackage{enumerate}
\usepackage{enumitem}

\usepackage{mathtools}

\title{Learning to Adapt: In-Context Learning \\ Beyond Stationarity}

\author{
Zhen Qin$^{\dagger,\ddagger,}$\thanks{Corresponding Author: zhenqin@umich.edu} \ ,
Jiachen Jiang$^{\ddagger}$,
Zhihui Zhu$^{\ddagger}$ \\[0.5em]
$^{\dagger}$Michigan Institute for Computational Discovery and Engineering, \\
\hspace*{0.4em}Department of Electrical Engineering and Computer Science, \\
\hspace*{0.4em}Department of Statistics, University of Michigan \\
$^{\ddagger}$Department of Computer Science and Engineering, The Ohio State University
}

\iclrfinalcopy 
\begin{document}

\maketitle

\begin{abstract}
Transformer models have become foundational across a wide range of scientific and engineering domains due to their strong empirical performance. A key capability underlying their success is in-context learning (ICL): when presented with a short prompt from an unseen task, transformers can perform per-token and next-token predictions without any parameter updates. Recent theoretical efforts have begun to uncover the mechanisms behind this phenomenon, particularly in supervised regression settings. However, these analyses predominantly assume stationary task distributions, which overlook a broad class of real-world scenarios where the target function varies over time. In this work, we bridge this gap by providing a theoretical analysis of ICL under non-stationary regression problems. We study how the gated linear attention (GLA) mechanism adapts to evolving input-output relationships and rigorously characterize its advantages over standard linear attention in this dynamic setting. To model non-stationarity, we adopt a first-order autoregressive process and show that GLA achieves lower training and testing errors by adaptively modulating the influence of past inputs---effectively implementing a learnable recency bias. Our theoretical findings are further supported by empirical results, which validate the benefits of gating mechanisms in non-stationary ICL tasks.
\end{abstract}

\vspace{-0.7cm}
\section{Introduction}
\vspace{-0.3cm}

Transformer-based architectures \citep{vaswani2017attention} have emerged as a powerful and versatile modeling framework, achieving state-of-the-art results across a wide spectrum of scientific and engineering domains. Their remarkable effectiveness has been demonstrated in natural language processing \citep{radford2019language,brown2020language}, recommendation systems \citep{zhou2018deep,chen2019behavior}, reinforcement learning \citep{chen2021decision,janner2021offline}, computer vision \citep{dosovitskiy2020image}, and multi-modal signal processing \citep{tsai2019multimodal}, as well as in more specialized areas such as quantum information \citep{ma2025tomography} and wireless communication systems \citep{kim2023transformer}. A particularly notable instance is their pivotal role in the development of large language models like GPT-4 \citep{achiam2023gpt}, where the Transformer backbone enables highly advanced generative capabilities.

A distinctive and increasingly studied feature of Transformer models is in-context learning (ICL) \citep{min2021metaicl}, which allows the model to perform previously unseen tasks at inference time by conditioning on sequences of input-output examples, without requiring any explicit parameter updates. This emergent capability has spurred a growing body of research aiming to understand the underlying mechanisms that enable such behavior \citep{brown2020language,min2021metaicl,dong2022survey,wies2023learnability,zhang2023makes,bai2023transformers,li2024long,bertsch2024context,akyurek2024context,jiang2025compression,
song2024unraveling,wu2023convergence,qin2025convergence}. In particular, recent theoretical works have investigated the realization of ICL in supervised regression settings, showing that certain architectural components--such as linear attention mechanisms--can effectively emulate simple learning algorithms, e.g., a single step of gradient descent, when the input data distribution is stationary \citep{garg2022can,akyurek2022learning, von2023transformers,zhang2024trained, huang2023context,chen2024training, yang2024context, zhang2025training,mahankali2023one,ahn2023transformers,li2024fine, li2025gating, fu2024transformers,dingcausallm}. These findings offer valuable insights into the algorithmic behaviors implicitly encoded by architectural design, shedding light on the interplay between representation, memory, and adaptation in modern Transformer models.

However, much of the existing theoretical understanding is limited to stationary data settings, where the input-output relationships remain consistent across in-context examples and the query point. In contrast, many practical scenarios--including time-series forecasting, streaming data, and natural language--exhibit non-stationarity, where the underlying target function evolves over time. In such settings, recency bias, or the increased predictive relevance of more recent examples, plays a crucial role in accurate prediction. Empirically, linear attention mechanisms are often insufficient for these non-stationary tasks, motivating the introduction of architectural variants that incorporate inductive biases better suited for adaptation, such as gated linear attention (GLA) \citep{yang2023gated,jiang2025context}, RetNet \citep{sun2023retentive},  Gateloop \citep{katsch2023gateloop}, RWKV-6 \citep{peng2024eagle}, as well as state-space models like Mamba-2 \citep{gu2023mamba}. These methods have achieved strong performance in non-stationary sequence modeling, yet there remains a lack of formal theoretical understanding of their behavior in ICL settings.

{\bf Contribution: In this paper, we aim to bridge this gap by presenting a theoretical analysis of ICL in non-stationary or time-varying regression problems.} We investigate how the GLA mechanism adapts to evolving input-output relationships and provide a rigorous characterization of its advantages over standard linear attention in this setting. To model non-stationarity, we adopt a first-order autoregressive process, which allows us to analytically capture temporal variations in the regression targets. Within this framework, we show that standard linear attention incurs higher training and testing errors due to its limited capacity to adapt to distributional shifts over time. In contrast, GLA exhibits inherent adaptability by dynamically modulating the contributions of past inputs, effectively inducing a learnable recency bias. This gating mechanism enables the model to better accommodate time-varying input-output mappings, thereby achieving more robust in-context generalization. Our analysis underscores the importance of architectural components--particularly gating--in equipping transformer models with the ability to implement adaptive learning algorithms in non-stationary environments. Experimental results further corroborate our theoretical findings. Collectively, our work contributes a theoretical perspective that clarifies the design choices behind transformer variants and offers a conceptual framework for understanding and developing architectures suited for adaptive ICL.

\vspace{-0.2cm}

\paragraph{Notation} We use  bold capital letters (e.g., $\mY$) to denote matrices,  bold lowercase letters (e.g., $\vy$) to denote vectors, and italic letters (e.g., $y$) to denote scalar quantities.  Elements of matrices are denoted in parentheses, as in Matlab notation. For example, $\mY(s_1, s_2)$ denotes the element in position
$(s_1, s_2)$ of the matrix $\mY$.
The inner product of $\mA\in\R^{d_1\times d_2}$ and $\mB\in\R^{d_1\times d_2}$ can be denoted as $\<\mA, \mB \> = \sum_{s_1=1}^{d_1}\sum_{s_2=1}^{d_2} \mA(s_1,s_2)\mB(s_1,s_2) $.
$\|\mX\|_F$  represents the Frobenius norm of $\mX$. ${\bm 0}_{d}$ and ${\bm 0}_{d\times d}$ denote the zero vector in $\R^d$ and the zero matrix in $\R^{d\times d}$, respectively. For a positive integer $K$, $[K]$ denotes the set $\{1,\dots, K \}$.

\vspace{-0.3cm}

\subsection{Related Works}
\label{sec: related works}
\vspace{-0.2cm}

A growing body of work has investigated the emergent phenomenon of ICL, with a focus on understanding its behavior in stationary regression tasks. For example, \citep{garg2022can} empirically demonstrated the ICL capabilities of transformers by analyzing prompts where each input is labeled by a task-specific function drawn from a predefined function class, such as linear models. Along similar lines, \citep{akyurek2022learning} investigated linear regression and introduced a transformer construction capable of performing a single gradient descent (GD) step using in-context examples. Building upon this, \citep{von2023transformers} designed weight matrices for linear attention-only transformers that replicate GD updates in linear regression tasks, and notably, they observed that the learned weights resemble those obtained through end-to-end training on ICL prompts.

Further progress has been made by studying the convergence behavior of transformer architectures. In particular, \citep{zhang2024trained} showed that, for a single-layer linear self-attention model, gradient flow with carefully chosen random initialization converges to a global minimum, yielding low prediction error on anisotropic Gaussian data. Complementary work by \citep{huang2023context} initiated the theoretical study of softmax attention, analyzing the training dynamics of one-layer, single-head transformers and providing convergence guarantees for linear regression. This line of research was subsequently extended by \citep{chen2024training, yang2024context, zhang2025training}, who provided sufficient conditions for the convergence of multi-head softmax transformers trained with GD in ICL scenarios. Alternative theoretical perspectives have also been explored: for instance, \citep{mahankali2023one} demonstrated that a transformer performing a single GD step on a least-squares objective can serve as a global minimizer of the pre-training loss, offering a different interpretation of training objectives in ICL. Similarly, \citep{ahn2023transformers} showed that a single-layer model, when trained on random linear regression tasks, implicitly learns to perform a preconditioned GD step at test time, further reinforcing the connection between ICL and optimization-based learning rules. Meanwhile, \citep{li2024fine, li2025gating} offered a theoretical interpretation of GLA through the lens of weighted preconditioned GD, although their analysis remains limited to stationary regression settings. Beyond first-order methods, more advanced optimization techniques have also been considered; for example, \citep{fu2024transformers} analyzed the convergence behavior of second-order methods in ICL, highlighting their potential for accelerated adaptation relative to first-order approaches.

\vspace{-0.3cm}

\section{In-context Learning Time-varying Functions}

\vspace{-0.2cm}

This work builds upon the well-established in-context learning (ICL) framework introduced in \citep{garg2022can}, which aims to train models capable of performing ICL within a specified function class.  As discussed in prior work, significant efforts have been devoted to elucidating the mechanisms underlying ICL. In particular, a number of studies~\citep{garg2022can,akyurek2022learning,mahankalione,ahn2023transformers,huang2024context,zhang2024trained,li2024fine,li2025gating,zhang2025training} have investigated the dynamics of ICL in transformer architectures through the lens of linear regression tasks, where the target function is typically assumed to take the form $f(\vx) = \< \vw, \vx \>$. However, these studies commonly rely on the simplifying assumption that the regression weight vector $\vw$ remains fixed throughout the task. This stationarity assumption creates a theoretical-practical gap, as it does not faithfully reflect real-world scenarios in which data distributions are often nonstationary and the underlying regression weights may vary across different input samples.

{\bf In-context Learning Time-varying Functions} To bridge this gap and advance the theoretical understanding of ICL in non-stationary settings, we introduce a more realistic framework in which the labels in the training prompt are generated by time-varying functions. Formally, let $\calD_{\calX}$ denote a distribution over inputs and $\calD_{\calF_i}$ a time-varying distribution over functions in $\calF_i$. A prompt $P$ is defined as a sequence $(\vx_1, f_1(\vx_1), \dots, \vx_n, f_n(\vx_n), \vx_{\text{query}})$, where the inputs $\vx_1, \ldots, \vx_n\in \R^{d}$ and query $\vx_{\text{query}} = \vx_{n+1}\in \R^{d}$ are drawn from $\calD_{\calX}$, and each $f_i$ is drawn from $\calD_{\calF_i}$. One may consider two canonical types of time-varying functions inspired by the literature:
\begin{itemize}
[leftmargin=12pt,itemsep=2pt,topsep=0pt,parsep=0pt]
\item {\it Deterministic time-varying functions}: Here, $f_i = f(\cdot, i/(n+1))$, where $f$ is assumed to vary smoothly over rescaled time. This setting captures gradual and predictable evolution in the underlying mapping, as extensively studied in time-varying nonlinear regression models \citep{zhang2012inference, zhang2015time}. \item {\it Stochastic time-varying functions}: In this case, the evolution of $f_i$ is modeled as a stochastic process, allowing for random fluctuations in the function mapping. A representative model is $f_i(x) = \gamma f_{i-1}(x) + e_i(x)$, where $0<\gamma<1$ is a forgetting factor modeling gradual drift in task mappings and $e_i(x)$ is a zero-mean stochastic perturbation.
\end{itemize}

\begin{definition}
We say that a model $\calM$ can {in-context learn} the time-varying function class $\calF_i$ up to accuracy $\epsilon$, with respect to $(\calF_i, \calD_{\calX})$, if it can predict $f_{n+1}(\vx_{\text{query}})$ based on the prompt $P$ with average error
\begin{equation}
\E_{P}\big[\ell(\calM(P), f_{n+1}(\vx_{\text{query}}))\big] \leq \epsilon,
\label{eq:icl_def}
\end{equation}
where $\ell(\cdot,\cdot)$ denotes an appropriate loss function, such as squared error.
\end{definition}
Within this framework, we then pose the following central question:

\centerline{\textbf{Question:} Can we train a model to in-context learn a given time-varying function class?}


In this work, to facilitate theoretical analysis while preserving non-stationarity, we consider a simple yet expressive instantiation of the function class:
\begin{eqnarray}
    \label{definition of y}
    y_i = f_i(\vx_i) = \<\vw_i,  \vx_i \>\in\R, i \in[n+1],
\end{eqnarray}
where each weight vector $\vw_i$ evolves according to a first-order autoregressive  process given by
\begin{eqnarray}
    \label{definition of w}
    \vw_i = \gamma \vw_{i-1} + \ve_i, i \in[n+1].
\end{eqnarray}
Here, $0< \gamma \le 1 $ is the autoregressive coefficient that controls the temporal correlation of the weight vectors, the sequence ${\vw_i}$ follows a random walk model, which is a widely adopted generative model in signal processing and adaptive filtering literature~\citep{sayed2011adaptive}. To facilitate tractable analysis, we further assume that the initial weight vector is drawn i.i.d. as $\vw_0 \overset{\text{i.i.d.}}{\sim} \calN({\bm 0}, \sigma_w^2\mId)$, the noisy terms are i.i.d. Gaussian with\footnote{One may assume $\sigma_e^2 = (1 - \gamma^2)\sigma_w^2$ to ensure $\E[\|\vw_i\|^2] = \E[\|\vw_{i-1}\|^2]$. In this work, we relax this constraint to allow more general settings.} $\ve_i \overset{\text{i.i.d.}}{\sim} \calN({\bm 0}, \sigma_e^2\mId)$, and the input vectors are i.i.d. samples from a zero-mean Gaussian distribution with covariance matrix $\vx_i \overset{\text{i.i.d.}}{\sim} \calN({\bm 0}, {\bm \Lambda})$. Moreover, we assume that the random variables ${\vw_{i-1}}$, ${\ve_i}$, and ${\vx_i}$ are mutually independent. Following a long line of theoretical work on in-context learning \citep{mahankalione,ahn2023transformers,zhang2024trained,chen2024training,yang2024context,li2024fine,li2025gating,zhang2025training}, we adopt Gaussian assumptions in our analysis. This modeling choice enables sharp and explicit characterizations of both the training and test errors--rather than only providing loose upper bounds--and is therefore essential for isolating how key quantities such as $\gamma$ govern the behavior of the learned in-context learner.


\paragraph{Gated Linear Attention}

In the non-stationary regression setting introduced above, where the underlying task weights evolve gradually over time, it is crucial for the model to effectively capture pairwise correlations while adapting to the dynamics of changing tasks. Although standard linear attention mechanisms offer computational efficiency and scalability, they lack the flexibility to modulate the influence of prior context based on its relevance to the current input--an ability that is particularly important in nonstationary environments.

To address this limitation, we employ Gated Linear Attention (GLA) \citep{yang2023gated,li2024fine,li2025gating}, which enhances linear attention by introducing a gating mechanism that controls the flow of past information. This structure enables the model to selectively integrate relevant historical patterns while suppressing outdated ones, thereby offering a better inductive bias for capturing evolving structures in non-stationary tasks.

Formally, we consider the following implementation of GLA. Let $\mW_Q \in\R^{(d+1) \times (d+1)}$, $\mW_K\in\R^{(d+1) \times (d+1)}$, and $\mW_V\in\R^{(d+1)\times (d+1)}$ denote the query, key, and value weight matrices, respectively. To streamline the subsequent analysis, we follow prior works~\citep{ ahn2023transformers,huang2024context,zhang2024trained,li2024fine,li2025gating,zhang2025training} and construct the prompt by evaluating each function $f_i$ on the sampled inputs and pairing each input with its corresponding output:
\begin{eqnarray}
    \label{definition of input prompt}
    \mZ = \begin{bmatrix}\vz_1 & \cdots & \vz_n & \vz_{n+1} \end{bmatrix}= \begin{bmatrix}\vx_1 & \cdots & \vx_n & \vx_{n+1} \\ y_1 & \cdots & y_n & 0 \end{bmatrix}\in\R^{(d+1)\times (n+1)}.
\end{eqnarray}
For each input $\vz_i$, we define the corresponding query, key, and value vectors as $\vq_i = \mW_Q \vz_i$, $\vk_i = \mW_K \vz_i$ and $\vv_i = \mW_V \vz_i$. The output of GLA at position $i$ is given by:
\begin{eqnarray}
    \label{out of GLA1}
    \vo_{i} = \mS_{i}\vq_i \ \  \text{and} \ \  \mS_{i} = \lambda \mS_{i-1} + \vv_i\vk_i^\top,
\end{eqnarray}
where $\lambda \in (0,1]$ is a forgetting factor that determines how quickly the attention mechanism discounts earlier information. For ease of theoretical analysis, we adopt a simplified formulation where a single global forgetting factor $\lambda$ is used, rather than assigning a separate, data-dependent gating coefficient to each token as done in the original GLA model. By unrolling the recursive update in \eqref{out of GLA1}, we obtain:
\begin{eqnarray}
    \label{out of GLA2}
    \mS_{n+1} = \lambda \mS_{n} + \vv_{n+1}\vk_{n+1}^\top = \sum_{i=1}^{n+1} \lambda^{n+1-i} \vv_i\vk_i^\top = \mW_V\bigg(\sum_{i=1}^{n+1} \lambda^{n+1-i} \vz_i\vz_i^\top \bigg)\mW_K^\top,
\end{eqnarray}
which leads to the following expression for the output vector:
\begin{eqnarray}
    \label{out of GLA3}
    \vo_{n+1} = \mS_{n+1}\vq_{n+1} = \mW_V\bigg(\sum_{i=1}^{n+1} \lambda^{n+1-i} \vz_i\vz_i^\top \bigg)\mW_K^\top \mW_Q\vz_{n+1}.
\end{eqnarray}
It is worth noting that when $\lambda = 1$, the weighted sum degenerates into an unweighted accumulation, i.e., $\sum_{i=1}^{n+1} \vz_i\vz_i^\top = \mZ\mZ^\top$, under which the GLA formulation reduces to the standard linear attention model. This highlights that GLA generalizes linear attention by introducing a learnable memory decay.

Since the final prediction is taken as the last entry of the token vector output by the GLA layer, only a subset of the entries in the  weight matrices $\mW_{V}$ and $\mW_Q$, $\mW_{K}$ influence the output. To simplify the notation and subsequent analysis, we merge the query and key matrices into a single matrix and define
\begin{eqnarray}
    \label{out of WV and WKQ}
    \mW_V = \begin{bmatrix}\mW_{11}^V &   \vw_{12}^{V} \\ {\vw_{21}^{V}}^\top   & w_{-1}^V \end{bmatrix}\in\R^{(d+1)\times (d+1)} \ \ \text{and} \ \ \mW_{KQ} = \begin{bmatrix}\mW_{11}^{KQ} &   \vw_{12}^{KQ} \\ {\vw_{21}^{KQ}}^\top   & w_{-1}^{KQ} \end{bmatrix}\in\R^{(d+1)\times (d+1)},
\end{eqnarray}
where $\mW_{11}^V, \mW_{11}^{KQ}\in\R^{d\times d}$, $\vw_{12}^{V}, \vw_{21}^{V}, \vw_{12}^{KQ}, \vw_{21}^{KQ}\in\R^{d\times 1}$ and $w_{-1}^V, w_{-1}^{KQ}\in\R$. Using this decomposition, we express the predicted output as
\begin{eqnarray}
    \label{output of y}
    \wh y_{n+1} = \vo_{n+1}(d+1) = \begin{bmatrix} {\vw_{21}^{V}}^\top   & w_{-1}^V \end{bmatrix} \bigg(\sum_{i=1}^{n+1} \lambda^{n+1-i} \vz_i\vz_i^\top \bigg) \begin{bmatrix}\mW_{11}^{KQ}  \\ {\vw_{21}^{KQ}}^\top   \end{bmatrix} \vx_{n+1}.
\end{eqnarray}
Note that only the last row of $\mW_V$ and the first $d$ columns of $\mW_{KQ}$ contribute to the final prediction. Therefore, without loss of generality, we may set the remaining entries in $\mW_V$ and $\mW_{KQ}$ to zero in the subsequent analysis.

\vspace{-0.3cm}

\section{Theoretical Analysis of GLA for Time-varying Regression}

\vspace{-0.2cm}

In this work, we investigate the convergence behavior, training error, and testing error of ICL linear predictors based on the GLA model for time-varying functions. Suppose we are given $B$ independent in-context learning training tasks, where each task prompt corresponds to an embedding matrix $\mZ_{\tau}$, for $\tau = 1,\dots, B$, constructed according to the transformation defined in~\eqref{definition of input prompt}:
\begin{eqnarray}
    \label{definition of input prompt sample}
    \mZ_{\tau} = \begin{bmatrix}\vz_{\tau,1} & \cdots & \vz_{\tau,n} & \vz_{\tau,n+1} \end{bmatrix}=\begin{bmatrix}\vx_{\tau,1} & \cdots & \vx_{\tau,n} & \vx_{\tau,n+1} \\ \<\vw_{\tau,1}, \vx_{\tau,1}  \> & \cdots & \<\vw_{\tau,n}, \vx_{\tau,n}  \> & 0 \end{bmatrix},
\end{eqnarray}
where the weight vectors $\vw_{\tau,i}$ evolves according to \eqref{definition of w}.

We denote the prediction produced by the GLA model on the query input of task $\tau$ as $\wh y_{\tau, n+1}$, whose exact form is given in~\eqref{output of y}. The empirical risk over $B$ independent task prompts is then defined as:
\begin{eqnarray}
    \label{empirical loss function}
    l(\vtheta) = \frac{1}{2B}\sum_{\tau = 1}^{B}\big( \wh y_{\tau, n+1} - \<\vw_{\tau, n+1}, \vx_{\tau, n+1}  \> \big)^2,
\end{eqnarray}
where the model parameters are denoted by $\vtheta = \{\mW_{KQ}, \mW_{V} \}$. To analyze the learning dynamics, we consider the population risk induced in the limit as the number of training prompts tends to infinity, i.e., $B \to \infty$:
\begin{eqnarray}
    \label{definition of loss function}
    L(\vtheta) = \lim_{B\to\infty}l(\vtheta) =  \frac{1}{2}\E_{\vw_{n+1}, \vx_{n+1}}[ (\wh y_{n+1} - \<\vw_{n+1}, \vx_{n+1}  \>)^2 ],
\end{eqnarray}
where we omit the task index $\tau$ for notational simplicity.

We study the evolution of the model parameters under gradient flow, which characterizes the continuous-time limit of gradient descent with infinitesimal step sizes. The parameter dynamics are governed by the ordinary differential equation $\frac{\text{d}\vtheta}{\text{d}t}  = -\nabla L(\vtheta)$. To facilitate the analysis, following the approach of \citep{zhang2024trained}, we introduce an initialization scheme that satisfies the following assumption.
\begin{assum} (Initialization)
\label{Assumption of initialization}
Let $\sigma>0$ be a parameter and $\mTheta \in\R^{d\times d}$ be any matrix satisfying $\|\mTheta\mTheta^\top\|_F=1$ and $\mLambda\mTheta \neq {\bm 0}_{d\times d}\in\R^{d\times d}$. We assume
\begin{eqnarray}
    \label{out of WV initialization}
    \mW_V(0) = \sigma\begin{bmatrix}{\bm 0}_{d\times d} &   {\bm 0}_{d} \\ {\bm 0}_{d}^\top   & 1 \end{bmatrix}\in\R^{(d+1)\times (d+1)} \ \ \text{and} \ \ \mW_{KQ}(0) = \sigma\begin{bmatrix}\mTheta\mTheta^\top &   {\bm 0}_{d} \\ {\bm 0}_{d}^\top   & 0 \end{bmatrix}\in\R^{(d+1)\times (d+1)}.
\end{eqnarray}
\end{assum}
Under this setup, the following result establishes that the gradient flow dynamics with respect to the population loss converge to a specific global optimum.
\begin{theorem} (Convergence of gradient flow)
\label{Theorem of global minimum simplified}
Consider gradient flow over the population loss in \eqref{definition of loss function}. Assume that $\gamma<1$, the initial task weight $\vw_0\overset{\text{i.i.d.}}{\sim}\calN({\bm 0}, \sigma_w^2\mId)$, noises $\ve_i\overset{\text{i.i.d.}}{\sim}\calN({\bm 0}, \sigma_e^2\mId)$ and inputs $\vx_i\overset{\text{i.i.d.}}{\sim}\calN({\bm 0}, {\bm \Lambda})$.  Suppose the initialization satisfies {Assumption} \ref{Assumption of initialization}  with
initialization scale $\sigma>0$  satisfying $\sigma < \sqrt{\frac{2D_1}{\sqrt{d}\|\wt\mLambda\|}}$ where
\begin{eqnarray*}
D_1 =\begin{cases}
       \lambda^{2n+2}n\sigma_w^2 + \bigg(\frac{\lambda^2( 1 - \lambda^{2n}) }{(1-\lambda^2)^2} - \frac{\lambda^{2n+2}}{1-\lambda^2}n \bigg)\sigma_e^2, &  \lambda = \gamma,\\
    \frac{\lambda^{n+1}\gamma^{n+2} - \lambda \gamma^{2n+2} }{\lambda - \gamma}\sigma_w^2 +\bigg( \frac{\lambda\gamma( 1 - \lambda^n\gamma^n) }{(1-\gamma^2)(1-\lambda\gamma)} - \frac{\lambda^{n+1}\gamma^{n+2} - \lambda \gamma^{2n+2} }{(\lambda - \gamma)(1-\gamma^2)}  \bigg)\sigma_e^2, &  \lambda \neq \gamma,
\end{cases}
\end{eqnarray*}
and $\wt\mLambda = D_2(2\mLambda +\trace(\mLambda)\mId ) + D_3\mLambda $ with
\begin{eqnarray*}
D_2 =\begin{cases}
     \lambda^{2n+2}n\sigma_w^2 - (\frac{n\lambda^{2n+2}}{1-\lambda^2} - \frac{\lambda^4 - \lambda^{2n+2}}{(1- \lambda^2 )^2}  )\sigma_e^2, &  \lambda = \gamma,\\
    \frac{\gamma^2\lambda^{2n+2} - \lambda^2 \gamma^{2n+2}}{\lambda^2 - \gamma^2}\sigma_w^2 - (\frac{\gamma^2\lambda^{2n+2} - \lambda^2 \gamma^{2n+2}}{(\lambda^2 - \gamma^2)(1-\gamma^2)} - \frac{\lambda^2 - \lambda^{2n+2}}{(1-\gamma^2)(1-\lambda^2)}    )\sigma_e^2, &  \lambda \neq \gamma,
\end{cases}
\end{eqnarray*}
and
{\small\begin{eqnarray*}
D_3 =  \begin{cases}

    \lambda^{2n+2}n(n-1)\sigma_w^2 + (\frac{2n(\lambda^4 - \lambda^{2n+2}) }{(1-\lambda^2)^2} - \frac{2(\lambda^{2n+4}  - n\lambda^6 + (n-1)\lambda^4)}{(1-\lambda^2)^3}
     - \frac{\lambda^{2n+2}n(n-1)}{1-\lambda^2} )\sigma_e^2, &  \lambda = \gamma,\\
    (\frac{2\gamma^3\lambda^{2n+3} - 2\lambda^5\gamma^{2n+1}}{\lambda(\lambda - \gamma)^2(\lambda + \gamma)} - \frac{2\gamma^{n+2}\lambda^{n+2} - 2\gamma^{2n+1}\lambda^3}{\lambda - \gamma}     )\sigma_w^2  +  (\frac{2\gamma^{-1}(\lambda^4 - \lambda^{2n+2})}{(1-\gamma^2)(\lambda - \gamma)(1-  \lambda^2)}\\
    - \frac{2(\lambda^{3} - \lambda^{n+2}\gamma^{n-1})}{(1 - \gamma^2)(\lambda - \gamma)(1 - \lambda\gamma)} - \frac{2\gamma^3\lambda^{2n+3} - 2\lambda^5\gamma^{2n+1}}{\lambda(\lambda - \gamma)^2(\lambda + \gamma)(1 - \gamma^2)} + \frac{2\gamma^{n+2}\lambda^{n+2} - 2\gamma^{2n+1}\lambda^3}{(\lambda - \gamma)(1 - \gamma^2)}    )\sigma_e^2, &  \lambda \neq \gamma.
\end{cases}
\end{eqnarray*}
}
Then gradient flow converges to a global minimum of the population loss \eqref{definition of loss function}. Moreover, $\mW_{KQ}(0)$ and $\mW_V(0)$ respectively converge to
{\small\begin{eqnarray}
    \label{out of WV global minimum simplified}
    \lim_{t\to\infty}\mW_V(t) = \sqrt{D_1\|\wt\mLambda^{-1} \|_F}\begin{bmatrix}{\bm 0}_{d\times d} &   {\bm 0}_{d} \\ {\bm 0}_{d}^\top   & 1 \end{bmatrix} \ \ \text{and} \ \ \lim_{t\to\infty}\mW_{KQ}(t) = \sqrt{D_1\|\wt\mLambda^{-1} \|_F^{-1}}\begin{bmatrix}\wt\mLambda^{-1} &   {\bm 0}_{d} \\ {\bm 0}_{d}^\top   & 0 \end{bmatrix}.
\end{eqnarray}}
\end{theorem}
\vspace{-0.25cm}
The proof is deferred to {Appendix}~\ref{proof of Theorem of global minimum}. Despite the non-stationary nature of the regression model considered in this work, we establish that gradient flow converges to a global minimum even under random initialization. The closed-form solution in \eqref{out of WV global minimum simplified} reveals that the location of the global optimum is explicitly determined by $\lambda$ and $\gamma$, highlighting their structural influence on the solution. While the main theorem focuses on the regime $0 < \lambda \leq 1$ and $0 < \gamma < 1$, a more general result accommodating arbitrary $\lambda > 0$ and $\gamma > 0$ is established in \Cref{Theorem of global minimum} of {Appendix}~\ref{proof of Theorem of global minimum}. Moreover, in the limiting case where $\lambda = \gamma = 1$ and $\sigma_e^2 = 0$, the result reduces precisely to that in \citep[Theorem 4]{zhang2024trained}, thereby recovering the stationary setting as a special case of our more general formulation.

\vspace{-0.2cm}
\paragraph{Training error} We now analyze the training error of the learned network. At the global optimum--i.e., when the parameters converge to $\lim_{t\to\infty}\mW_V(t)$ and $\lim_{t\to\infty}\mW_{KQ}(t)$ in \eqref{out of WV global minimum simplified}, a straightforward calculation yields the prediction $\wh y_{n+1}$ as follows:
\begin{eqnarray}
    \label{output of y global minimum}
    \wh y_{n+1}  \! = \! D_1 \begin{bmatrix} {\bm 0}_{d}^\top   & 1 \end{bmatrix} \bigg(\sum_{i=1}^{n+1} \lambda^{n+1-i} \vz_i\vz_i^\top \bigg) \begin{bmatrix}\wt\mLambda^{-1}  \\ {\bm 0}_{d}^\top   \end{bmatrix} \vx_{n+1}
     =  D_1 \bigg( \sum_{i=1}^n\lambda^{n+1-i} \vw_i^\top\vx_i\vx_i^\top \bigg)\wt\mLambda^{-1} \vx_{n+1}.
\end{eqnarray}
This expression confirms that, for sufficiently long prompts, the trained model successfully in-context learns the family of linear predictors. We emphasize that both $\lambda$ and $\gamma$ jointly influence the degree of time variation in the underlying model. We next quantify the training error at the global optimum.
\begin{theorem} (Training error)
\label{theorem recovery error using global minimum simplified}
Assuming the conditions in \Cref{Theorem of global minimum simplified} hold, the recovery error between \eqref{output of y global minimum} and \eqref{definition of y} is
\begin{eqnarray}
    \label{recovery error final global minimum simplified}
    \E[(\wh y_{n+1} - y_{n+1})^2] &\!\!\!\! = \!\!\!\!& D_1^2\trace\big(D_2(\mLambda\trace(\wt\mLambda^{-1}\mLambda\wt\mLambda^{-1}\mLambda ) + 2\mLambda \wt\mLambda^{-1}\mLambda \wt\mLambda^{-1}\mLambda)\nonumber\\
    &\!\!\!\! \!\!\!\!& + D_3 \mLambda \wt\mLambda^{-1}\mLambda \wt\mLambda^{-1}\mLambda  \big) + D_4\trace(\mLambda) - 2D_1^2\trace(\mLambda \wt\mLambda^{-1}\mLambda ),
\end{eqnarray}
where $D_4 = \gamma^{2n+2}\sigma_w^2 + \frac{1- \gamma^{2n+2}}{1-\gamma^2 }\sigma_e^2$.
\end{theorem}
The proof is provided in {Appendix}~\ref{proof of Theorem of recovery error global minimum}. Equation~\eqref{recovery error final global minimum simplified} illustrates that the training error depends jointly on the parameters $\lambda$ and $\gamma$. Consequently, for fixed $\gamma$, there exists an optimal value of $\lambda$ that minimizes the error. Although the expressions of $D_i$ suggest a symmetric structure in $\lambda$ and $\gamma$, it does not necessarily imply that choosing $\lambda=\gamma$ minimizes the recovery error. In fact, the error involves a subtle balance between the $\sigma_w^2$- and $\sigma_e^2$-dependent terms as well as the trace terms with $\wt\Lambda^{-1}$. When $\lambda=\gamma$, the simplification of $D_i$ may amplify certain noise-dependent factors and deteriorate the overall error. This observation highlights that the optimal choice of $\lambda$ depends not only on the apparent algebraic symmetry but also on the interplay between noise statistics, system dimension, and the spectral structure of $\mLambda$.

We now consider a special case with $\mLambda = \mId$, in which the training error in \eqref{recovery error final global minimum simplified} reduces to $\E[(\wh y_{n+1} - y_{n+1})^2]= dD_4 - \frac{d D_1^2 }{(2+d)D_2+D_3} $. The first term $dD_4$ only depends on the autoregressive process and is independent to the choice of $\lambda$. To study how the second term varies according to $\lambda$, assume $\gamma<1$ is fixed and $d$ is sufficiently large, such that $\frac{d D_1^2 }{(2+d)D_2+D_3} \approx D_1^2/D_2=:\xi(\lambda)$. To further illustrate how $\xi(\lambda)$ might vary with $\lambda$, we consider two particular cases. Case I: $\sigma_e \ll \sigma_w$. In this regime, $D_1$ and $D_2$ are respectively dominated by $\frac{\lambda^{n+1}\gamma^{n+2} - \lambda \gamma^{2n+2} }{\lambda - \gamma}\sigma_w^2$ and $ \frac{\gamma^2\lambda^{2n+2} - \lambda^2 \gamma^{2n+2}}{\lambda^2 - \gamma^2}\sigma_w^2$. Consequently, $\xi(\lambda)$ can be well approximated by $\frac{\gamma^{2n+4}(\lambda^n - \gamma^n)(\lambda + \gamma)}{(\lambda^n + \gamma^n)(\lambda - \gamma)}\sigma_w^2$, which increases monotonically with $\lambda$ on $(0,\gamma)$ and then decreases monotonically on $(\gamma,1)$. Case II: $n$ is sufficiently large such that $D_1$ and $D_2$ are dominated by $ \frac{\lambda\gamma}{(1-\gamma^2)(1-\lambda\gamma)}\sigma_e^2$ and $ \frac{\lambda^2 - \lambda^{2n+2}}{(1-\gamma^2)(1-\lambda^2 )}\sigma_e^2$, respectively. In this case, $\xi(\lambda)\approx\frac{\gamma^2(1-\lambda^2)}{(1-\gamma^2)(1-\lambda\gamma)^2(1 - \lambda^{2n})}\sigma_e^2$, which likewise increases monotonically with $\lambda$ on $(0,\gamma)$ and then decreases monotonically on $(\gamma,1)$. Although these results are derived under simplifying approximations, they suggest that the training loss exhibits an inverse U-shaped dependence on $\lambda$, attaining its minimum at some $\lambda < 1$ rather than at $\lambda = 1$. The subsequent experiments provide direct validation of these theoretical predictions.


\vspace{-0.2cm}
\paragraph{Testing error}  In this part, we characterize the prediction performance of the trained transformer when evaluated on a test prompt drawn from a potentially different task distribution. Notably, the model parameters are fixed at their global optimum obtained from training, and the test prompt may differ in its length, data distribution, and underlying dynamics. We consider test prompts of the form
\begin{eqnarray}
    \label{definition of input prompt sample test}
    \ol\mZ = \begin{bmatrix}\ol\vz_{1} & \cdots & \ol\vz_{m} & \ol\vz_{m+1} \end{bmatrix}&\!\!\!\!=\!\!\!\!& \begin{bmatrix}\ol\vx_1 & \cdots & \ol\vx_m & \ol\vx_{m+1} \\ \ol y_1 & \cdots & \ol y_m & 0 \end{bmatrix} \nonumber\\
    &\!\!\!\!=\!\!\!\!&\begin{bmatrix}\ol\vx_{1} & \cdots & \ol\vx_{m} & \ol\vx_{m+1} \\ \<\ol\vw_{1}, \ol\vx_{1}  \> & \cdots & \<\ol\vw_{m}, \ol\vx_{m}  \> & 0 \end{bmatrix},
\end{eqnarray}
where the latent task weights$\{ \ol\vw_i \}_{i=1}^{m+1}$ evolve according to the first-order autoregressive model $\ol\vw_i = \ol\gamma \cdot \ol\vw_{i-1} + \ol\ve_i, i=1,\dots, m+1$. To distinguish between training and testing distributions, we assume that the initial weight vector satisfies $\ol\vw_0\overset{\text{i.i.d.}}{\sim}\calN({\bm 0}, \ol\sigma_w^2\mId)$, and the driving noise $\ol\ve_i\overset{\text{i.i.d.}}{\sim}\calN({\bm 0}, \ol\sigma_e^2\mId)$. The inputs are drawn independently as $\ol\vx_i\overset{\text{i.i.d.}}{\sim}\calN({\bm 0}, \ol{\bm \Lambda})$, and we assume mutual independence among  random variables $\ol \vw_{i-1}$, $\ol \ve_i$, and $\ol \vx_i$.

Given a forgetting factor $\ol\lambda$, the prediction $\wt y_{m+1}$ produced by the model at test time (evaluated at the training global optimum) is
\begin{eqnarray}
    \label{output of y global minimum test}
    \wt y_{m+1}   =   D_1 \big( \sum_{i=1}^m\ol\lambda^{m+1-i} \ol\vw_i^\top\ol\vx_i\ol\vx_i^\top \big)\wt\mLambda^{-1} \ol\vx_{m+1}.
\end{eqnarray}

\vspace{-0.25cm}

We now characterize the mean squared prediction error on the test prompt:
\begin{theorem} (Testing error)
\label{theorem recovery error using global minimum test}
Under the assumptions in \Cref{Theorem of global minimum}, the expected prediction error of the model on the test prompt is given by
\begin{eqnarray}
    \label{recovery error final global minimum test}
    \hspace{-0.5cm}\E[(\wt y_{m+1} - \ol y_{m+1})^2] &\!\!\!\! = \!\!\!\!& D_1^2\trace\big(\ol D_2(\ol\mLambda\trace(\wt\mLambda^{-1}\ol\mLambda\wt\mLambda^{-1}\ol\mLambda ) + 2\ol\mLambda \wt\mLambda^{-1}\ol\mLambda \wt\mLambda^{-1}\ol\mLambda)\nonumber\\
    &\!\!\!\! \!\!\!\!&\hspace{-0.5cm} + \ol D_3 \ol\mLambda \wt\mLambda^{-1}\ol\mLambda \wt\mLambda^{-1}\ol\mLambda\big) + \ol D_4\trace(\ol\mLambda) - 2D_1\cdot \ol D_1\trace(\ol\mLambda \wt\mLambda^{-1}\ol\mLambda ),
\end{eqnarray}
where $\ol D_i$ for $i = 1,\dots, 4$ are defined analogously to the $D_i$ constants from training, with the substitution $\lambda \to \ol\lambda$, $\gamma \to \ol\gamma$, $\sigma_w^2 \to \ol\sigma_w^2$, $\sigma_e^2 \to \ol\sigma_e^2$, and $n \to m$.
\end{theorem}
The proof has been provided in {Appendix} \ref{proof of Theorem of recovery error global minimum test}. This result quantifies the generalization behavior of the trained model when applied to unseen prompts sampled from a potentially different distribution.   Notably, the prediction error depends jointly on the training and testing task statistics through the interaction between $\wt\mLambda$ and $\ol\mLambda$. Moreover, the expected error $\E[(\wt y_{m+1} - \ol y_{m+1})^2]$ is inherently nonzero due to the stochastic nature of the task evolution--specifically, the noise in the dynamics of $\ol\vw_i$ introduces irreducible uncertainty in the test labels $\ol y_i$. This highlights the importance of employing GLA, which adaptively modulates the influence of past observations and better accommodates temporal variations in the underlying regression weights. In the subsequent experimental section, we empirically demonstrate the effectiveness of the GLA mechanism in handling non-stationary tasks.

\vspace{-0.2cm}
\paragraph{Comparison with Adaptive Signal Processing}
The non-stationary regression setting considered in this paper is closely related to classical problems in adaptive signal processing, where the underlying model parameters evolve gradually over time~\citep{sayed2011adaptive,das2015steady,abdolee2016tracking,qin2020proportionate,claser2021tracking,yu2021tracking,wang2022performance}. To track such non-stationary dynamics, a wide range of online algorithms have been developed, including the least mean squares (LMS) algorithm, the affine projection algorithm (APA), and the recursive least squares (RLS) algorithm. These methods are designed to update model parameters iteratively in response to streaming data, with the goal of minimizing instantaneous or long-term prediction error.  Under non-stationary models such as the first-order autoregressive process described in \eqref{definition of w}, the corresponding theoretical error analyses for these methods also indicate that, for a fixed $\gamma$, there exists an optimal choice of step size (in LMS/APA) or forgetting factor (in RLS) that minimizes the tracking error.

While classical adaptive signal processing methods explicitly update model parameters over time based on streaming observations, the paradigm studied in this paper--in-context learning with the GLA model--adopts a fundamentally different approach. Instead of relying on explicit parameter updates, as in LMS, APA, or RLS, the GLA implicitly adapts to task dynamics via internal representations conditioned on the prompt. In particular, the gating mechanism in GLA enables the model to selectively integrate past information in a soft and differentiable manner, thereby tracking non-stationary structures without modifying its parameters. This architectural distinction offers a new perspective on learning in non-stationary environments, where adaptation arises not from external optimization procedures, but from the model's forward computation itself.

\vspace{-0.3cm}
\section{Experimental Results}
\label{sec:experiment result}

\vspace{-0.3cm}

In this section, we present experiments to validate the theoretical analysis and demonstrate the advantages of GLA in non-stationary models. The experiments are conducted under the following settings. The training and testing losses are defined as $\frac{1}{B}\sum_{\tau=1}^{B}(\wh y_{\tau, n+1} - y_{\tau, n+1})^2$ and $(\wt y_{m+1} - \ol y_{m+1})^2$, respectively. Unless otherwise specified, we set $d=10$, $n=100$, $\sigma_{w}^2=1$, $\sigma_{e}^2=0.01$, and $B=10^7$. The AdamW optimizer is adopted with learning rate $10^{-2}$, weight decay $0.05$, and momentum parameter $0.9$. Each model is trained for $2000$ epochs with a batch size of $5000$ samples. The loss associated with the optimal $\lambda$ is highlighted by a star. Although the theoretical analysis imposes a constraint on the initialization matrix, our experiments use a random Gaussian initialization and still observe the predicted behavior, indicating that the constraint is not necessary in practice.
\begin{figure}[!ht]
\centering
\begin{subfigure}[t]{0.21\textwidth}
    \centering
    \includegraphics[width=3.1cm]{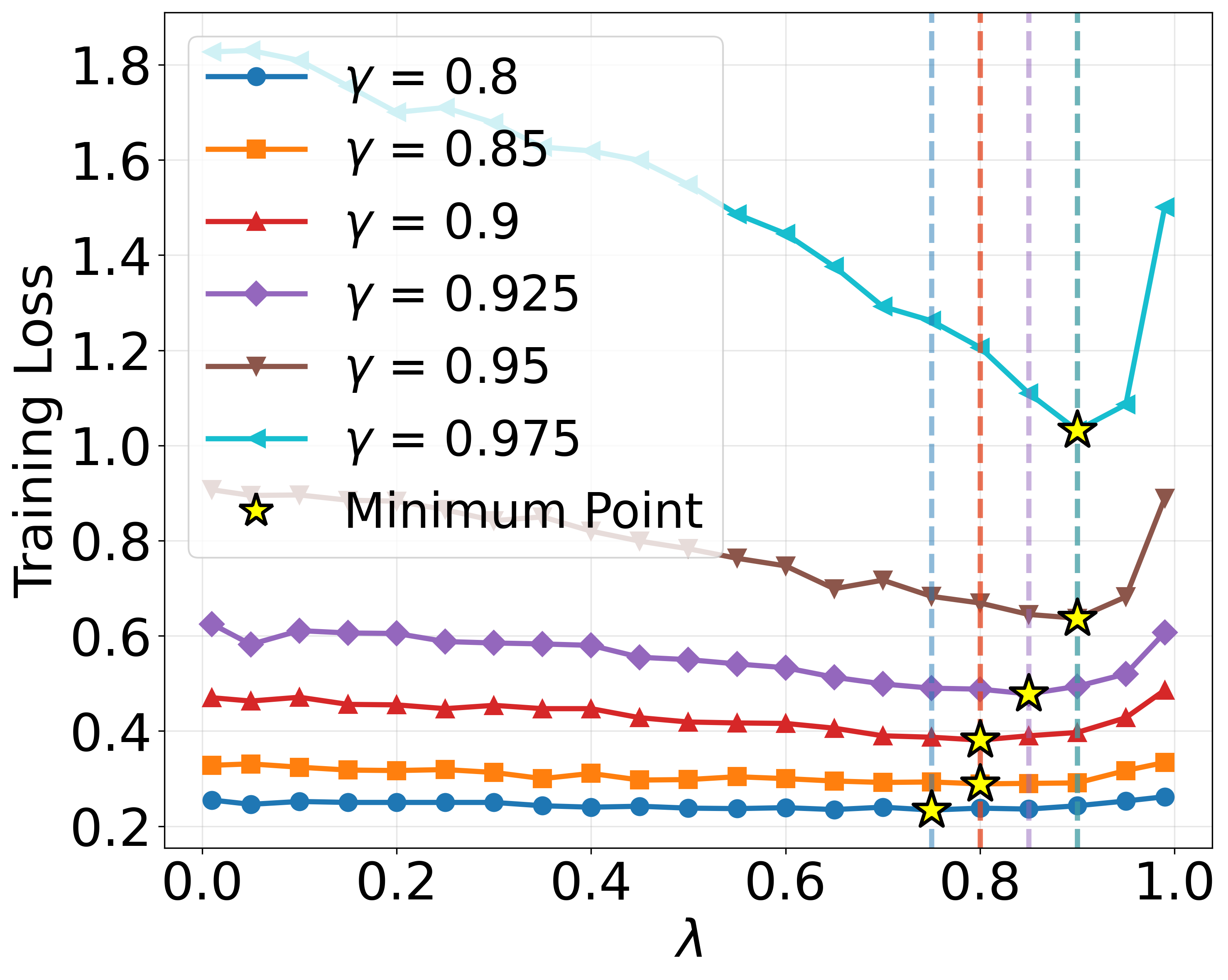}
    \caption{\footnotesize Training loss}
    \label{trainloss}
\end{subfigure}
\begin{subfigure}[t]{0.25\textwidth}
    \centering
    \includegraphics[width=3.1cm]{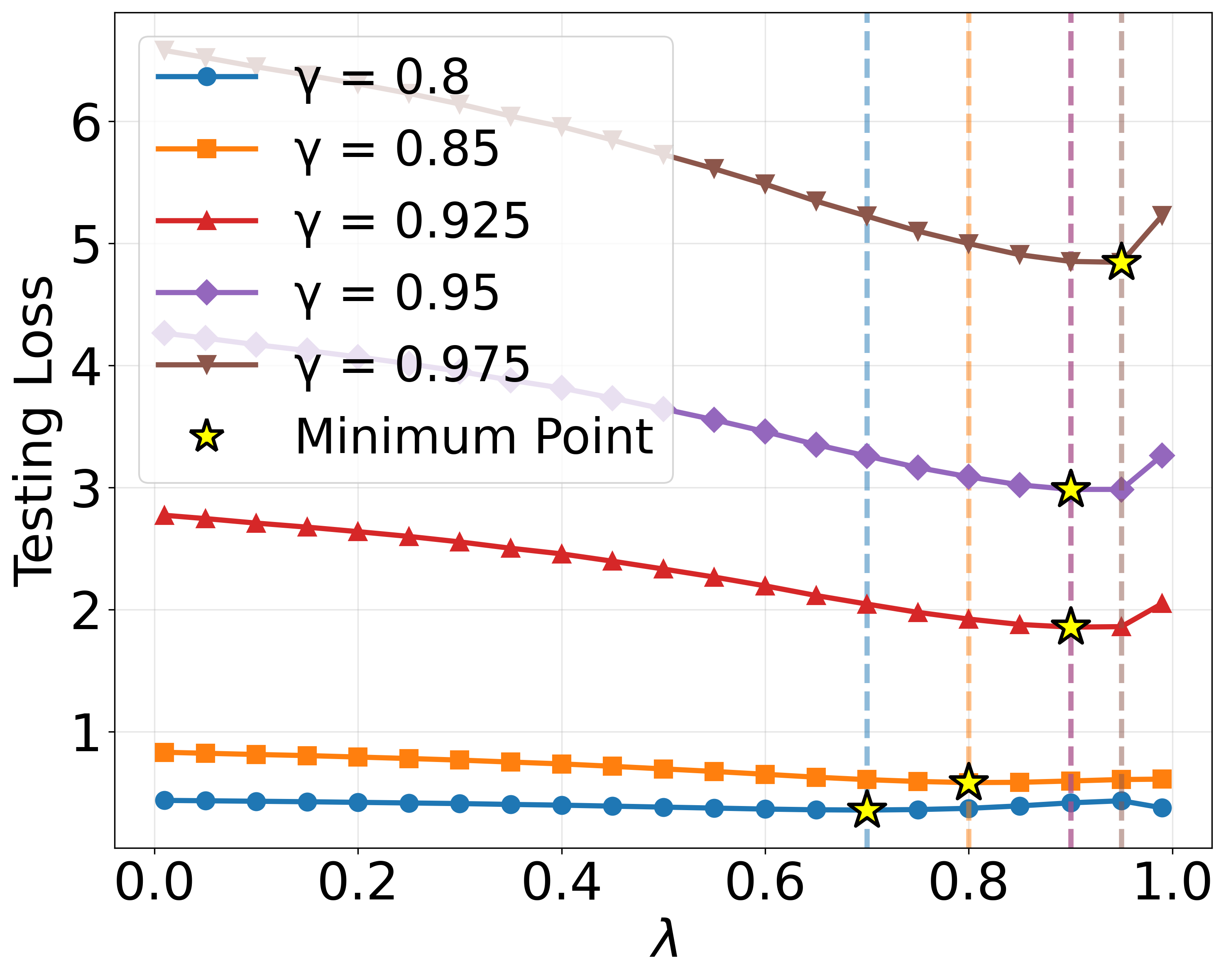}
    \caption{\footnotesize Testing loss ($m = 10$)}
    \label{testloss_m10}
\end{subfigure}
\begin{subfigure}[t]{0.25\textwidth}
    \centering
    \includegraphics[width=3.1cm]{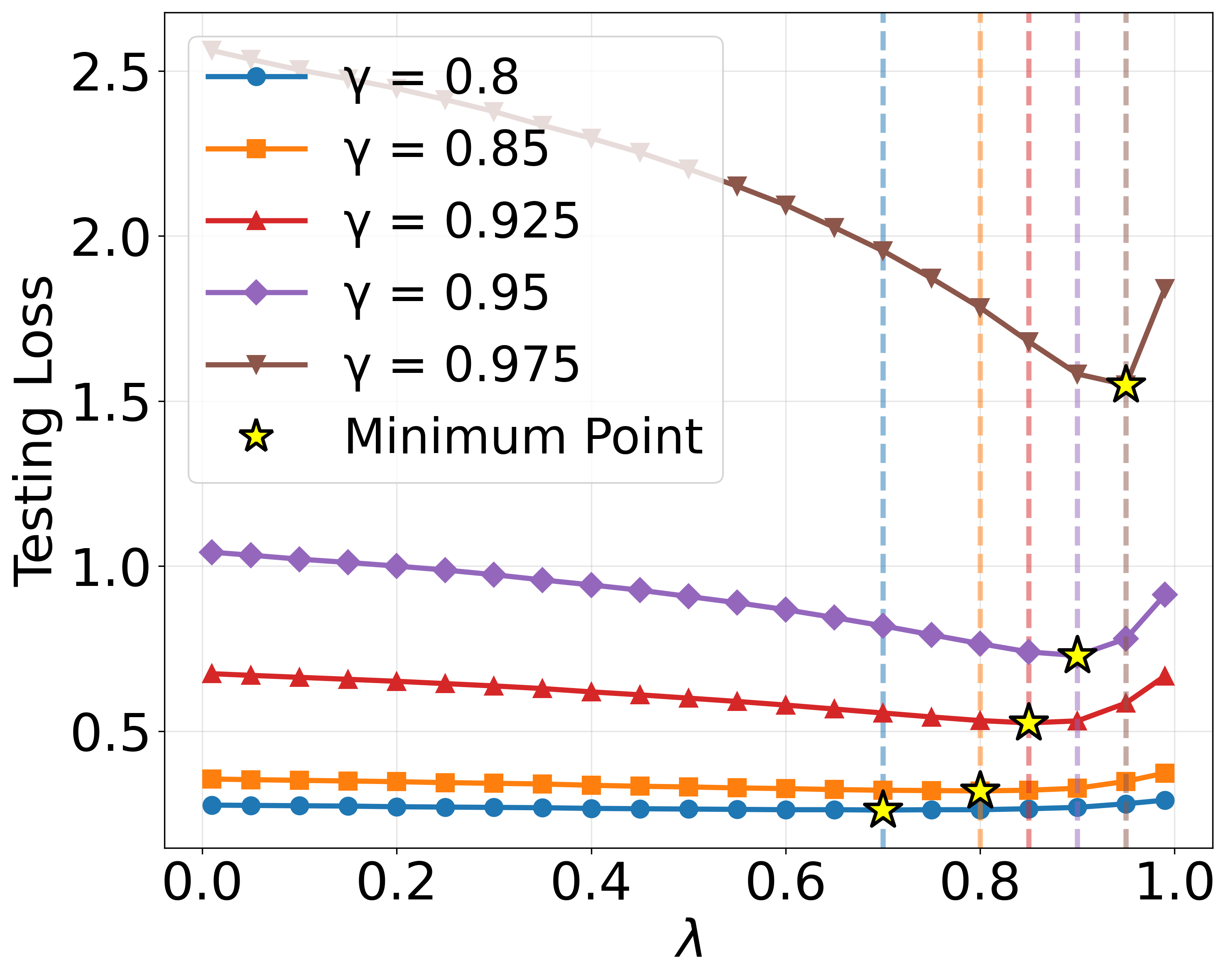}
    \caption{\footnotesize Testing loss ($m = 50$)}
    \label{testloss_m50}
\end{subfigure}
\begin{subfigure}[t]{0.27\textwidth}
    \centering
    \includegraphics[width=3.1cm]{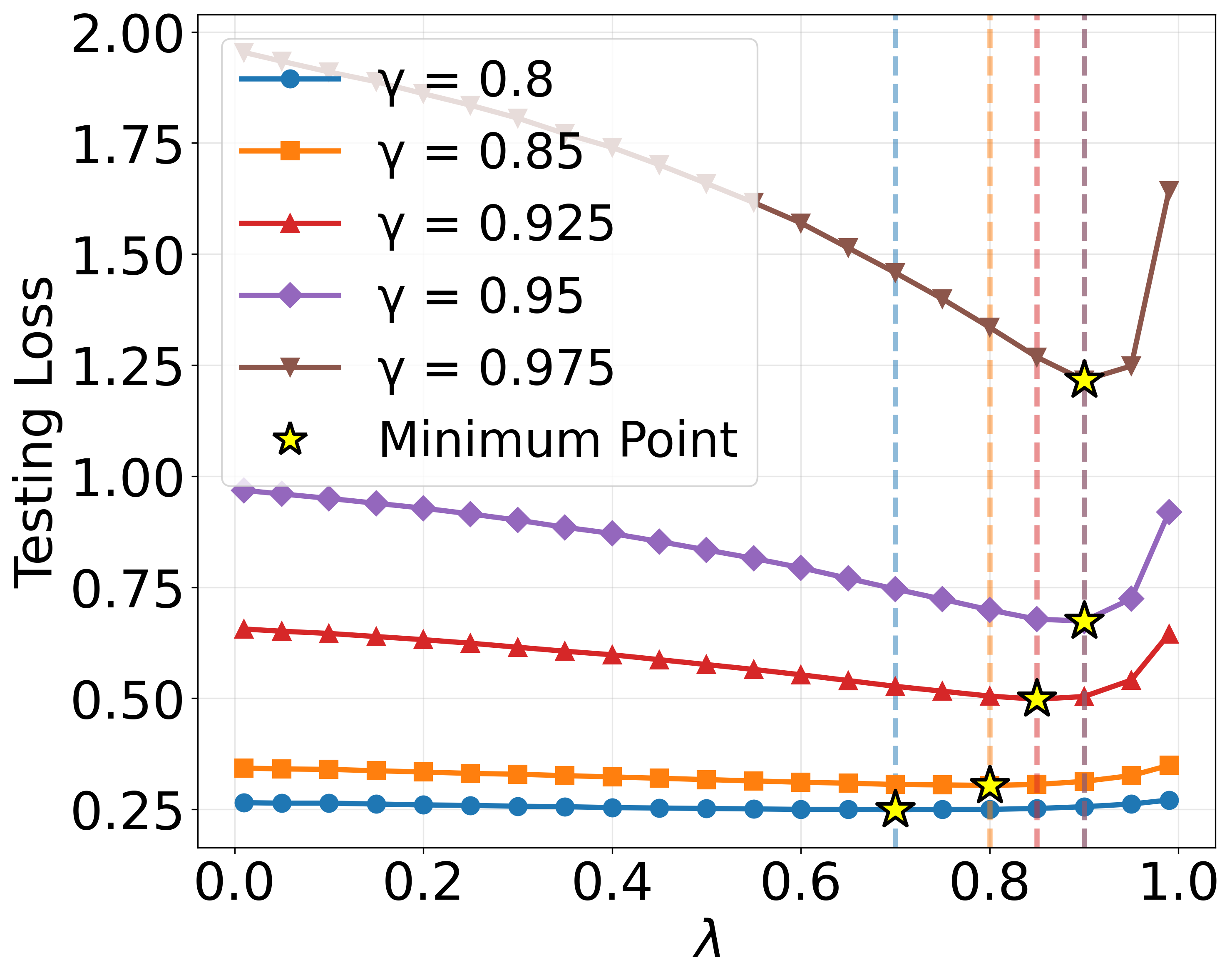}
    \caption{\footnotesize Testing loss ($m = 100$)}
    \label{testloss_m100}
\end{subfigure}
\vspace{-0.1in}
\caption{Training and testing performance of the one-layer GLA model with different $\lambda$ and $\gamma$.}
\vspace{-0.2in}
\label{fig:train_test_lambda_gamma}
\end{figure}


The first experiment compares the training and testing performance of the one-layer GLA model under varying choices of $\gamma$ and $\lambda$. As shown in \Cref{trainloss}, when the autoregressive coefficient $\gamma$ decreases and the impact of noise becomes more pronounced, an appropriate choice of $\lambda$ is required to attain the lowest training loss. During testing, we evaluate the GLA model trained with $\lambda=0.9$ under different sequence lengths $m \in \{10, 50, 100\}$. The results in \Cref{testloss_m10,testloss_m50,testloss_m100} show that, across different values of $\gamma$, selecting an appropriate $\lambda$ remains crucial for minimizing the test loss. These results highlight the role of GLA in stabilizing learning under non-stationary conditions. By introducing a gating mechanism into linear attention, GLA effectively regulates the influence of past inputs, thereby mitigating error accumulation and enhancing the model’s adaptability to distributional shifts. Consequently, GLA achieves longer effective memory and improved generalization, underscoring its advantage in handling time-varying data. As mentioned previously, a one-layer GLA model applied to a first-order autoregressive process functions analogously to an adaptive filter. To illustrate this, we compare its performance with LMS and RLS algorithms. We set the LMS step size to 0.01 and the RLS forgetting factor to 0.98, train on sequences of length 1000, and perform 10,000 Monte Carlo trials, averaging the results. The training errors for LMS and RLS are respectively $[0.2639 \ 0.3168  \ 0.6058 \ 1.0072 \ 1.4758 ]$ and $[0.2555 \ 0.3746  \ 0.6658 \ 0.8881 \ 1.2916]$ for $\gamma = [0.8 \ 0.85 \ 0.925 \ 0.95 \ 0.975]$. Compared to LMS and RLS, which require fixed or slowly adapting parameters, a one-layer GLA model achieves lower training errors (see \Cref{trainloss}) because it possesses higher representational flexibility. Furthermore, LMS and RLS adapt only to a single sequence at a time, requiring retraining for each new input, and therefore cannot leverage cross-sequence information. In contrast, GLA's learnable weights are shared across sequences, allowing the model to generalize and adapt efficiently to new inputs without retraining.

\begin{figure}[!ht]
\centering
\begin{subfigure}[t]{0.22\textwidth}
    \centering
    \includegraphics[width=3.1cm]{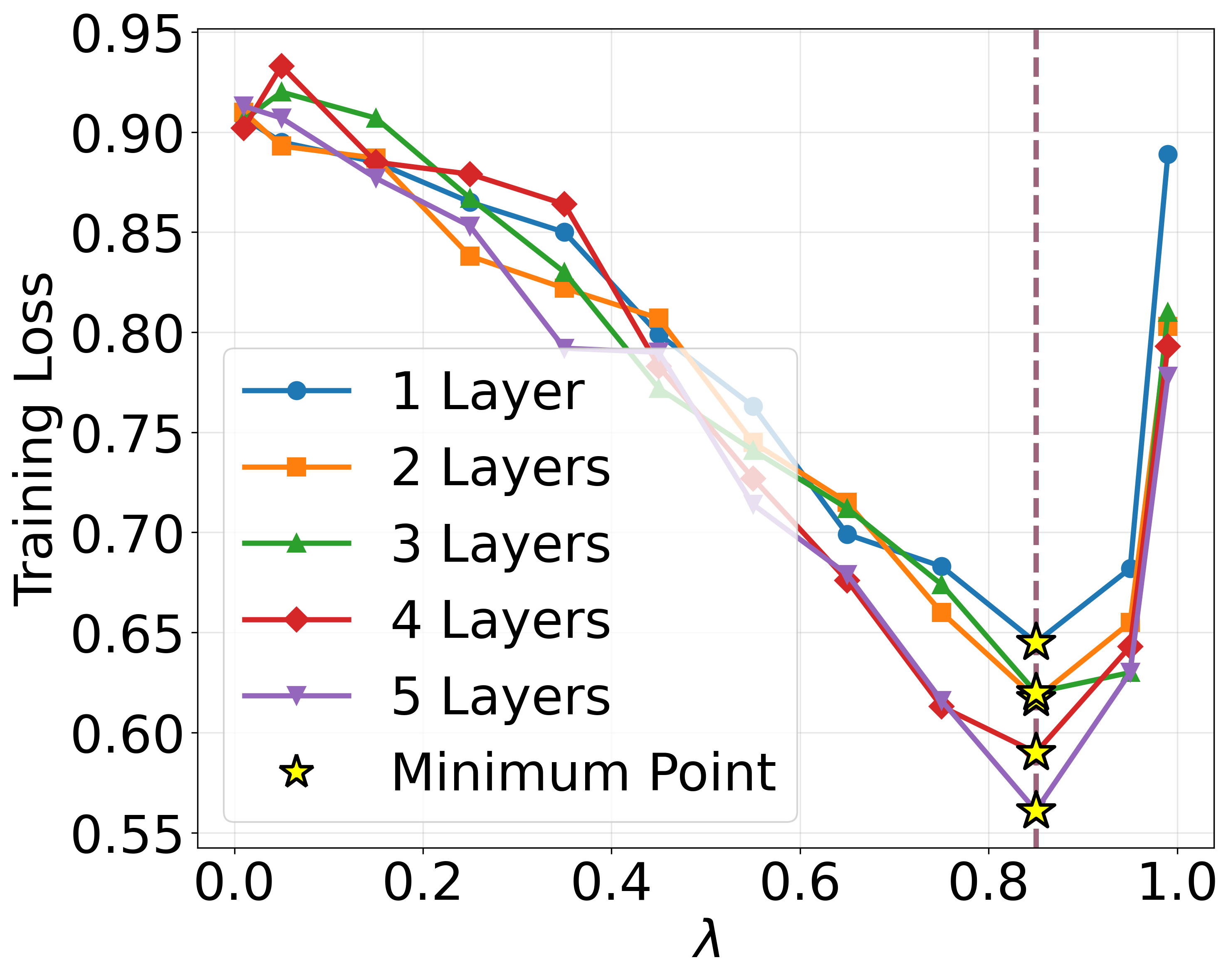}
    \caption{\footnotesize Training loss}
    \label{train_multi-layer1}
\end{subfigure}
\begin{subfigure}[t]{0.27\textwidth}
    \centering
    \includegraphics[width=3.1cm]{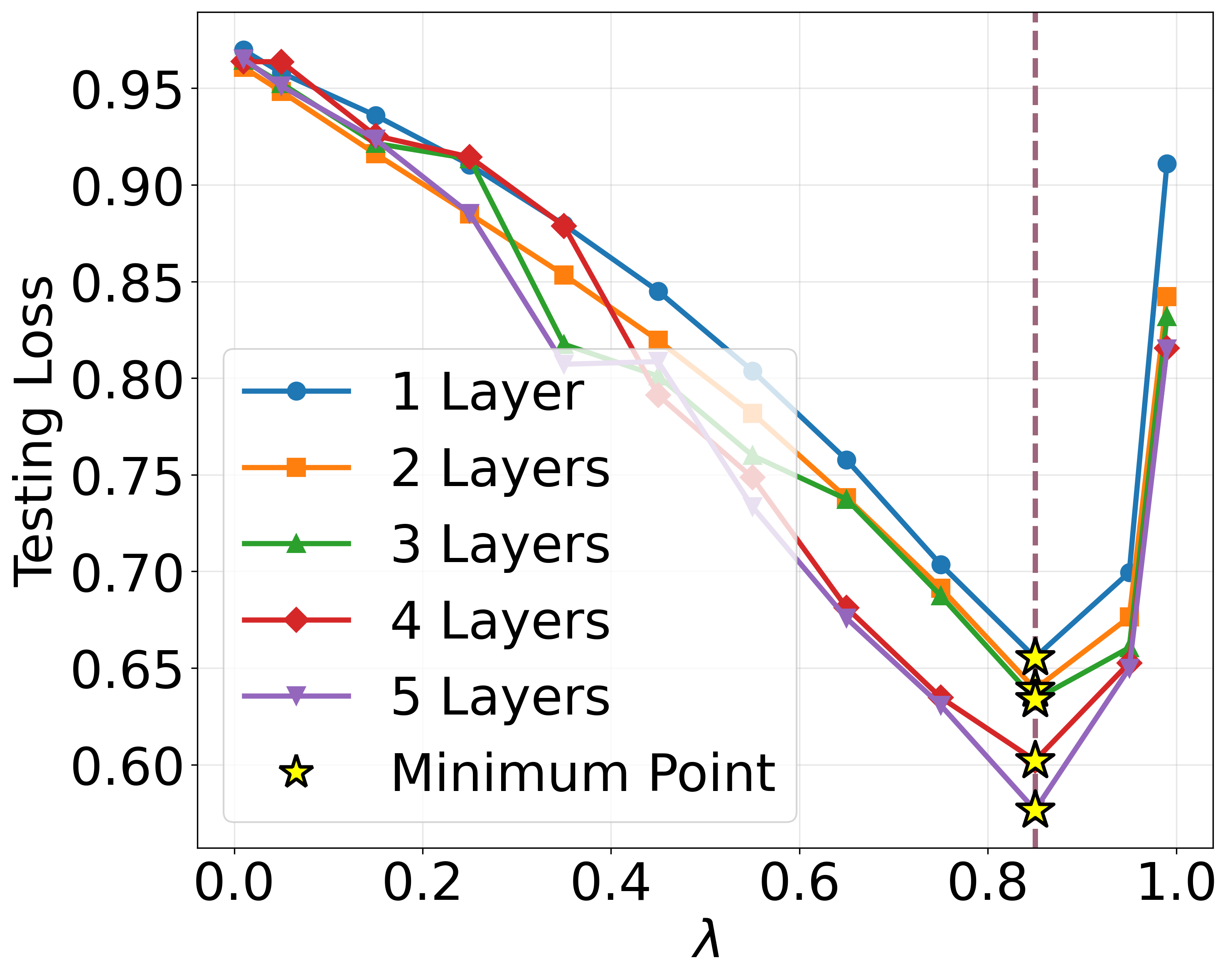}
    \caption{\footnotesize Testing loss ($m = 100$)}
    \label{test_multi-layer1}
\end{subfigure}
\begin{subfigure}[t]{0.22\textwidth}
    \centering
    \includegraphics[width=3.1cm]{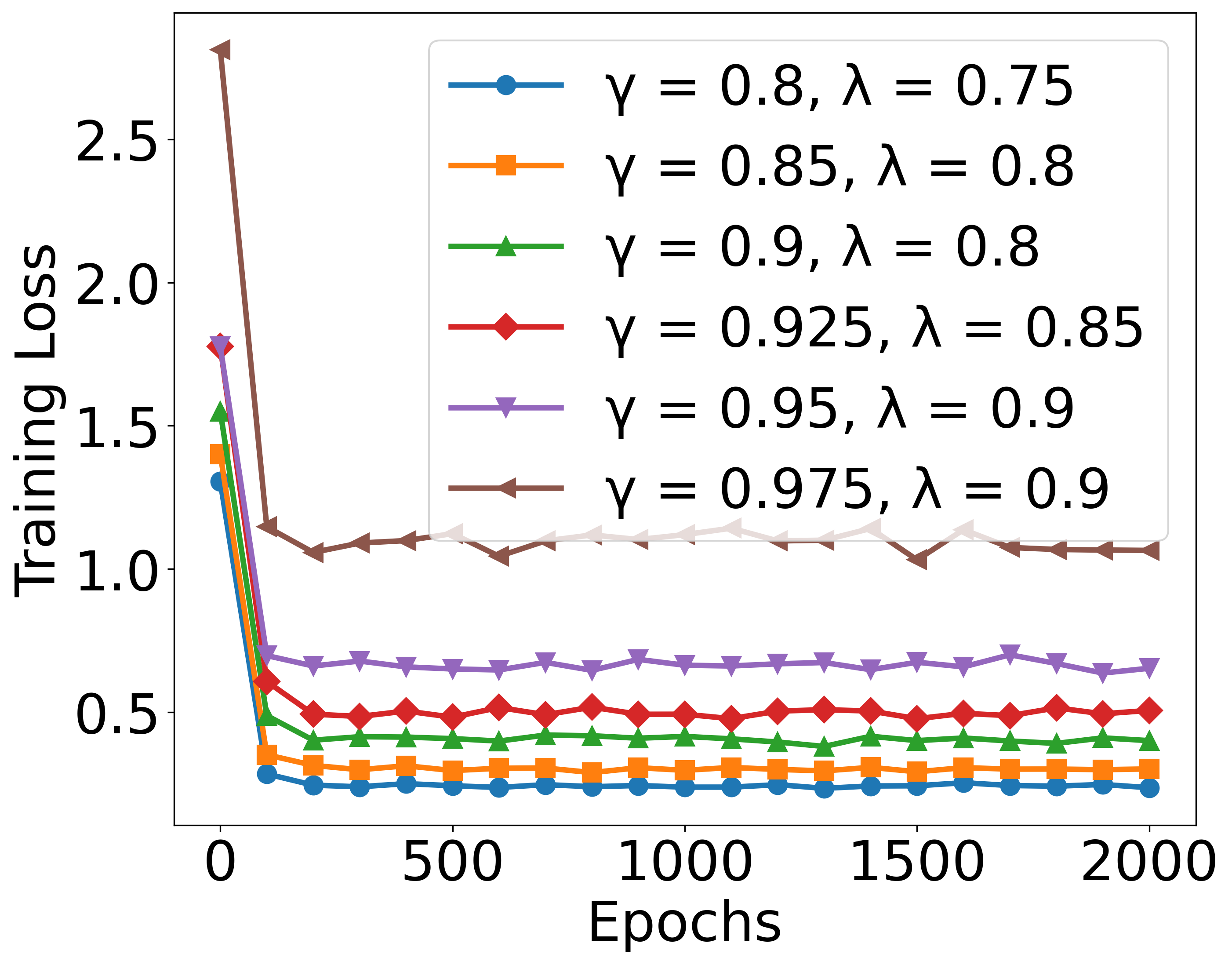}
    \caption{One Layer}
    \label{fig:training_curve_one1}
\end{subfigure}
\begin{subfigure}[t]{0.22\textwidth}
    \centering
    \includegraphics[width=3.1cm]{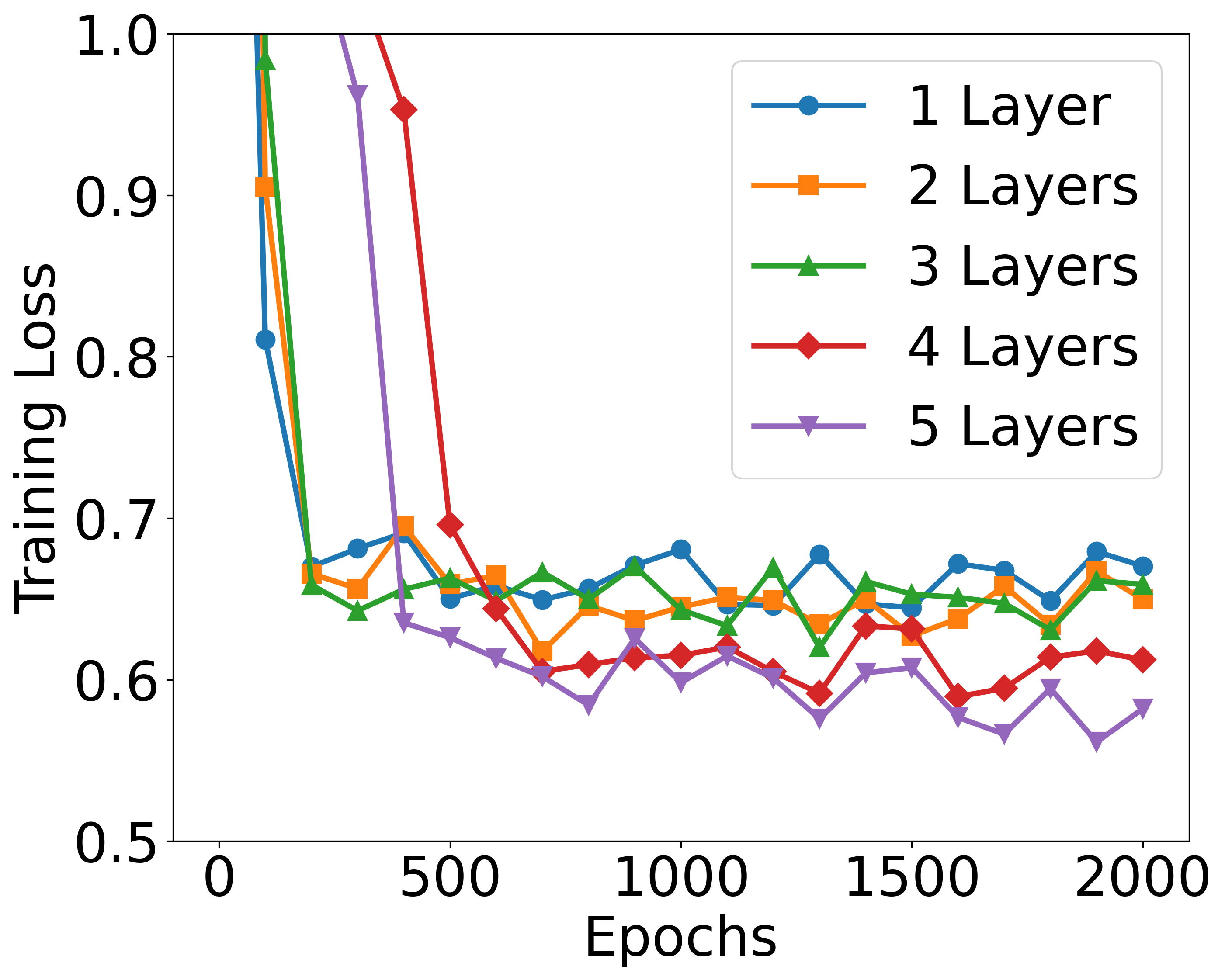}
    \caption{Multiple Layers}
    \label{fig:training_curve_Multiple Layers}
\end{subfigure}
\caption{(a-b) Training and testing performance of the multi-layer GLA model with $\gamma = 0.95$ and different $\lambda$; (c) convergence performance of the one-layer GLA model; (d) convergence performance for GLA models with different layers. }
\vspace{-.1in}
\label{fig:loss_multilayers11}
\end{figure}

In the second experiment, we investigate the impact of network depth on the performance of the GLA model. As illustrated in \Cref{train_multi-layer1,test_multi-layer1}, increasing the number of layers consistently enhances both training and testing performance, suggesting that deeper architectures can more effectively capture long-range dependencies in non-stationary sequences. Conceptually, each GLA layer implements a linear adaptive filter whose effective behavior is determined by its gating weights. When multiple layers are stacked, these adaptive filters operate at different timescales, enabling the network to simultaneously capture short-term fluctuations and longer-term trends in the evolving regression weights. This multi-timescale structure explains why deeper GLA models achieve better performance under non-stationary regression: a single layer
can track only one effective timescale of drift, while multiple layers collectively approximate a richer family of dynamic predictors.  While formal theoretical analysis for multi-layer GLA models is not yet established, the empirical results underscore the critical role of the adaptive gating mechanism in regulating information flow across layers, thereby mitigating error accumulation and improving generalization.

Under the same experimental settings as the first two experiments, we examine the training convergence of the one-layer GLA with the optimal $\lambda$ corresponding to the minimum loss, and of the multi-layer GLA with $\lambda = 0.85$. With random Gaussian initialization and a sufficiently large number of training samples, \Cref{fig:training_curve_one1} shows that the one-layer GLA achieves linear convergence, in agreement with our previous analysis. \Cref{fig:training_curve_Multiple Layers} further demonstrates that the multi-layer GLA maintains linear convergence, indicating that the adaptive gating mechanism effectively stabilizes gradient propagation across layers. A rigorous theoretical characterization of convergence for multi-layer GLA is left for future work.

In the third experiment, we assess the ICL capability of GLA and Linear Attention (LA) models on a real-world language task. We focus on sentiment classification using the SST-2 dataset \citep{socher2013recursive}, which contains 67,349 training samples and 872 validation samples with binary labels (positive/negative). To initialize the models, we employ GPT-2 (small) \citep{radford2019language}, which consists of 12 layers, a hidden size of 768, 12 attention heads, and approximately 117M parameters. We then replace the original softmax attention with (i) linear attention, resulting in LinearGPT2, and (ii) gated linear attention, resulting in GatedLinearGPT2. Both models are optimized using AdamW with a learning rate of $5\times10^{-5}$, weight decay of $0.05$, and momentum parameter of $0.9$ for 1,000 iterations. For ICL fine-tuning, we provide 20 in-context demonstrations per instance, computing the loss only on label tokens. During evaluation on the SST-2 validation set, we vary the number of demonstrations $K \in \{1, 5, 10, 15, 20\}$. Performance is assessed using two metrics: (1) Accuracy, defined as the standard prediction accuracy; and (2) Confidence, calculated for each correctly classified example by converting the model’s logits over {positive, negative} to probabilities $(p_{\text{pos}}, p_{\text{neg}})$ and taking $\max(p_{\text{pos}}, p_{\text{neg}})$, with the reported value being the average over all correctly classified examples. As shown in \Cref{fig:sst2_acc,fig:sst2_con}, when $\lambda = 0.9$, GLA achieves the highest accuracy and confidence, outperforming LA by a clear margin. This empirical advantage can be attributed to its gating mechanism: unlike LA, which implicitly assumes a stationary linear regression structure, GLA is able to adapt to the non-stationarity of real-world data by selectively integrating or discarding historical information--an ability that proves critical for reliable prediction.

\vspace{-0.3cm}
\begin{figure}[!ht]
\centering

\begin{subfigure}{0.23\textwidth}
    \centering
    \includegraphics[width=\linewidth]{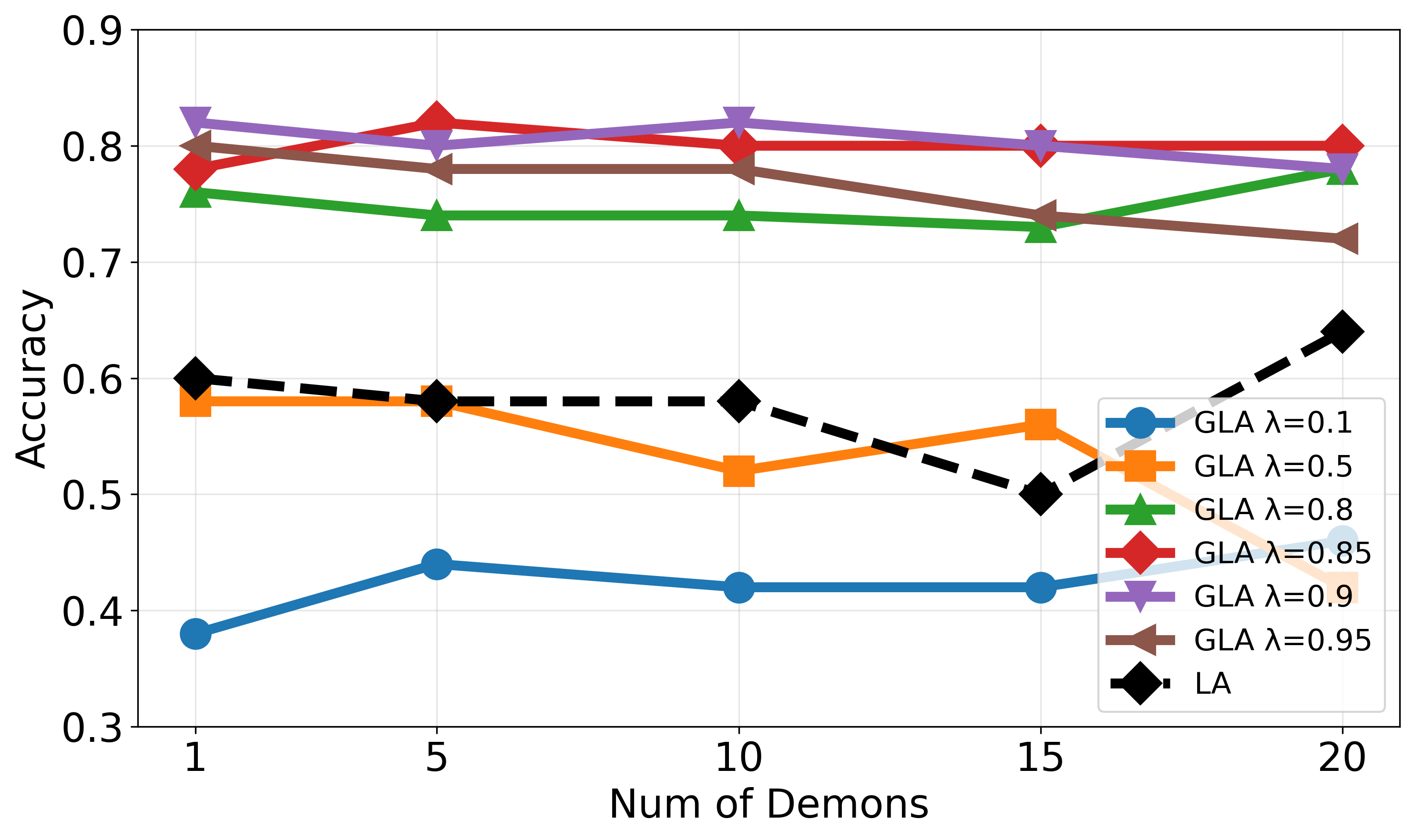}
    \caption{SST-2 Accuracy}
    \label{fig:sst2_acc}
\end{subfigure}\hfill
\begin{subfigure}{0.23\textwidth}
    \centering
    \includegraphics[width=\linewidth]{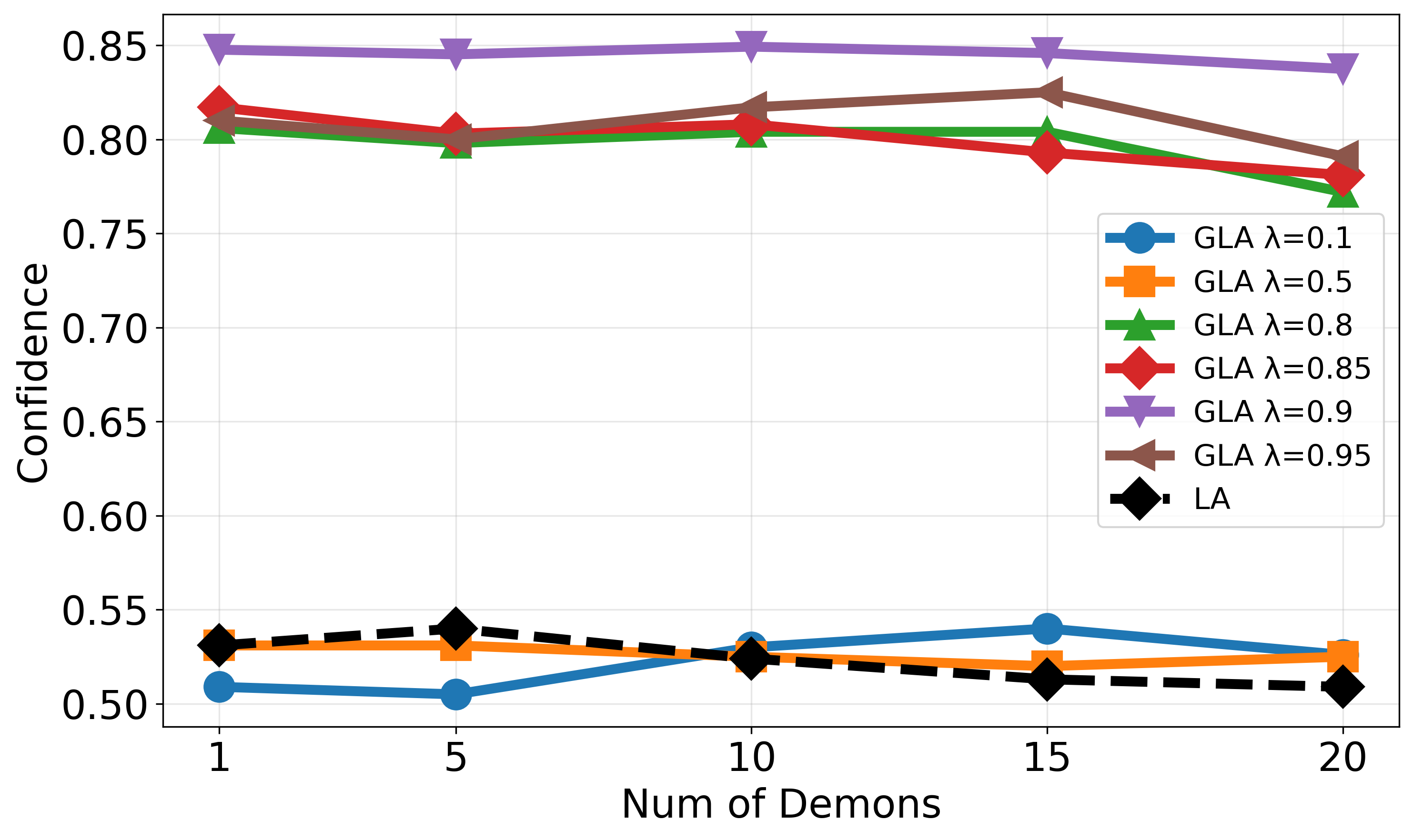}
    \caption{SST-2 Confidence}
    \label{fig:sst2_con}
\end{subfigure}\hfill
\begin{subfigure}{0.23\textwidth}
    \centering
    \includegraphics[width=\linewidth]{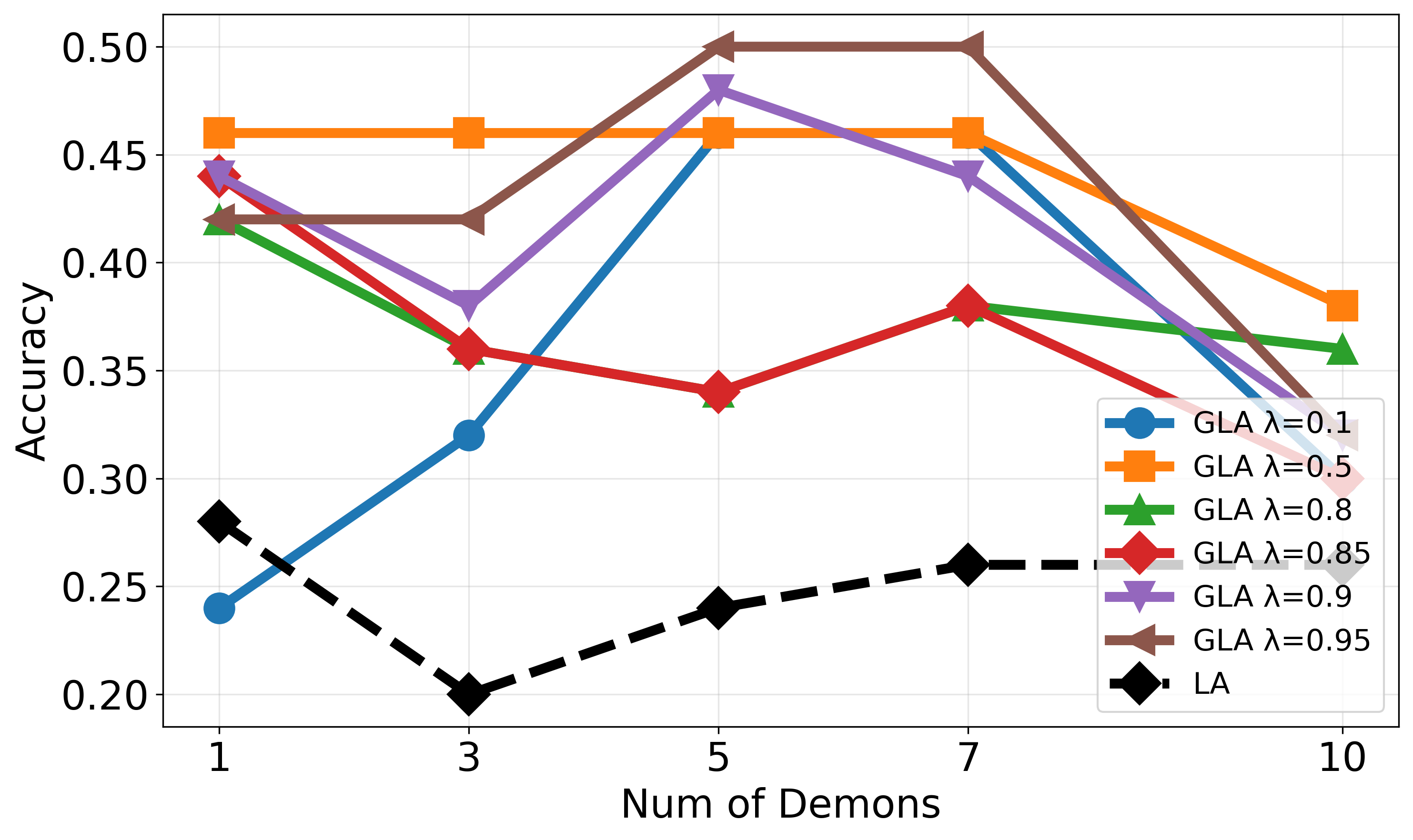}
    \caption{MNLI Accuracy}
    \label{fig:mnli_acc}
\end{subfigure}\hfill
\begin{subfigure}{0.23\textwidth}
    \centering
    \includegraphics[width=\linewidth]{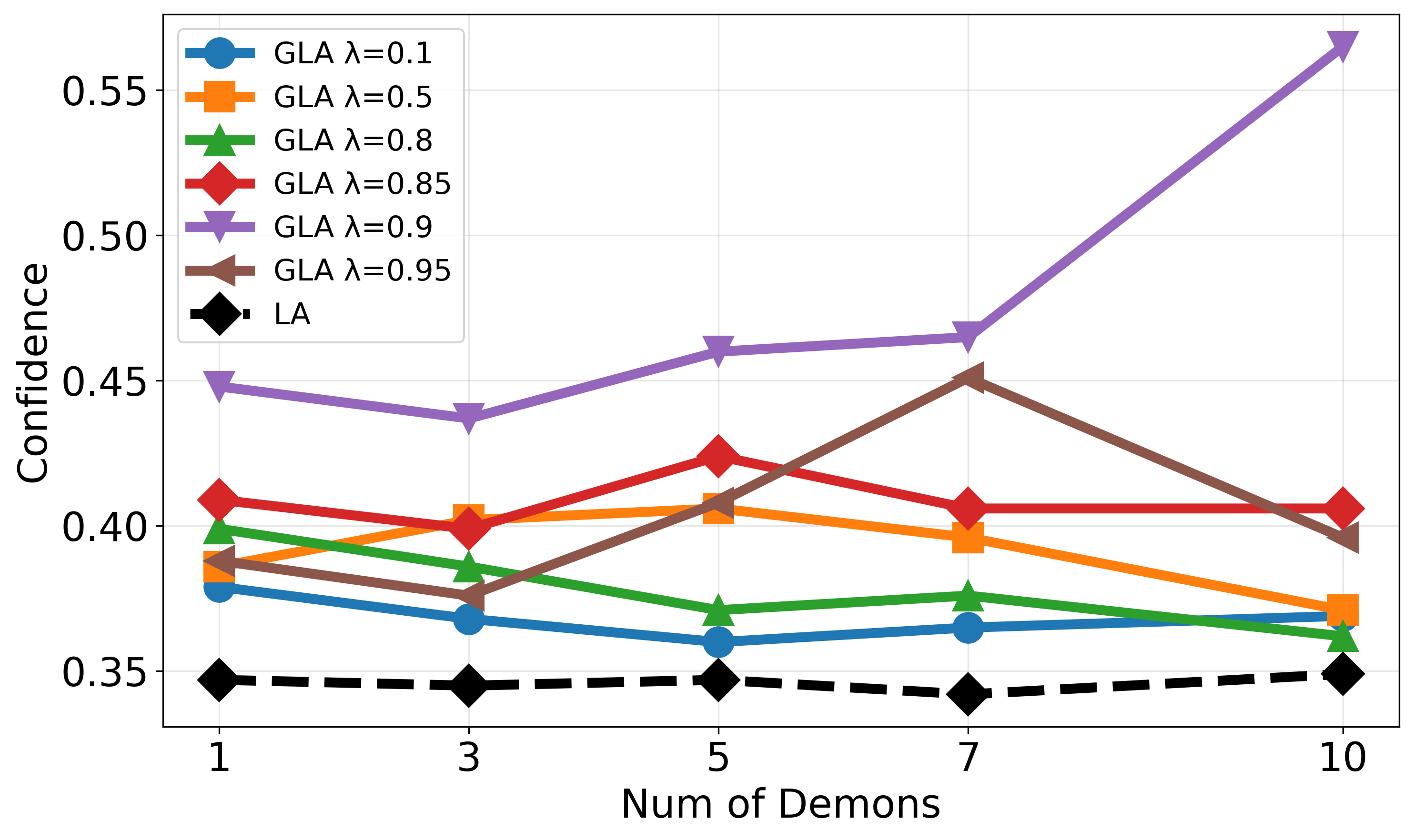}
    \caption{MNLI Confidence}
    \label{fig:mnli_con}
\end{subfigure}

\caption{
Accuracy and confidence of GatedLinearGPT2 vs.\ LinearGPT2 on SST-2 sentiment classification (left two) and MNLI natural language inference (right two) across different numbers of demonstrations.
}
\label{fig:combined_gla_la}
\end{figure}

In the final experiment, we evaluate the ICL capabilities of GLA and Linear Attention (LA) models on a more challenging natural language inference task, which requires determining the logical relationship between a premise–hypothesis pair (entailment, contradiction, or neutral) across a broad range of text genres. We use the Multi-Genre Natural Language Inference (MNLI) dataset~\citep{williams2018broad}, which spans multiple genres and contains approximately 393k training examples with three class labels. Following the same setup as in the third experiment, we provide 10 in-context demonstrations per instance for ICL fine-tuning—constrained by context length—and compute the loss only on the label tokens. For evaluation on the MNLI validation set, we vary the number of demonstrations  $K \in \{1, 3, 5, 7, 10\}$. As shown in \Cref{fig:mnli_acc,fig:mnli_con}, GLA consistently achieves higher accuracy and confidence than LA, highlighting the benefit of the gating mechanism.

\vspace{-0.2cm}

\section{Conclusion}
\label{sec:conclusion}

\vspace{-0.2cm}

This work presents a theoretical investigation of in-context learning in non-stationary regression problems, addressing an important gap in the current understanding of transformer models. Under a first-order autoregressive model of non-stationarity, we show that GLA outperforms standard linear attention by dynamically reweighting past inputs, enabling more accurate prediction in time-varying settings. Our analysis provides rigorous justification for the advantage of gating in capturing distributional shifts and highlights its role as an architectural inductive bias in adaptive learning. These findings not only deepen the theoretical foundations of ICL in dynamic environments but also suggest broader implications for the design of transformer variants in real-world applications characterized by non-stationarity.

A natural direction for future work is to generalize the first-order autoregressive assumption to a broader class of dynamic-weight models. In particular, allowing more flexible temporal evolutions--such as higher-order dynamics, stochastic drift, or slowly varying adversarial changes--would further illuminate how in-context learning behaves in general non-stationary settings. A second direction for future work is to develop a rigorous theoretical characterization of how gating mechanisms interact across multiple GLA layers. While our experiments show that stacking layers consistently improves performance, a principled analysis of how multi-layer structures capture multiple timescales of drift remains an important open problem. The third future direction is to analyze the global optimization landscape of the GLA model studied in this paper. Our numerical experiments suggest that random Gaussian initialization consistently converges to a global minimum under gradient flow, even when the theoretical initialization conditions are violated. This indicates the existence of a benign global optimum and motivates a deeper theoretical study of the model's optimization landscape.

\section*{Acknowledgments}

The work has been supported in part by NSF grants  IIS-2312840 and IIS-2402952. ZQ gratefully acknowledges support from the MICDE Research Scholars Program at the University of Michigan.



\newpage

\appendices

\section*{Ethics statement}
This study presents a theoretical analysis of in-context learning in non-stationary regression problems. No new human or animal data were collected, and all experiments rely exclusively on publicly available datasets that have been widely used in prior research. We recognize that pretrained LLMs may inherit biases from their training corpora. Since our method does not involve additional fine-tuning of LLMs, it does not directly address such biases. We encourage future investigations to place greater emphasis on fairness, accountability, and transparency in deploying these models in practical scenarios.

\section*{Reproducibility statement}
We have taken extensive measures to facilitate reproducibility. Detailed descriptions of model architectures, hyperparameters, and experimental setups are included in both the main text and the appendix. In addition, we will release the full codebase, configuration files, and comprehensive documentation upon publication, thereby enabling independent verification and extension of our results.

\section*{The Use of Large Language Models}
Large language models were employed solely for improving clarity, grammar, and overall readability of the manuscript. They were not used for conceptual development, experimental design, data analysis, or the generation of research content. All theoretical contributions, implementations, and experimental results reported in this paper are the authors’ own work.

\section{Technical tools used in the proofs}
\label{Technical tools used in proofs}

\begin{lemma}
\label{covariance matrix of gaussian vector}
Suppose $\vw_i = \gamma \vw_{i-1} + \ve_i$ where $\vw_0\overset{\text{i.i.d.}}{\sim}\calN({\bm 0}, \sigma_w^2\mId)$, $\ve_i\overset{\text{i.i.d.}}{\sim}\calN({\bm 0}, \sigma_e^2\mId)$ and they are mutually independent. We have
\begin{eqnarray}
    \label{covariance matrix of gaussian vector1}
    \E[\vw_a \vw_b^\top] = \begin{cases}
    (\sigma_w^2 + \min\{a,b\}\sigma_e^2  )\mId, & \gamma = 1, \\
    ( \gamma^{a+b}\sigma_w^2 + \frac{ \gamma^{a+b} -  \gamma^{|a-b|}}{\gamma^2 -1}\sigma_e^2  )\mId , &  \gamma \neq 1.
\end{cases}
\end{eqnarray}

\end{lemma}
\begin{proof}
For any $a$ and $b$, we have
\begin{eqnarray}
    \label{expansion of wa and wb}
    &&\vw_a = \gamma^a \vw_0 + \sum_{i=1}^a\gamma^{a-i}\ve_i,\\
    &&\vw_b = \gamma^b \vw_0 + \sum_{i=1}^b\gamma^{b-i}\ve_i.
\end{eqnarray}
Then we have
\begin{eqnarray}
    \label{covariance matrix of gaussian vector derivation}
    \E[\vw_a \vw_b^\top] &\!\!\!\!=\!\!\!\!& \E[ \gamma^{a+b} \vw_0 \vw_0^\top] + \sum_{j=1}^{\min\{a,b\}}\gamma^{a+b-2j} \E[\ve_j\ve_j^\top ] \nonumber\\
    &\!\!\!\!=\!\!\!\!& \gamma^{a+b}\sigma_w^2\mId + \sum_{j=1}^{\min\{a,b\}}\gamma^{a+b-2j}\sigma_e^2\mId\nonumber\\
    &\!\!\!\!=\!\!\!\!& \begin{cases}
    (\sigma_w^2 + \min\{a,b\}\sigma_e^2  )\mId, & \gamma = 1, \\
    ( \gamma^{a+b}\sigma_w^2 + \frac{ \gamma^{a+b} -  \gamma^{|a-b|}}{\gamma^2 -1}\sigma_e^2  )\mId , &  \gamma \neq 1.
\end{cases}
\end{eqnarray}
\end{proof}

According to \citep[Lemma 5.1]{zhang2024trained}, we have
\begin{lemma}
\label{different form of y n+1}
Let $\vz_i, i=1,\dots, n+1$ be prompts defined in \eqref{definition of input prompt}. Then the prediction $\wh y_{n+1}$ can be written as a quadratic function as follows:
\begin{eqnarray}
    \label{quadratic function for query}
    \wh y_{n+1} = \vu^\top \mH \vu,
\end{eqnarray}
where $\mH = \frac{1}{2}\mX \otimes  \bigg(\sum_{i=1}^{n+1} \lambda^{n+1-i} \vz_i\vz_i^\top \bigg)$ with $\mX = \begin{bmatrix}{\bm 0}_{d\times d} &   \vx_{n+1} \\ \vx_{n+1}^\top   & 0 \end{bmatrix}\in\R^{(d+1)\times (d+1)} $ and $\vu = \vec(\mU)\in\R^{(d+1)^2 \times 1}$, $\mU = \begin{bmatrix}\mU_{11} &   \vu_{12} \\ \vu_{21}^\top   & u_{-1} \end{bmatrix} \in\R^{(d+1)\times (d+1)}$ with $\mU_{11} = \mW_{11}^{KQ}$, $ \vu_{21} = \vw_{21}^{KQ}$, $\vu_{12} = \vw_{21}^{V}$ and $u_{-1} = w_{-1}^V $ defined in \eqref{out of WV and WKQ}.
\end{lemma}

\begin{lemma}
\label{dynamics of population loss}
Let $\vu = \vec(\mU) = \vec\bigg( \begin{bmatrix}\mU_{11} &   \vu_{12} \\ \vu_{21}^\top   & u_{-1} \end{bmatrix} \bigg)$ as \Cref{different form of y n+1}. Consider gradient flow over
\begin{eqnarray}
    \label{new loss function}
    L(\vu) = \frac{1}{2}\E[( \vu^\top\mH \vu - \vw_{n+1}^\top \vx_{n+1}   )^2]
\end{eqnarray}
with respect to $\vu$ starting from an initial value satisfying {Assumption} \ref{Assumption of initialization}. Then the dynamics of $\mU$ satisfies
\begin{eqnarray}
    \label{dynamics of U11}
    &&\frac{\text{d}\mU_{11}(t)}{\text{d}t} = -\vu_{-1}^2\wt\mLambda\mLambda \mU_{11}\mLambda + D_1u_{-1}\mLambda^2,\\
    \label{dynamics of u-1}
    &&\frac{\text{d}u_{-1}(t)}{\text{d}t} = -\trace(u_{-1}\wt\mLambda \mLambda \mU_{11}\mLambda \mU_{11}^\top ) + D_1\trace(\mLambda^2\mU_{11}^\top),
\end{eqnarray}
and $\vu_{12}(t) = {\bm 0}_d$, $\vu_{21}(t) = {\bm 0}_d$ for all $t\geq 0$, where
\begin{eqnarray*}
D_1 =\begin{cases}
    n\sigma_w^2 + \frac{n(n+1)}{2}\sigma_e^2, & \lambda = \gamma = 1, \\
    \frac{\lambda - \lambda^{n+1}}{1-\lambda}\sigma_w^2 + \frac{\lambda^{n+2} - (n+1)\lambda^2 + n\lambda}{(1-\lambda)^2}\sigma_e^2, &  \lambda \neq 1, \gamma = 1,\\
    \frac{\gamma^{n+2} - \gamma^{2n+2}}{1 - \gamma}\sigma_w^2 + \frac{\gamma - \gamma^{n+1} - \gamma^{n+2} + \gamma^{2n+2}}{(1-\gamma)^2(1+\gamma)}\sigma_e^2, & \lambda = 1, \gamma \neq 1, \\
    \lambda^{2n+2}n\sigma_w^2 + \bigg(\frac{\lambda^2( 1 - \lambda^{2n}) }{(1-\lambda^2)^2} - \frac{\lambda^{2n+2}}{1-\lambda^2}n \bigg)\sigma_e^2, &  \lambda \neq 1, \gamma \neq 1, \lambda = \gamma,\\
    \frac{\gamma -  \gamma^{2n+1} }{\lambda - \gamma}\sigma_w^2 +\bigg( \frac{1}{1-\gamma^2}n - \frac{\gamma -  \gamma^{2n+1}}{(\lambda - \gamma)(1-\gamma^2)}  \bigg)\sigma_e^2, &  \lambda \neq 1, \gamma \neq 1, \lambda = 1/\gamma, \\
    \frac{\lambda^{n+1}\gamma^{n+2} - \lambda \gamma^{2n+2} }{\lambda - \gamma}\sigma_w^2 +\bigg( \frac{\lambda\gamma( 1 - \lambda^n\gamma^n) }{(1-\gamma^2)(1-\lambda\gamma)} - \frac{\lambda^{n+1}\gamma^{n+2} - \lambda \gamma^{2n+2} }{(\lambda - \gamma)(1-\gamma^2)}  \bigg)\sigma_e^2, &  \lambda \neq 1, \gamma \neq 1, \lambda \neq \gamma, \lambda \neq 1/\gamma.
\end{cases}
\end{eqnarray*}
and $\wt\mLambda = D_2(2\mLambda +\trace(\mLambda)\mId ) + D_3\mLambda $ with \begin{eqnarray*}
D_2 =\begin{cases}
    n\sigma_w^2 + \frac{(n+1)n}{2}\sigma_e^2 , & \lambda = 1, \gamma = 1, \\
    \frac{\lambda^2 - \lambda^{2n+2}}{1-\lambda^2}\sigma_w^2 + \frac{(n+1)\lambda^2 - \lambda^{2n+2} - n}{1-\lambda^2}\sigma_e^2  , & \lambda \neq 1, \gamma =1,\\
    \frac{\gamma^2 - \gamma^{2n+2}}{1-\gamma^2}\sigma_w^2 + \big( \frac{n\gamma^2}{1-\gamma^2} - \frac{\gamma^2 - \gamma^{2n+2}}{(1 - \gamma^2  )^2}   \big)\sigma_e^2, & \lambda = 1, \gamma \neq 1,\\
    \lambda^{2n+2}n\sigma_w^2 - (\frac{n\lambda^{2n+2}}{1-\lambda^2} - \frac{\lambda^6 - \lambda^{2n+4}}{(1- \lambda^2 )^2}  )\sigma_e^2, &  \lambda \neq 1, \gamma \neq 1, \lambda = \gamma,\\
    \frac{\lambda^{2n} -  \gamma^{2n}}{\lambda^2 - \gamma^2}\sigma_w^2 - (\frac{\lambda^{2n} -  \gamma^{2n}}{(\lambda^2 - \gamma^2)(1-\gamma^2)} - \frac{\lambda^2 - \lambda^{2n}}{(1-\gamma^2)(1-\lambda^2)}    )\sigma_e^2, &  \lambda \neq 1, \gamma \neq 1, \lambda = 1/\gamma, \\
    \frac{\gamma^2\lambda^{2n+2} - \lambda^2 \gamma^{2n+2}}{\lambda^2 - \gamma^2}\sigma_w^2 - (\frac{\gamma^2\lambda^{2n+2} - \lambda^2 \gamma^{2n+2}}{(\lambda^2 - \gamma^2)(1-\gamma^2)} - \frac{\gamma^2(\lambda^4 - \lambda^{2n+2})}{(1-\gamma^2)(1-\lambda^2)}    )\sigma_e^2, &  \lambda \neq 1, \gamma \neq 1, \lambda \neq \gamma, \lambda \neq 1/\gamma,
\end{cases}
\end{eqnarray*}
and
\begin{eqnarray*}
D_3 =  \begin{cases}
    n(n-1)\sigma_w^2+ \frac{(n-1)n(n+1)}{3}\sigma_e^2, & \lambda = 1, \gamma = 1, \\
    (\frac{2(\lambda^{n+1} - \lambda^2)}{(1 - \lambda)^2} - \frac{2(\lambda^{2n+1} - \lambda^3)}{(1-\lambda)^2(1+\lambda)})\sigma_w^2 +   (\frac{2(n\lambda^4 - \lambda^{2n+4} - (n-1)\lambda^2 )}{(1-\lambda)(1-\lambda^2)^2} + \frac{2(\lambda^n - n\lambda +n-1)}{(1-\lambda)^2(1+\lambda) \lambda^{n-2}})\sigma_e^2   , & \lambda \neq 1, \gamma =1,\\
    \big( 2\frac{\gamma^3 - \gamma^{2n+1}}{(1-\gamma)^2(1+\gamma)} - 2\frac{\gamma^{n+2} - \gamma^{2n+1}}{(1-\gamma)^2}   \big)\sigma_w^2 + \big(\frac{2}{\gamma^2 - 1}(\frac{\gamma^3 - \gamma^{2n+1}}{(1-\gamma)^2(1+\gamma)} - \frac{\gamma^{n+2} - \gamma^{2n+1}}{(1-\gamma)^2})\\
     - \frac{2\gamma^3}{(\gamma^2 - 1)(1-\gamma)}(n-1 - \frac{\gamma^n - \gamma}{\gamma-1})   \big)\sigma_e^2 , & \lambda = 1, \gamma \neq 1,\\
    \lambda^{2n+2}n(n-1)\sigma_w^2 + (\frac{2n(\lambda^6 - \lambda^{2n+4}) }{(1-\lambda^2)^2} - \frac{2(\lambda^{2n+6}  - n\lambda^8 + (n-1)\lambda^6)}{(1-\lambda^2)^3} - \frac{\lambda^{2n+2}n(n-1)}{1-\lambda^2} )\sigma_e^2 , &  \lambda \neq 1, \gamma \neq 1, \lambda = \gamma,\\
    (\frac{2\lambda^{2n} - 2\gamma^{2n-4}}{\lambda(\lambda - \gamma)^2(\lambda + \gamma)} - \frac{2-2\gamma^{2n-2}}{\lambda - \gamma} )\sigma_w^2   + \big( \frac{2\lambda^{2n-2} - 2}{(1 - \gamma^2)(\lambda - \gamma)^2} - \frac{2(n-1)\gamma}{(1 - \gamma^2)(\lambda -\gamma)}\\
    -  \frac{2\lambda^{2n} - 2\gamma^{2n-4}}{\lambda(\lambda - \gamma)^2(\lambda + \gamma)(1 - \gamma^2)} + \frac{2-2\gamma^{2n-2}}{(1 - \gamma^2)(\lambda - \gamma)}  \big)\sigma_e^2 , &  \lambda \neq 1, \gamma \neq 1, \lambda = 1/\gamma, \\
    (\frac{2\gamma^3\lambda^{2n+3} - 2\lambda^5\gamma^{2n+1}}{\lambda(\lambda - \gamma)^2(\lambda + \gamma)} - \frac{2\gamma^{n+2}\lambda^{n+2} - 2\gamma^{2n+1}\lambda^3}{\lambda - \gamma}     )\sigma_w^2  +  (\frac{2\gamma(\lambda^4 - \lambda^{2n+2})}{(1-\gamma^2)(\lambda - \gamma)(1-  \lambda^2)}\\
    - \frac{2(\lambda^{3}\gamma^2 - \lambda^{n+2}\gamma^{n+1})}{(1 - \gamma^2)(\lambda - \gamma)(1 - \lambda\gamma)} - \frac{2\gamma^3\lambda^{2n+3} - 2\lambda^5\gamma^{2n+1}}{\lambda(\lambda - \gamma)^2(\lambda + \gamma)(1 - \gamma^2)} + \frac{2\gamma^{n+2}\lambda^{n+2} - 2\gamma^{2n+1}\lambda^3}{(\lambda - \gamma)(1 - \gamma^2)}    )\sigma_e^2, &  \lambda \neq 1, \gamma \neq 1, \lambda \neq \gamma, \lambda \neq 1/\gamma.
\end{cases}
\end{eqnarray*}

\end{lemma}

\begin{proof}
We first calculate the derivatives of $\vu$ as following
\begin{eqnarray}
    \label{derivative of new L}
    \frac{\text{d}\vu(t)}{\text{d}t} = -\nabla_{\vu}L(\vu)  =  -2\E[\<\mH, \vu\vu^\top \>\mH ]\vu + 2\E[\vw_{n+1}^\top \vx_{n+1}\mH]\vu.
\end{eqnarray}
Following the same analysis of \citep[Lemma 5.2]{zhang2024trained}, we can first obtain
\begin{eqnarray}
    \label{first term in derivative of new L}
    &\!\!\!\!\!\!\!\!&2\E[\<\mH, \vu\vu^\top \>\mH\vu]\nonumber\\
    &\!\!\!\!=\!\!\!\!& \frac{1}{2}\E\bigg[\vec(\mU)^\top\vec\bigg(\bigg(\sum_{i=1}^{n+1} \lambda^{n+1-i} \vz_i\vz_i^\top \bigg) \mU \mX\bigg)  \vec\bigg(\bigg(\sum_{i=1}^{n+1} \lambda^{n+1-i} \vz_i\vz_i^\top \bigg) \mU \mX\bigg) \bigg]\nonumber\\
    &\!\!\!\!=\!\!\!\!& \frac{1}{2}\E\bigg[\sum_{i=1}^{d+1}\sum_{j=1}^{d+1}\mT(i,j)\mU(i,j)\vec(\mT)  \bigg]
\end{eqnarray}
where $\mT = \bigg(\sum_{i=1}^{n+1} \lambda^{n+1-i} \vz_i\vz_i^\top \bigg) \mU \mX$ and $\mH$, $\mX$ are defined in \Cref{different form of y n+1}.

In addition, we can derive
\begin{eqnarray}
    \label{second term in derivative of new L}
    &\!\!\!\!\!\!\!\!&2\E[\vw_{n+1}^\top \vx_{n+1}\mH]\vu\nonumber\\
    &\!\!\!\!=\!\!\!\!& \sum_{j=1}^d \E\bigg[(\vx_{n+1}(j)\mX )\otimes \bigg(\vw_{n+1}(j)\bigg(\sum_{i=1}^{n+1} \lambda^{n+1-i} \vz_i\vz_i^\top \bigg) \bigg) \bigg]\vu \nonumber\\
    &\!\!\!\!=\!\!\!\!&  \sum_{j=1}^d \bigg(\E\bigg[\vx_{n+1}(j)\mX \bigg]\otimes \E\bigg[\vw_{n+1}(j)\bigg(\sum_{i=1}^{n+1} \lambda^{n+1-i} \vz_i\vz_i^\top \bigg)  \bigg]\bigg)\vu\nonumber\\
    &\!\!\!\!=\!\!\!\!& \sum_{j=1}^d \bigg(\E\bigg[\begin{bmatrix} {\bm 0}_{d\times d} & \vx_{n+1}(i)\vx_{n+1} \\  \vx_{n+1}(i)\vx_{n+1}^\top  & 0   \end{bmatrix} \bigg]\nonumber\\
    &\!\!\!\!\!\!\!\!& \otimes \E\bigg[ \begin{bmatrix} \sum_{i=1}^{n+1}\lambda^{n+1-i}\vw_{n+1}(j)\vx_i \vx_i^\top  & \sum_{i=1}^{n}\lambda^{n+1-i}\vw_{n+1}(j)\vx_i \vx_i^\top\vw_i \\  \sum_{i=1}^{n}\lambda^{n+1-i}\vw_{n+1}(j)\vw_i^\top\vx_i \vx_i^\top  & \sum_{i=1}^{n}\lambda^{n+1-i}\vw_{n+1}(j)\vw_i^\top\vx_i \vx_i^\top\vw_i   \end{bmatrix} \bigg]\bigg)\vu\nonumber\\
    &\!\!\!\!=\!\!\!\!& D_1\sum_{j=1}^{d} \bigg(\begin{bmatrix} {\bm 0}_{d\times d} & \mLambda(:,j) \\  \mLambda^\top(:,j)  & 0   \end{bmatrix} \otimes \begin{bmatrix} {\bm 0}_{d\times d} & \mLambda(:,j) \\  \mLambda^\top(:,j)  & 0   \end{bmatrix}\bigg)\vu,
\end{eqnarray}
where the second equation follows the fact that $\{\vw_i \}$ are independent of $\{ \vx_i \}$, and the last line follows $\vw_0\overset{\text{i.i.d.}}{\sim}\calN({\bm 0}, \sigma_w^2\mId)$, $\ve_i\overset{\text{i.i.d.}}{\sim}\calN({\bm 0}, \sigma_e^2\mId)$, $\vx_i\overset{\text{i.i.d.}}{\sim}\calN({\bm 0}, {\bm \Lambda})$ and they are mutually independent. In addition, to compute $D_1$, we apply \Cref{covariance matrix of gaussian vector}  to evaluate the cross-covariance $\E[\vw_{n+1}\vw_i^\top] = \begin{cases}
    (\sigma_w^2 + i\sigma_e^2  )\mId, & \gamma = 1 \\
    ( \gamma^{n+1+i}\sigma_w^2 + \frac{ \gamma^{n+1+i} -  \gamma^{1+n-i}}{\gamma^2 -1}\sigma_e^2  )\mId , &  \gamma \neq 1
\end{cases}$ and then get
\begin{eqnarray}
    \label{Definition of D1}
    D_1 &\!\!\!\!=\!\!\!\!& \begin{cases}
    \sum_{i=1}^{n}\lambda^{n+1-i}(\sigma_w^2 + i\sigma_e^2  ), & \gamma = 1 \\
    \sum_{i=1}^{n}\lambda^{n+1-i}( \gamma^{n+1+i}\sigma_w^2 + \frac{ \gamma^{n+1+i} -  \gamma^{1+n-i}}{\gamma^2 -1}\sigma_e^2  ), &  \gamma \neq 1
\end{cases}\nonumber\\
&\!\!\!\!=\!\!\!\!&\begin{cases}
    n\sigma_w^2 + \frac{n(n+1)}{2}\sigma_e^2, & \lambda = \gamma = 1, \\
    \frac{\lambda - \lambda^{n+1}}{1-\lambda}\sigma_w^2 + \frac{\lambda^{n+2} - (n+1)\lambda^2 + n\lambda}{(1-\lambda)^2}\sigma_e^2, &  \lambda \neq 1, \gamma = 1,\\
    \frac{\gamma^{n+2} - \gamma^{2n+2}}{1 - \gamma}\sigma_w^2 + \frac{\gamma - \gamma^{n+1} - \gamma^{n+2} + \gamma^{2n+2}}{(1-\gamma)^2(1+\gamma)}\sigma_e^2, & \lambda = 1, \gamma \neq 1, \\
    \lambda^{2n+2}n\sigma_w^2 + \bigg(\frac{\lambda^2( 1 - \lambda^{2n}) }{(1-\lambda^2)^2} - \frac{\lambda^{2n+2}}{1-\lambda^2}n \bigg)\sigma_e^2, &  \lambda \neq 1, \gamma \neq 1, \lambda = \gamma,\\
    \frac{\gamma -  \gamma^{2n+1} }{\lambda - \gamma}\sigma_w^2 +\bigg( \frac{1}{1-\gamma^2}n - \frac{\gamma -  \gamma^{2n+1}}{(\lambda - \gamma)(1-\gamma^2)}  \bigg)\sigma_e^2, &  \lambda \neq 1, \gamma \neq 1, \lambda = 1/\gamma, \\
    \frac{\lambda^{n+1}\gamma^{n+2} - \lambda \gamma^{2n+2} }{\lambda - \gamma}\sigma_w^2 +\bigg( \frac{\lambda\gamma( 1 - \lambda^n\gamma^n) }{(1-\gamma^2)(1-\lambda\gamma)} - \frac{\lambda^{n+1}\gamma^{n+2} - \lambda \gamma^{2n+2} }{(\lambda - \gamma)(1-\gamma^2)}  \bigg)\sigma_e^2, &  \lambda \neq 1, \gamma \neq 1, \lambda \neq \gamma, \lambda \neq 1/\gamma.
\end{cases}
\end{eqnarray}

Combing \eqref{first term in derivative of new L} and \eqref{second term in derivative of new L}, \eqref{derivative of new L} can be rewritten as
\begin{eqnarray}
    \label{derivative of new L another form}
    \frac{\text{d}\vu(t)}{\text{d}t} = - \frac{1}{2}\E\bigg[\sum_{i=1}^{d+1}\sum_{j=1}^{d+1}\mT(i,j)\mU(i,j)\vec(\mT)  \bigg] + D_1\sum_{j=1}^{d} \bigg(\begin{bmatrix} {\bm 0}_{d\times d} & \mLambda(:,j) \\  \mLambda^\top(:,j)  & 0   \end{bmatrix} \otimes \begin{bmatrix} {\bm 0}_{d\times d} & \mLambda(:,j) \\  \mLambda^\top(:,j)  & 0   \end{bmatrix}\bigg)\vu.
\end{eqnarray}

\paragraph{Dynamics of $\mU$} Under the {Assumption} \ref{Assumption of initialization}, following the same analysis of \citep[Lemma 5.2]{zhang2024trained}, we can guarantee $\vu_{12}(t) = {\bm 0}_d\in\R^{d}$ and $\vu_{21}(t) = {\bm 0}_d\in\R^{d}$ for all time $t\geq 0$.

Now, we will respectively analyze dynamics of $\mU_{11}$ and $u_{-1}$ when  $\vu_{12}(t) = {\bm 0}_d\in\R^{d}$ and $\vu_{21}(t) = {\bm 0}_d\in\R^{d}$.

\paragraph{Dynamics of $\mU_{11}$} First, for $k,l\in[d]$, we use $\frac{\text{d}\vu(t)}{\text{d}t}((l-1)(d+1) + k) = \frac{\text{d}\mU(t)}{\text{d}t}(k,l)$ and then expand \eqref{first term in derivative of new L} as
\begin{eqnarray}
    \label{expansion of first term in the New L}
    &\!\!\!\!\!\!\!\!&\frac{1}{2}\E\bigg[\sum_{i=1}^{d+1}\sum_{j=1}^{d+1}\mT(i,j)\mU(i,j)\mT(k,l)  \bigg]\nonumber\\
    &\!\!\!\!=\!\!\!\!& \frac{1}{2}\E\bigg[\sum_{i=1}^{d}\sum_{j=1}^{d}\mT(i,j)\mU(i,j)\mT(k,l)  \bigg] + \frac{1}{2}\E[\mT(d+1,d+1)u_{-1}\mT(k,l)  ].
\end{eqnarray}

For the term $\mT(i,j)\mU(i,j)\mT(k,l)$, we have
{\small \begin{eqnarray}
    \label{expectation of first term in the expansion}
    &\!\!\!\!\!\!\!\!&\E[\mT(i,j)\mU(i,j)\mT(k,l) ]\nonumber\\
    &\!\!\!\!=\!\!\!\!& \mU(i,j) u_{-1}^2\E\bigg[ \bigg(\sum_{a=1}^n \lambda^{n+1-a}\vx_a(i) \vx_a^\top \vw_a \vx_{n+1}(j) \bigg)\bigg(\sum_{b=1}^n \lambda^{n+1-b} \vx_{n+1}(l) \vw_b^\top \vx_b \vx_b(k) \bigg) \bigg]\nonumber\\
    &\!\!\!\!=\!\!\!\!& \mU(i,j) u_{-1}^2 \mLambda(j,l) \E\bigg[ \bigg(\sum_{a=1}^n \lambda^{n+1-a}\vx_a(i) \vx_a^\top \vw_a  \bigg)\bigg(\sum_{b=1}^n \lambda^{n+1-b} \vw_b^\top \vx_b \vx_b(k) \bigg) \bigg]\nonumber\\
    &\!\!\!\!=\!\!\!\!& \begin{cases}
    \mU(i,j) u_{-1}^2 \mLambda(j,l)\E\big[\sum_{a=1}^n \sum_{b=1}^n \lambda^{2n+2-a - b}(\sigma_w^2 + \min\{a,b\}\sigma_e^2  ) \vx_a(i) \vx_a^\top \vx_b \vx_b(k)\big], & \gamma = 1, \\
    \mU(i,j) u_{-1}^2 \mLambda(j,l)\E\big[\sum_{a=1}^n \sum_{b=1}^n \lambda^{2n+2-a - b}( \gamma^{a+b}\sigma_w^2 + \frac{ \gamma^{a+b} -  \gamma^{|a-b|}}{\gamma^2 -1}\sigma_e^2  ) \vx_a(i) \vx_a^\top \vx_b \vx_b(k)\big], &  \gamma \neq 1,
\end{cases}\nonumber\\
\end{eqnarray}}
where the last line follows \Cref{covariance matrix of gaussian vector}. For the summation, we have
{\small \begin{eqnarray}
    \label{expectation of first term in the expansion 1}
    &\!\!\!\!\!\!\!\!&\frac{1}{2}\E\bigg[\sum_{i=1}^{d}\sum_{j=1}^{d}\mT(i,j)\mU(i,j)\mT(k,l)  \bigg]\nonumber\\
    &\!\!\!\!=\!\!\!\!& \begin{cases}
    \frac{1}{2}u_{-1}^2 \E\big[\sum_{a=1}^n \sum_{b=1}^n \lambda^{2n+2-a - b}(\sigma_w^2 + \min\{a,b\}\sigma_e^2  ) \vx_b(k)\vx_b^\top \vx_a \vx_a^\top \big] \mU_{11}\mLambda(:,l), & \gamma = 1, \\
    \frac{1}{2}u_{-1}^2 \E\big[\sum_{a=1}^n \sum_{b=1}^n \lambda^{2n+2-a - b}( \gamma^{a+b}\sigma_w^2 + \frac{ \gamma^{a+b} -  \gamma^{|a-b|}}{\gamma^2 -1}\sigma_e^2  ) \vx_b(k)\vx_b^\top \vx_a \vx_a^\top \big]\mU_{11}\mLambda(:,l), &  \gamma \neq 1.
\end{cases}\nonumber\\
\end{eqnarray}}

For the term $\mT(d+1,d+1)u_{-1}\mT(k,l)$, we get
{\small \begin{eqnarray}
    \label{expectation of first term in the expansion2}
    &\!\!\!\!\!\!\!\!&\frac{1}{2}\E[\mT(d+1,d+1)u_{-1}\mT(k,l) ]\nonumber\\
    &\!\!\!\!=\!\!\!\!& \frac{1}{2}u_{-1}^2\E\bigg[  \bigg(\sum_{a=1}^n\lambda^{n+1-a} \vw_a^\top \vx_a\vx_a^\top \mU_{11} \vx_{n+1}    \bigg) \bigg(\sum_{b=1}^n\lambda^{n+1-b} \vx_{n+1}(l) \vw_b^\top \vx_b\vx_b(k)  \bigg)   \bigg]\nonumber\\
    &\!\!\!\!=\!\!\!\!& \frac{1}{2}u_{-1}^2\E\bigg[  \bigg(\sum_{a=1}^n\sum_{b=1}^n\lambda^{2n+2-a-b} \vx_b(k) \vx_b^\top \vw_b \vw_a^\top \vx_a\vx_a^\top\bigg)\mU_{11} \vx_{n+1}\vx_{n+1}(l)   \bigg]\nonumber\\
    &\!\!\!\!=\!\!\!\!& \begin{cases}
    \frac{1}{2}u_{-1}^2 \E\big[\sum_{a=1}^n \sum_{b=1}^n \lambda^{2n+2-a - b}(\sigma_w^2 + \min\{a,b\}\sigma_e^2  ) \vx_b(k)\vx_b^\top \vx_a \vx_a^\top \big] \mU_{11}\mLambda(:,l), & \gamma = 1, \\
    \frac{1}{2}u_{-1}^2 \E\big[\sum_{a=1}^n \sum_{b=1}^n \lambda^{2n+2-a - b}( \gamma^{a+b}\sigma_w^2 + \frac{ \gamma^{a+b} -  \gamma^{|a-b|}}{\gamma^2 -1}\sigma_e^2  ) \vx_b(k)\vx_b^\top \vx_a \vx_a^\top \big]\mU_{11}\mLambda(:,l), &  \gamma \neq 1.
\end{cases}\nonumber\\
\end{eqnarray}}

Combing \eqref{expectation of first term in the expansion 1} and \eqref{expectation of first term in the expansion2}, we have
{\small \begin{eqnarray}
    \label{expansion of first term in the New L final}
    &\!\!\!\!\!\!\!\!&\frac{1}{2}\E\bigg[\sum_{i=1}^{d+1}\sum_{j=1}^{d+1}\mT(i,j)\mU(i,j)\mT(k,l)  \bigg]\nonumber\\
    &\!\!\!\!=\!\!\!\!& \begin{cases}
    u_{-1}^2 \E\big[\sum_{a=1}^n \sum_{b=1}^n \lambda^{2n+2-a - b}(\sigma_w^2 + \min\{a,b\}\sigma_e^2  ) \vx_b(k)\vx_b^\top \vx_a \vx_a^\top \big] \mU_{11}\mLambda(:,l), & \gamma = 1, \\
    u_{-1}^2 \E\big[\sum_{a=1}^n \sum_{b=1}^n \lambda^{2n+2-a - b}( \gamma^{a+b}\sigma_w^2 + \frac{ \gamma^{a+b} -  \gamma^{|a-b|}}{\gamma^2 -1}\sigma_e^2  ) \vx_b(k)\vx_b^\top \vx_a \vx_a^\top \big]\mU_{11}\mLambda(:,l), &  \gamma \neq 1.
\end{cases}\nonumber\\
\end{eqnarray}}

In addition, using \eqref{second term in derivative of new L} and following the analysis of \citep[Lemma 5.2]{zhang2024trained}, the $((l-1)(d+1) + k)$-th element of the second term in \eqref{derivative of new L another form} can be computed as $D_1 \mLambda^\top(:,k)\mLambda(:,l)u_{-1}$. Furthermore, we can obtain
{\small \begin{eqnarray*}
    \label{dunamics of U11}
    &\!\!\!\!\!\!\!\!&\frac{\text{d}\mU_{11}(t)}{\text{d}t}\nonumber\\
    &\!\!\!\!=\!\!\!\!& D_1u_{-1}\mLambda^2 -   \begin{cases}
    u_{-1}^2 \E\big[\sum_{a=1}^n \sum_{b=1}^n \lambda^{2n+2-a - b}(\sigma_w^2 + \min\{a,b\}\sigma_e^2  ) \vx_b\vx_b^\top \vx_a \vx_a^\top \big] \mU_{11}\mLambda, & \gamma = 1, \\
    u_{-1}^2 \E\big[\sum_{a=1}^n \sum_{b=1}^n \lambda^{2n+2-a - b}( \gamma^{a+b}\sigma_w^2 + \frac{ \gamma^{a+b} -  \gamma^{|a-b|}}{\gamma^2 -1}\sigma_e^2  ) \vx_b\vx_b^\top \vx_a \vx_a^\top \big]\mU_{11}\mLambda, &  \gamma \neq 1.
\end{cases}
\end{eqnarray*}}

Notice that based on \citep[Lemma A.2]{sayed2011adaptive}, $\E\big[ \vx_b\vx_b^\top \vx_a \vx_a^\top \big] = \begin{cases}
    \mLambda^2, & a \neq b \\
    2\mLambda^2 + \trace(\mLambda)\mLambda, &  a = b
\end{cases}$. We can further derive
{\small \begin{eqnarray}
    \label{statisctis of fourth-order Gaussian vector}
    &\!\!\!\!\!\!\!\!&\begin{cases}
    \E\big[\sum_{a=1}^n \sum_{b=1}^n \lambda^{2n+2-a - b}(\sigma_w^2 + \min\{a,b\}\sigma_e^2  ) \vx_b\vx_b^\top \vx_a \vx_a^\top \big], & \gamma = 1 \\
    \E\big[\sum_{a=1}^n \sum_{b=1}^n \lambda^{2n+2-a - b}( \gamma^{a+b}\sigma_w^2 + \frac{ \gamma^{a+b} -  \gamma^{|a-b|}}{\gamma^2 -1}\sigma_e^2  ) \vx_b\vx_b^\top \vx_a \vx_a^\top \big], &  \gamma \neq 1
\end{cases}\nonumber\\
    &\!\!\!\! = \!\!\!\!& \begin{cases}
    \sum_{a=1}^n\lambda^{2n+2-2a}(\sigma_w^2 + a\sigma_e^2  )(2\mLambda^2 + \trace(\mLambda)\mLambda) + 2\sum_{b=1}^{n-1}\sum_{a = b+1}^n \lambda^{2n+2-a - b}(\sigma_w^2 + b\sigma_e^2  )\mLambda^2, & \gamma = 1 \\
    \sum_{a=1}^n\lambda^{2n+2-2a}( \gamma^{2a}\sigma_w^2 + \frac{ \gamma^{2a} -  1}{\gamma^2 -1}\sigma_e^2  )(2\mLambda^2 + \trace(\mLambda)\mLambda)\\
    + 2\sum_{b=1}^{n-1}\sum_{a = b+1}^n \lambda^{2n+2-a - b}( \gamma^{a+b}\sigma_w^2 + \frac{ \gamma^{a+b} -  \gamma^{a-b}}{\gamma^2 -1}\sigma_e^2  )\mLambda^2, &  \gamma \neq 1
\end{cases}\nonumber\\
&\!\!\!\! = \!\!\!\!& \begin{cases}
    (n\sigma_w^2 + \frac{(n+1)n}{2}\sigma_e^2 )(2\mLambda^2 + \trace(\mLambda)\mLambda) + (n(n-1)\sigma_w^2+ \frac{(n-1)n(n+1)}{3}\sigma_e^2 )\mLambda^2, & \lambda = 1, \gamma = 1, \\
    \big(\frac{\lambda^2 - \lambda^{2n+2}}{1-\lambda^2}\sigma_w^2 + \frac{(n+1)\lambda^2 - \lambda^{2n+2} - n}{1-\lambda^2}\sigma_e^2  \big)(2\mLambda^2 + \trace(\mLambda)\mLambda)\\
    +\big((\frac{2(\lambda^{n+1} - \lambda^2)}{(1 - \lambda)^2} - \frac{2(\lambda^{2n+1} - \lambda^3)}{(1-\lambda)^2(1+\lambda)})\sigma_w^2 +   (\frac{2(n\lambda^4 - \lambda^{2n+4} - (n-1)\lambda^2 )}{(1-\lambda)(1-\lambda^2)^2} + \frac{2(\lambda^n - n\lambda +n-1)}{(1-\lambda)^2(1+\lambda) \lambda^{n-2}})\sigma_e^2    \big)\mLambda^2  , & \lambda \neq 1, \gamma =1,\\
    \big(\frac{\gamma^2 - \gamma^{2n+2}}{1-\gamma^2}\sigma_w^2 + \big( \frac{n}{1-\gamma^2} - \frac{\gamma^2 - \gamma^{2n+2}}{(1 - \gamma^2  )^2}   \big)\sigma_e^2     \big)(2\mLambda^2 + \trace(\mLambda)\mLambda)\\
    + \big( \big( 2\frac{\gamma^3 - \gamma^{2n+1}}{(1-\gamma)^2(1+\gamma)} - 2\frac{\gamma^{n+2} - \gamma^{2n+1}}{(1-\gamma)^2}   \big)\sigma_w^2 + \big(\frac{2}{\gamma^2 - 1}(\frac{\gamma^3 - \gamma^{2n+1}}{(1-\gamma)^2(1+\gamma)} - \frac{\gamma^{n+2} - \gamma^{2n+1}}{(1-\gamma)^2})\\
     - \frac{2\gamma}{(\gamma^2 - 1)(1-\gamma)}(n-1 - \frac{\gamma^n - \gamma}{\gamma-1})   \big)\sigma_e^2 \big)\mLambda^2, & \lambda = 1, \gamma \neq 1,\\
    \big(\lambda^{2n+2}n\sigma_w^2 - (\frac{n\lambda^{2n+2}}{1-\lambda^2} - \frac{\lambda^4 - \lambda^{2n+2}}{(1- \lambda^2 )^2}  )\sigma_e^2  \big)(2\mLambda^2 + \trace(\mLambda)\mLambda)\\
    +\big(\lambda^{2n+2}n(n-1)\sigma_w^2 + (\frac{2n(\lambda^4 - \lambda^{2n+2}) }{(1-\lambda^2)^2} - \frac{2(\lambda^{2n+4}  - n\lambda^6 + (n-1)\lambda^4)}{(1-\lambda^2)^3} - \frac{\lambda^{2n+2}n(n-1)}{1-\lambda^2} )\sigma_e^2    \big)\mLambda^2, &  \lambda \neq 1, \gamma \neq 1, \lambda = \gamma,\\
    \big( \frac{\lambda^{2n} -  \gamma^{2n}}{\lambda^2 - \gamma^2}\sigma_w^2 - (\frac{\lambda^{2n} -  \gamma^{2n}}{(\lambda^2 - \gamma^2)(1-\gamma^2)} - \frac{1 - \lambda^{2n-2}}{(1-\gamma^2)(1-\lambda^2)}    )\sigma_e^2   \big)(2\mLambda^2 + \trace(\mLambda)\mLambda)\\
    +\big((\frac{2\lambda^{2n} - 2\gamma^{2n-4}}{\lambda(\lambda - \gamma)^2(\lambda + \gamma)} - \frac{2-2\gamma^{2n-2}}{\lambda - \gamma} )\sigma_w^2   + \big( \frac{2\lambda^{2n} - 2\lambda^{2}}{(1 - \gamma^2)(\lambda - \gamma)^2} - \frac{2(n-1)\gamma\lambda^{2}}{(1 - \gamma^2)(\lambda -\gamma)}\\
    -  \frac{2\lambda^{2n} - 2\gamma^{2n-4}}{\lambda(\lambda - \gamma)^2(\lambda + \gamma)(1 - \gamma^2)} + \frac{2-2\gamma^{2n-2}}{(1 - \gamma^2)(\lambda - \gamma)}  \big)\sigma_e^2      \big)\mLambda^2 , &  \lambda \neq 1, \gamma \neq 1, \lambda = 1/\gamma, \\
    \big( \frac{\gamma^2\lambda^{2n+2} - \lambda^2 \gamma^{2n+2}}{\lambda^2 - \gamma^2}\sigma_w^2 - (\frac{\gamma^2\lambda^{2n+2} - \lambda^2 \gamma^{2n+2}}{(\lambda^2 - \gamma^2)(1-\gamma^2)} - \frac{\lambda^2 - \lambda^{2n+2}}{(1-\gamma^2)(1-\lambda^2)}    )\sigma_e^2   \big)(2\mLambda^2 + \trace(\mLambda)\mLambda)\\
    +\big((\frac{2\gamma^3\lambda^{2n+3} - 2\lambda^5\gamma^{2n+1}}{\lambda(\lambda - \gamma)^2(\lambda + \gamma)} - \frac{2\gamma^{n+2}\lambda^{n+2} - 2\gamma^{2n+1}\lambda^3}{\lambda - \gamma}     )\sigma_w^2  +  (\frac{2\gamma^{-1}(\lambda^4 - \lambda^{2n+2})}{(1-\gamma^2)(\lambda - \gamma)(1-  \lambda^2)}\\
    - \frac{2(\lambda^{3} - \lambda^{n+2}\gamma^{n-1})}{(1 - \gamma^2)(\lambda - \gamma)(1 - \lambda\gamma)} - \frac{2\gamma^3\lambda^{2n+3} - 2\lambda^5\gamma^{2n+1}}{\lambda(\lambda - \gamma)^2(\lambda + \gamma)(1 - \gamma^2)} + \frac{2\gamma^{n+2}\lambda^{n+2} - 2\gamma^{2n+1}\lambda^3}{(\lambda - \gamma)(1 - \gamma^2)}    )\sigma_e^2   \big)\mLambda^2, &  \lambda \neq 1, \gamma \neq 1, \lambda \neq \gamma, \lambda \neq 1/\gamma.
\end{cases}
\end{eqnarray}}
Now, we can obtain the dynamics of $\mU_{11}$ as following
\begin{eqnarray}
    \label{final dynamics of U11}
    \frac{\text{d}\mU_{11}(t)}{\text{d}t}=  -\vu_{-1}^2\wt\mLambda\mLambda \mU_{11}\mLambda + D_1u_{-1}\mLambda^2,
\end{eqnarray}
where $\wt\mLambda = D_2(2\mLambda +\trace(\mLambda)\mId ) + D_3\mLambda $ with
{\small\begin{eqnarray*}
D_2 =\begin{cases}
    n\sigma_w^2 + \frac{(n+1)n}{2}\sigma_e^2 , & \lambda = 1, \gamma = 1, \\
    \frac{\lambda^2 - \lambda^{2n+2}}{1-\lambda^2}\sigma_w^2 + \frac{(n+1)\lambda^2 - \lambda^{2n+2} - n}{1-\lambda^2}\sigma_e^2  , & \lambda \neq 1, \gamma =1,\\
    \frac{\gamma^2 - \gamma^{2n+2}}{1-\gamma^2}\sigma_w^2 + \big( \frac{n}{1-\gamma^2} - \frac{\gamma^2 - \gamma^{2n+2}}{(1 - \gamma^2  )^2}   \big)\sigma_e^2, & \lambda = 1, \gamma \neq 1,\\
    \lambda^{2n+2}n\sigma_w^2 - (\frac{n\lambda^{2n+2}}{1-\lambda^2} - \frac{\lambda^4 - \lambda^{2n+2}}{(1- \lambda^2 )^2}  )\sigma_e^2, &  \lambda \neq 1, \gamma \neq 1, \lambda = \gamma,\\
    \frac{\lambda^{2n} -  \gamma^{2n}}{\lambda^2 - \gamma^2}\sigma_w^2 - (\frac{\lambda^{2n} -  \gamma^{2n}}{(\lambda^2 - \gamma^2)(1-\gamma^2)} - \frac{1 - \lambda^{2n-2}}{(1-\gamma^2)(1-\lambda^2)}    )\sigma_e^2, &  \lambda \neq 1, \gamma \neq 1, \lambda = 1/\gamma, \\
    \frac{\gamma^2\lambda^{2n+2} - \lambda^2 \gamma^{2n+2}}{\lambda^2 - \gamma^2}\sigma_w^2 - (\frac{\gamma^2\lambda^{2n+2} - \lambda^2 \gamma^{2n+2}}{(\lambda^2 - \gamma^2)(1-\gamma^2)} - \frac{\lambda^2 - \lambda^{2n+2}}{(1-\gamma^2)(1-\lambda^2)}    )\sigma_e^2, &  \lambda \neq 1, \gamma \neq 1, \lambda \neq \gamma, \lambda \neq 1/\gamma,
\end{cases}
\end{eqnarray*}}
and
{\small \begin{eqnarray*}
D_3 =  \begin{cases}
    n(n-1)\sigma_w^2+ \frac{(n-1)n(n+1)}{3}\sigma_e^2, & \lambda = 1, \gamma = 1, \\
    (\frac{2(\lambda^{n+1} - \lambda^2)}{(1 - \lambda)^2} - \frac{2(\lambda^{2n+1} - \lambda^3)}{(1-\lambda)^2(1+\lambda)})\sigma_w^2 +   (\frac{2(n\lambda^4 - \lambda^{2n+4} - (n-1)\lambda^2 )}{(1-\lambda)(1-\lambda^2)^2} + \frac{2(\lambda^n - n\lambda +n-1)}{(1-\lambda)^2(1+\lambda) \lambda^{n-2}})\sigma_e^2   , & \lambda \neq 1, \gamma =1,\\
    \big( 2\frac{\gamma^3 - \gamma^{2n+1}}{(1-\gamma)^2(1+\gamma)} - 2\frac{\gamma^{n+2} - \gamma^{2n+1}}{(1-\gamma)^2}   \big)\sigma_w^2 + \big(\frac{2}{\gamma^2 - 1}(\frac{\gamma^3 - \gamma^{2n+1}}{(1-\gamma)^2(1+\gamma)} - \frac{\gamma^{n+2} - \gamma^{2n+1}}{(1-\gamma)^2})\\
     - \frac{2\gamma}{(\gamma^2 - 1)(1-\gamma)}(n-1 - \frac{\gamma^n - \gamma}{\gamma-1})   \big)\sigma_e^2, & \lambda = 1, \gamma \neq 1,\\
    \lambda^{2n+2}n(n-1)\sigma_w^2 + (\frac{2n(\lambda^4 - \lambda^{2n+2}) }{(1-\lambda^2)^2} - \frac{2(\lambda^{2n+4}  - n\lambda^6 + (n-1)\lambda^4)}{(1-\lambda^2)^3} - \frac{\lambda^{2n+2}n(n-1)}{1-\lambda^2} )\sigma_e^2, &  \lambda \neq 1, \gamma \neq 1, \lambda = \gamma,\\
    (\frac{2\lambda^{2n} - 2\gamma^{2n-4}}{\lambda(\lambda - \gamma)^2(\lambda + \gamma)} - \frac{2-2\gamma^{2n-2}}{\lambda - \gamma} )\sigma_w^2   + \big( \frac{2\lambda^{2n} - 2\lambda^{2}}{(1 - \gamma^2)(\lambda - \gamma)^2} - \frac{2(n-1)\gamma\lambda^{2}}{(1 - \gamma^2)(\lambda -\gamma)}\\
    -  \frac{2\lambda^{2n} - 2\gamma^{2n-4}}{\lambda(\lambda - \gamma)^2(\lambda + \gamma)(1 - \gamma^2)} + \frac{2-2\gamma^{2n-2}}{(1 - \gamma^2)(\lambda - \gamma)}  \big)\sigma_e^2, &  \lambda \neq 1, \gamma \neq 1, \lambda = 1/\gamma, \\
    (\frac{2\gamma^3\lambda^{2n+3} - 2\lambda^5\gamma^{2n+1}}{\lambda(\lambda - \gamma)^2(\lambda + \gamma)} - \frac{2\gamma^{n+2}\lambda^{n+2} - 2\gamma^{2n+1}\lambda^3}{\lambda - \gamma}     )\sigma_w^2  +  (\frac{2\gamma^{-1}(\lambda^4 - \lambda^{2n+2})}{(1-\gamma^2)(\lambda - \gamma)(1-  \lambda^2)}\\
    - \frac{2(\lambda^{3} - \lambda^{n+2}\gamma^{n-1})}{(1 - \gamma^2)(\lambda - \gamma)(1 - \lambda\gamma)} - \frac{2\gamma^3\lambda^{2n+3} - 2\lambda^5\gamma^{2n+1}}{\lambda(\lambda - \gamma)^2(\lambda + \gamma)(1 - \gamma^2)} + \frac{2\gamma^{n+2}\lambda^{n+2} - 2\gamma^{2n+1}\lambda^3}{(\lambda - \gamma)(1 - \gamma^2)}    )\sigma_e^2, &  \lambda \neq 1, \gamma \neq 1, \lambda \neq \gamma, \lambda \neq 1/\gamma.
\end{cases}
\end{eqnarray*}}

\paragraph{Dynamics of $u_{-1}$} Next, similar with \eqref{expansion of first term in the New L}, we expand \eqref{first term in derivative of new L} at $(d+1)^2$-th element:
\begin{eqnarray}
    \label{expansion of first term in the New L last element}
    &\!\!\!\!\!\!\!\!&\frac{1}{2}\E\bigg[\sum_{i=1}^{d+1}\sum_{j=1}^{d+1}\mT(i,j)\mU(i,j)\mT(d+1,d+1)  \bigg]\nonumber\\
    &\!\!\!\!=\!\!\!\!& \frac{1}{2}\E\bigg[\sum_{i=1}^{d}\sum_{j=1}^{d}\mT(i,j)\mU(i,j)\mT(d+1,d+1)  \bigg] + \frac{1}{2}\E[\mT(d+1,d+1)u_{-1}\mT(d+1,d+1)  ]\nonumber\\
    &\!\!\!\!=\!\!\!\!& \frac{u_{-1}}{2}\sum_{i=1}^{n}\sum_{j=1}^{n}\mU(i,j)\E\bigg[  \bigg(\sum_{a=1}^n \lambda^{n+1-a }\vx_a(i)\vx_a^\top\vw_a\vx_{n+1}(j)\bigg) \bigg(\sum_{b=1}^n \lambda^{n+1-b }  \vw_b^\top\vx_b\vx_b^\top \mU_{11}\vx_{n+1}  \bigg)      \bigg]\nonumber\\
    &\!\!\!\!\!\!\!\!& + \frac{u_{-1}}{2}\E \bigg[  \bigg(\sum_{a=1}^n \lambda^{n+1-a}  \vw_a^\top\vx_a\vx_a^\top \mU_{11}\vx_{n+1}   \bigg) \bigg(\sum_{b=1}^n \lambda^{n+1-b}
    \vx_{n+1}^\top\mU_{11}^\top\vx_b\vx_b^\top\vw_b   \bigg)  \bigg]\nonumber\\
    &\!\!\!\!=\!\!\!\!&  \frac{u_{-1}}{2}\trace\bigg(\sum_{i=1}^{n}\sum_{j=1}^{n} \mLambda(:,j)\mU(i,j)\E\bigg[\sum_{a=1}^n\sum_{b=1}^n\lambda^{2n+2-a-b} \vx_a(i) \vx_a^\top \vw_a \vw_b^\top \vx_b\vx_b^\top   \bigg]\mU_{11}    \bigg)\nonumber\\
    &\!\!\!\!\!\!\!\!& +  \frac{u_{-1}}{2} \E\bigg[ \trace\bigg(  \bigg(\sum_{a=1}^n \lambda^{n+1-a}  \vw_a^\top\vx_a\vx_a^\top    \bigg) \mU_{11}\mLambda\mU_{11}^\top \bigg(\sum_{b=1}^n \lambda^{n+1-b} \vx_b\vx_b^\top\vw_b   \bigg)  \bigg)  \bigg]\nonumber\\
    &\!\!\!\!=\!\!\!\!& \frac{u_{-1}}{2} \trace\bigg(  \mLambda \mU_{11}^\top  \E\bigg[\sum_{a=1}^n\sum_{b=1}^n\lambda^{2n+2-a-b} \vx_a \vx_a^\top \vw_a \vw_b^\top \vx_b\vx_b^\top   \bigg]\mU_{11}   \bigg) \nonumber\\
    &\!\!\!\!\!\!\!\!&+  \frac{u_{-1}}{2} \E\bigg[  \trace\bigg(  \sum_{a=1}^n\sum_{b=1}^n\lambda^{2n+2-a-b} \vx_b\vx_b^\top\vw_b\vw_a^\top \vx_a \vx_a^\top \mU_{11}\mLambda\mU_{11}^\top  \bigg)   \bigg]\nonumber\\
    &\!\!\!\!=\!\!\!\!& u_{-1} \trace(\wt\mLambda \mLambda\mU_{11}\mLambda\mU_{11}^\top ),
\end{eqnarray}
where the last line follows \eqref{expectation of first term in the expansion2} and \eqref{final dynamics of U11}.

In addition, according to \citep[Lemma 5.2]{zhang2024trained}, the last element of the second term in \eqref{derivative of new L another form} can be represented as $D_1\trace(\mLambda^2\mU_{11}^\top)$. Hence, we have
\begin{eqnarray}
    \label{dynamics of u-1 final}
    \frac{\text{d}u_{-1}(t)}{\text{d}t} = -\trace(u_{-1}\wt\mLambda \mLambda \mU_{11}\mLambda \mU_{11}^\top ) + D_1\trace(\mLambda^2\mU_{11}^\top).
\end{eqnarray}

\end{proof}

Based on the same analysis of \citep[Lemma A.1]{zhang2024trained}, we have
\begin{lemma}
\label{Lemma: Different loss function}
Consider the gradient flow of $L$ in \eqref{definition of loss function} with respect to $\vu$, initialized according to Assumption~\ref{Assumption of initialization}. This is equivalent to performing gradient flow with respect to $\mU_{11}$ and $u_{-1}$ on the same loss function
\begin{eqnarray}
    \label{loss function another form}
    \wt L(\mU_{11},u_{-1}) = \frac{u_{-1}^2}{2}\trace(\wt\mLambda\mLambda\mU_{11}\mLambda\mU_{11}^\top) - D_1u_{-1}\trace(\mLambda^2\mU_{11}^\top),
\end{eqnarray}
where $\wt\mLambda$ and $D_1$ are defined in \Cref{dynamics of population loss}.
\end{lemma}

Note that $\wt L(\mU_{11}, u_{-1})$ differs from $L(\vu)$ by a constant and can therefore take negative values. We can further derive

\begin{cor}
\label{Corollary: minimum of Different loss function}
The loss function $\wt L(\mU_{11},u_{-1})$ in \Cref{Lemma: Different loss function} satisfies
\begin{eqnarray}
    \label{diff loss function minimum}
    \min_{\mU_{11}\in\R^{d\times d}, \atop u_{-1}\in\R}\wt L(\mU_{11},u_{-1}) = -\frac{D_1^2}{2}\trace(\mLambda^2 \wt\mLambda^{-1} ),
\end{eqnarray}
and
\begin{eqnarray}
    \label{diff loss function minimum11}
    \wt L(\mU_{11},u_{-1}) - \min_{\mU_{11}\in\R^{d\times d}, \atop u_{-1}\in\R}\wt L(\mU_{11},u_{-1}) = \frac{1}{2}\| \wt\mLambda^{\frac{1}{2}}( u_{-1}\mLambda^{\frac{1}{2}}\mU_{11}\mLambda^{\frac{1}{2}} - D_1\mLambda\wt\mLambda^{-1} )  \|_F^2.
\end{eqnarray}

Furthermore, $L$ in \eqref{definition of loss function} satisfies that
\begin{eqnarray}
    \label{diff original loss function minimum11}
    L(\mU_{11},u_{-1}) - \min_{\mU_{11}\in\R^{d\times d}, \atop u_{-1}\in\R} L(\mU_{11},u_{-1}) = \frac{1}{2}\| \wt\mLambda^{\frac{1}{2}}( u_{-1}\mLambda^{\frac{1}{2}}\mU_{11}\mLambda^{\frac{1}{2}} - D_1\mLambda\wt\mLambda^{-1} )  \|_F^2.
\end{eqnarray}
and the global minimum $(\mU_{11}, u_{-1})$ of $L$ in \eqref{definition of loss function}, when initialized according to Assumption~\ref{Assumption of initialization},  satisfies
\begin{eqnarray}
    \label{property of global minimum}
    u_{-1}\mU_{11} = D_1 \wt\mLambda^{-1}.
\end{eqnarray}

\end{cor}

\begin{proof}
First, we testify $\frac{1}{2}\trace(\wt\mLambda (u_{-1}\mLambda^{\frac{1}{2}}\mU_{11}\mLambda^{\frac{1}{2}} - D_1\mLambda\wt\mLambda^{-1}   )(u_{-1}\mLambda^{\frac{1}{2}}\mU_{11}\mLambda^{\frac{1}{2}} - D_1\mLambda\wt\mLambda^{-1}   )^\top) - \frac{D_1^2}{2}\trace(\mLambda^2 \wt\mLambda^{-1} ) $ is equivalent to $\wt L(\mU_{11},u_{-1})$. Specifically, we have
\begin{eqnarray}
    \label{another form of loss function}
    &\!\!\!\!\!\!\!\!&\frac{1}{2}\trace\bigg(\wt\mLambda (u_{-1}\mLambda^{\frac{1}{2}}\mU_{11}\mLambda^{\frac{1}{2}} - D_1\mLambda\wt\mLambda^{-1}   )(u_{-1}\mLambda^{\frac{1}{2}}\mU_{11}\mLambda^{\frac{1}{2}} - D_1\mLambda\wt\mLambda^{-1}   )^\top\bigg)\nonumber\\
    &\!\!\!\!\!\!\!\!&- \frac{D_1^2}{2}\trace(\mLambda^2 \wt\mLambda^{-1} )\nonumber\\
    &\!\!\!\! = \!\!\!\!& \frac{1}{2}\trace\bigg( \wt\mLambda  ( u_{-1}^2 \mLambda^{\frac{1}{2}}\mU_{11}\mLambda\mU_{11}^\top\mLambda^{\frac{1}{2}} - D_1u_{-1}\mLambda\wt\mLambda^{-1}\mLambda^{\frac{1}{2}}\mU_{11}\mLambda^{\frac{1}{2}}\nonumber\\
    &\!\!\!\!\!\!\!\!& - D_1u_1\mLambda^{\frac{1}{2}}\mU_{11}\mLambda^{\frac{3}{2}}\wt\mLambda^{-1} + D_1^2\wt\mLambda^{-2}\mLambda^2   )  \bigg) - \frac{D_1^2}{2}\trace(\mLambda^2 \wt\mLambda^{-1} )\nonumber\\
    &\!\!\!\! = \!\!\!\!& \frac{u_{-1}^2}{2}\trace(\wt\mLambda\mLambda\mU_{11}\mLambda\mU_{11}^\top) - D_1u_{-1}\trace(\mLambda^2\mU_{11}^\top)\nonumber\\
    &\!\!\!\! = \!\!\!\!& \wt L(\mU_{11},u_{-1}),
\end{eqnarray}
where the second equation uses the fact that $\wt\mLambda$ and $\mLambda$ commute.

Notice that $\wt\mLambda $ and $(u_{-1}\mLambda^{\frac{1}{2}}\mU_{11}\mLambda^{\frac{1}{2}} - D_1\mLambda\wt\mLambda^{-1}   )(u_{-1}\mLambda^{\frac{1}{2}}\mU_{11}\mLambda^{\frac{1}{2}} - D_1\mLambda\wt\mLambda^{-1}   )^\top$ are positive semidefinite matrices, we have $\frac{1}{2}\trace(\wt\mLambda (u_{-1}\mLambda^{\frac{1}{2}}\mU_{11}\mLambda^{\frac{1}{2}} - D_1\mLambda\wt\mLambda^{-1}   )(u_{-1}\mLambda^{\frac{1}{2}}\mU_{11}\mLambda^{\frac{1}{2}} - D_1\mLambda\wt\mLambda^{-1}   )^\top)\geq 0$. Hence, we have $\min_{\mU_{11}\in\R^{d\times d}, \atop u_{-1}\in\R}\wt L(\mU_{11},u_{-1}) = -\frac{D_1^2}{2}\trace(\mLambda^2 \wt\mLambda^{-1} )$. In addition, \eqref{diff loss function minimum11} follows $\trace(\mA^\top\mA) = \|\mA\|_F^2$ for any matrix $\mA$. Since $L(\mU_{11},u_{-1}) = \wt L(\mU_{11},u_{-1}) + C$ for some constant $C$, condition \eqref{diff original loss function minimum11} is satisfied.

To attain the global minimum characterized by \eqref{diff loss function minimum}, it is necessary that $u_{-1}\mLambda^{\frac{1}{2}}\mU_{11}\mLambda^{\frac{1}{2}} - D_1\mLambda\wt\mLambda^{-1}    = {\bm 0}$. Using the identity $\mLambda\wt\mLambda^{-1} = \mLambda^{\frac{1}{2}}\wt\mLambda^{-1}\mLambda^{\frac{1}{2}}$, this condition reduces to
$u_{-1}\mU_{11} = D_1 \wt\mLambda^{-1}$.

\end{proof}

Following the same analysis of \citep[Lemma A.3, Lemma A.4 and Lemma A.5]{zhang2024trained}, we can directly obtain
\begin{lemma}
\label{Lemma: relationshiop u1 U11}
Consider the gradient flow of $L$ in \eqref{definition of loss function} with respect to $\vu$, initialized according to Assumption~\ref{Assumption of initialization}.  For any $t\geq 0$, it holds that
\begin{eqnarray}
    \label{relationshiop u1 U11}
    u_{-1}^2(t) = \trace(\mU_{11}(t)\mU_{11}^\top(t)).
\end{eqnarray}

\end{lemma}

\begin{lemma}
\label{Lemma: nonnegativitity of u1}
Consider the gradient flow of $L$ in \eqref{definition of loss function} with respect to $\vu$, initialized according to Assumption~\ref{Assumption of initialization}.  If the initial scale satisfies
\begin{eqnarray}
    \label{requirement of sigma}
    0< \sigma < \sqrt{\frac{2D_1}{\sqrt{d}\|\wt\mLambda\|}},
\end{eqnarray}
then, for any $t\geq 0$, it holds that
\begin{eqnarray}
    \label{u-1 >0}
    u_{-1}(t) \geq \sqrt{\frac{\sigma^2}{2D_1\sqrt{d}\|\mLambda\|^2}\|\mLambda\mTheta \|_F^2(2D_1 - \sqrt{d}\sigma^2\|\wt\mLambda\|) } >0.
\end{eqnarray}
\end{lemma}

\section{Proof of \Cref{Theorem of global minimum simplified}}
\label{proof of Theorem of global minimum}

The following theorem is established under a more general setting with $\lambda, \gamma > 0$.
\begin{theorem}
\label{Theorem of global minimum}
Consider gradient flow over the population loss in \eqref{definition of loss function}. Assume that the initial task weight $\vw_0\overset{\text{i.i.d.}}{\sim}\calN({\bm 0}, \sigma_w^2\mId)$, noises $\ve_i\overset{\text{i.i.d.}}{\sim}\calN({\bm 0}, \sigma_e^2\mId)$ and inputs $\vx_i\overset{\text{i.i.d.}}{\sim}\calN({\bm 0}, {\bm \Lambda})$.  Suppose the initialization satisfies {Assumption} \ref{Assumption of initialization}  with
initialization scale $\sigma>0$  satisfying $\sigma < \sqrt{\frac{2D_1}{\sqrt{d}\|\wt\mLambda\|}}$ where
{\small \begin{eqnarray*}
D_1 =\begin{cases}
    n\sigma_w^2 + \frac{n(n+1)}{2}\sigma_e^2, & \lambda = \gamma = 1, \\
    \frac{\lambda - \lambda^{n+1}}{1-\lambda}\sigma_w^2 + \frac{\lambda^{n+2} - (n+1)\lambda^2 + n\lambda}{(1-\lambda)^2}\sigma_e^2, &  \lambda \neq 1, \gamma = 1,\\
    \frac{\gamma^{n+2} - \gamma^{2n+2}}{1 - \gamma}\sigma_w^2 + \frac{\gamma - \gamma^{n+1} - \gamma^{n+2} + \gamma^{2n+2}}{(1-\gamma)^2(1+\gamma)}\sigma_e^2, & \lambda = 1, \gamma \neq 1, \\
    \lambda^{2n+2}n\sigma_w^2 + \bigg(\frac{\lambda^2( 1 - \lambda^{2n}) }{(1-\lambda^2)^2} - \frac{\lambda^{2n+2}}{1-\lambda^2}n \bigg)\sigma_e^2, &  \lambda \neq 1, \gamma \neq 1, \lambda = \gamma,\\
    \frac{\gamma -  \gamma^{2n+1} }{\lambda - \gamma}\sigma_w^2 +\bigg( \frac{1}{1-\gamma^2}n - \frac{\gamma -  \gamma^{2n+1}}{(\lambda - \gamma)(1-\gamma^2)}  \bigg)\sigma_e^2, &  \lambda \neq 1, \gamma \neq 1, \lambda = 1/\gamma, \\
    \frac{\lambda^{n+1}\gamma^{n+2} - \lambda \gamma^{2n+2} }{\lambda - \gamma}\sigma_w^2 +\bigg( \frac{\lambda\gamma( 1 - \lambda^n\gamma^n) }{(1-\gamma^2)(1-\lambda\gamma)} - \frac{\lambda^{n+1}\gamma^{n+2} - \lambda \gamma^{2n+2} }{(\lambda - \gamma)(1-\gamma^2)}  \bigg)\sigma_e^2, &  \lambda \neq 1, \gamma \neq 1, \lambda \neq \gamma, \lambda \neq 1/\gamma,
\end{cases}
\end{eqnarray*}}
and $\wt\mLambda = D_2(2\mLambda +\trace(\mLambda)\mId ) + D_3\mLambda $ with
{\small  \begin{eqnarray*}
D_2 =\begin{cases}
    n\sigma_w^2 + \frac{(n+1)n}{2}\sigma_e^2 , & \lambda = 1, \gamma = 1, \\
    \frac{\lambda^2 - \lambda^{2n+2}}{1-\lambda^2}\sigma_w^2 + \frac{(n+1)\lambda^2 - \lambda^{2n+2} - n}{1-\lambda^2}\sigma_e^2  , & \lambda \neq 1, \gamma =1,\\
    \frac{\gamma^2 - \gamma^{2n+2}}{1-\gamma^2}\sigma_w^2 + \big( \frac{n}{1-\gamma^2} - \frac{\gamma^2 - \gamma^{2n+2}}{(1 - \gamma^2  )^2}   \big)\sigma_e^2, & \lambda = 1, \gamma \neq 1,\\
    \lambda^{2n+2}n\sigma_w^2 - (\frac{n\lambda^{2n+2}}{1-\lambda^2} - \frac{\lambda^4 - \lambda^{2n+2}}{(1- \lambda^2 )^2}  )\sigma_e^2, &  \lambda \neq 1, \gamma \neq 1, \lambda = \gamma,\\
    \frac{\lambda^{2n} -  \gamma^{2n}}{\lambda^2 - \gamma^2}\sigma_w^2 - (\frac{\lambda^{2n} -  \gamma^{2n}}{(\lambda^2 - \gamma^2)(1-\gamma^2)} - \frac{1 - \lambda^{2n-2}}{(1-\gamma^2)(1-\lambda^2)}    )\sigma_e^2, &  \lambda \neq 1, \gamma \neq 1, \lambda = 1/\gamma, \\
    \frac{\gamma^2\lambda^{2n+2} - \lambda^2 \gamma^{2n+2}}{\lambda^2 - \gamma^2}\sigma_w^2 - (\frac{\gamma^2\lambda^{2n+2} - \lambda^2 \gamma^{2n+2}}{(\lambda^2 - \gamma^2)(1-\gamma^2)} - \frac{\lambda^2 - \lambda^{2n+2}}{(1-\gamma^2)(1-\lambda^2)}    )\sigma_e^2, &  \lambda \neq 1, \gamma \neq 1, \lambda \neq \gamma, \lambda \neq 1/\gamma,
\end{cases}
\end{eqnarray*}}
and
{\small \begin{eqnarray*}
D_3 =  \begin{cases}
    n(n-1)\sigma_w^2+ \frac{(n-1)n(n+1)}{3}\sigma_e^2, & \lambda = 1, \gamma = 1, \\
    (\frac{2(\lambda^{n+1} - \lambda^2)}{(1 - \lambda)^2} - \frac{2(\lambda^{2n+1} - \lambda^3)}{(1-\lambda)^2(1+\lambda)})\sigma_w^2 +   (\frac{2(n\lambda^4 - \lambda^{2n+4} - (n-1)\lambda^2 )}{(1-\lambda)(1-\lambda^2)^2} + \frac{2(\lambda^n - n\lambda +n-1)}{(1-\lambda)^2(1+\lambda) \lambda^{n-2}})\sigma_e^2   , & \lambda \neq 1, \gamma =1,\\
    \big( 2\frac{\gamma^3 - \gamma^{2n+1}}{(1-\gamma)^2(1+\gamma)} - 2\frac{\gamma^{n+2} - \gamma^{2n+1}}{(1-\gamma)^2}   \big)\sigma_w^2 + \big(\frac{2}{\gamma^2 - 1}(\frac{\gamma^3 - \gamma^{2n+1}}{(1-\gamma)^2(1+\gamma)} - \frac{\gamma^{n+2} - \gamma^{2n+1}}{(1-\gamma)^2})\\
     - \frac{2\gamma}{(\gamma^2 - 1)(1-\gamma)}(n-1 - \frac{\gamma^n - \gamma}{\gamma-1})   \big)\sigma_e^2, & \lambda = 1, \gamma \neq 1,\\
    \lambda^{2n+2}n(n-1)\sigma_w^2 + (\frac{2n(\lambda^4 - \lambda^{2n+2}) }{(1-\lambda^2)^2} - \frac{2(\lambda^{2n+4}  - n\lambda^6 + (n-1)\lambda^4)}{(1-\lambda^2)^3} - \frac{\lambda^{2n+2}n(n-1)}{1-\lambda^2} )\sigma_e^2, &  \lambda \neq 1, \gamma \neq 1, \lambda = \gamma,\\
    (\frac{2\lambda^{2n} - 2\gamma^{2n-4}}{\lambda(\lambda - \gamma)^2(\lambda + \gamma)} - \frac{2-2\gamma^{2n-2}}{\lambda - \gamma} )\sigma_w^2   + \big( \frac{2\lambda^{2n} - 2\lambda^{2}}{(1 - \gamma^2)(\lambda - \gamma)^2} - \frac{2(n-1)\gamma\lambda^{2}}{(1 - \gamma^2)(\lambda -\gamma)}\\
    -  \frac{2\lambda^{2n} - 2\gamma^{2n-4}}{\lambda(\lambda - \gamma)^2(\lambda + \gamma)(1 - \gamma^2)} + \frac{2-2\gamma^{2n-2}}{(1 - \gamma^2)(\lambda - \gamma)}  \big)\sigma_e^2, &  \lambda \neq 1, \gamma \neq 1, \lambda = 1/\gamma, \\
    (\frac{2\gamma^3\lambda^{2n+3} - 2\lambda^5\gamma^{2n+1}}{\lambda(\lambda - \gamma)^2(\lambda + \gamma)} - \frac{2\gamma^{n+2}\lambda^{n+2} - 2\gamma^{2n+1}\lambda^3}{\lambda - \gamma}     )\sigma_w^2  +  (\frac{2\gamma^{-1}(\lambda^4 - \lambda^{2n+2})}{(1-\gamma^2)(\lambda - \gamma)(1-  \lambda^2)}\\
    - \frac{2(\lambda^{3} - \lambda^{n+2}\gamma^{n-1})}{(1 - \gamma^2)(\lambda - \gamma)(1 - \lambda\gamma)} - \frac{2\gamma^3\lambda^{2n+3} - 2\lambda^5\gamma^{2n+1}}{\lambda(\lambda - \gamma)^2(\lambda + \gamma)(1 - \gamma^2)} + \frac{2\gamma^{n+2}\lambda^{n+2} - 2\gamma^{2n+1}\lambda^3}{(\lambda - \gamma)(1 - \gamma^2)}    )\sigma_e^2, &  \lambda \neq 1, \gamma \neq 1, \lambda \neq \gamma, \lambda \neq 1/\gamma.
\end{cases}
\end{eqnarray*}}
Then gradient flow converges to a global minimum of the population loss \eqref{definition of loss function}. Moreover, $\mW_{KQ}(0)$ and $\mW_V(0)$ respectively converge to
\begin{eqnarray}
    \label{out of WV global minimum}
    \lim_{t\to\infty}\mW_V(t) = \sqrt{D_1\|\wt\mLambda^{-1} \|_F}\begin{bmatrix}{\bm 0}_{d\times d} &   {\bm 0}_{d} \\ {\bm 0}_{d}^\top   & 1 \end{bmatrix} \ \ \text{and} \ \ \lim_{t\to\infty}\mW_{KQ}(t) = \sqrt{D_1\|\wt\mLambda^{-1} \|_F^{-1}}\begin{bmatrix}\wt\mLambda^{-1} &   {\bm 0}_{d} \\ {\bm 0}_{d}^\top   & 0 \end{bmatrix}.
\end{eqnarray}
\end{theorem}

\begin{proof}
We begin by deriving a lower bound on the squared gradient norm of the loss function
\begin{eqnarray}
    \label{lower bound of gradient}
    &\!\!\!\!\!\!\!\!&\|\nabla L(\mU_{11}(t),u_{-1}(t))  \|_2^2\nonumber\\
    &\!\!\!\!=\!\!\!\!& \bigg\|\frac{\partial L(\mU_{11}(t),u_{-1}(t))}{\partial \mU_{11}(t)}\bigg\|_F^2 + \bigg|\frac{\partial L(\mU_{11}(t),u_{-1}(t))}{\partial u_{-1}(t)}\bigg|^2 \geq \bigg\|\frac{\partial L(\mU_{11}(t),u_{-1}(t))}{\partial \mU_{11}(t)}\bigg\|_F^2\nonumber\\
    &\!\!\!\!=\!\!\!\!& \|  u_{-1}^2\wt\mLambda\mLambda \mU_{11}\mLambda - D_1u_{-1}\mLambda^2 \|_F^2\nonumber\\
    &\!\!\!\!=\!\!\!\!& u_{-1}^2\|\wt\mLambda\mLambda^{\frac{1}{2}}(u_{-1}\mLambda^{\frac{1}{2}}\mU_{11}\mLambda^{\frac{1}{2}} - \mLambda\wt\mLambda^{-1}   ) \mLambda^{\frac{1}{2}}  \|_F^2\nonumber\\
    &\!\!\!\!\geq\!\!\!\!&  \frac{\sigma^2}{2D_1\sqrt{d}\|\mLambda\|^2}\|\mLambda\mTheta \|_F^2(2D_1 - \sqrt{d}\sigma^2\|\wt\mLambda\|)\|\wt\mLambda\mLambda^{\frac{1}{2}}(u_{-1}\mLambda^{\frac{1}{2}}\mU_{11}\mLambda^{\frac{1}{2}} - \mLambda\wt\mLambda^{-1}   ) \mLambda^{\frac{1}{2}}  \|_F^2,
\end{eqnarray}
where the third equation uses that $\wt\mLambda$ and $\mLambda$ commute and the last line follows \Cref{Lemma: nonnegativitity of u1}.

In addition, based on {Corollary} \ref{Corollary: minimum of Different loss function}, we have
\begin{eqnarray}
    \label{upper bound of difference between L and minimum}
    &\!\!\!\!\!\!\!\!& L(\mU_{11}(t),u_{-1}(t)) - \min_{\mU_{11}\in\R^{d\times d}, \atop u_{-1}\in\R} L(\mU_{11},u_{-1})\nonumber\\
    &\!\!\!\!=\!\!\!\!& \frac{1}{2}\| \wt\mLambda^{\frac{1}{2}}( u_{-1}\mLambda^{\frac{1}{2}}\mU_{11}\mLambda^{\frac{1}{2}} - D_1\mLambda\wt\mLambda^{-1} )  \|_F^2\nonumber\\
    &\!\!\!\!\leq\!\!\!\!& \frac{1}{2}\| \wt\mLambda\mLambda^{\frac{1}{2}}( u_{-1}\mLambda^{\frac{1}{2}}\mU_{11}\mLambda^{\frac{1}{2}} - D_1\mLambda\wt\mLambda^{-1} ) \mLambda^{\frac{1}{2}} \|_F^2\|\wt\mLambda^{-\frac{1}{2}}\mLambda^{-\frac{1}{2}}\|_F^2\|\mLambda^{-\frac{1}{2}}\|_F^2\nonumber\\
    &\!\!\!\!=\!\!\!\!& \frac{1}{2}\| \wt\mLambda\mLambda^{\frac{1}{2}}( u_{-1}\mLambda^{\frac{1}{2}}\mU_{11}\mLambda^{\frac{1}{2}} - D_1\mLambda\wt\mLambda^{-1} ) \mLambda^{\frac{1}{2}} \|_F^2\trace(\wt\mLambda^{-1}\mLambda^{-1} )\trace(\mLambda^{-1}),
\end{eqnarray}
where the first inequality uses that $\wt\mLambda$ and $\mLambda$ commute.

Combing \eqref{lower bound of gradient} and \eqref{upper bound of difference between L and minimum}, we can get Polyak--\L{}ojasiewicz (PL) inequality as follows:
\begin{eqnarray}
    \label{PL inequality appendix}
    &\!\!\!\!\!\!\!\!&\|\nabla L(\mU_{11}(t),u_{-1}(t))  \|_2^2\nonumber\\
    &\!\!\!\!\geq\!\!\!\!& \frac{\sigma^2\|\mLambda\mTheta \|_F^2(2D_1 - \sqrt{d}\sigma^2\|\wt\mLambda\|)}{D_1\sqrt{d}\|\mLambda\|^2\trace(\wt\mLambda^{-1}\mLambda^{-1} )\trace(\mLambda^{-1})} \big(L(\mU_{11}(t),u_{-1}(t)) - \min_{\mU_{11}\in\R^{d\times d}, \atop u_{-1}\in\R} L(\mU_{11},u_{-1})\big)\nonumber\\
    &\!\!\!\!:=\!\!\!\!& \alpha \big(L(\mU_{11}(t),u_{-1}(t)) - \min_{\mU_{11}\in\R^{d\times d}, \atop u_{-1}\in\R} L(\mU_{11},u_{-1})\big).
\end{eqnarray}

From the dynamics of gradient flow and the PL condition, we have
\begin{eqnarray}
    \label{dynamics of difference}
    &\!\!\!\!\!\!\!\!&\frac{\text{d}}{\text{d}t}\big(L(\mU_{11}(t),u_{-1}(t)) - \min_{\mU_{11}\in\R^{d\times d}, \atop u_{-1}\in\R} L(\mU_{11},u_{-1})\big)\nonumber\\
    &\!\!\!\! = \!\!\!\!& \bigg\<\frac{\text{d}\mU_{11}(t)}{\text{d}t}, \frac{\partial L(\mU_{11}(t),u_{-1}(t))}{\partial \mU_{11}(t)}  \bigg\> + \bigg\<\frac{\text{d}u_{-1}(t)}{\text{d}t}, \frac{\partial L(\mU_{11}(t),u_{-1}(t))}{\partial u_{-1}(t)}  \bigg\>\nonumber\\
    &\!\!\!\! = \!\!\!\!& - \bigg\|\frac{\partial L(\mU_{11}(t),u_{-1}(t))}{\partial \mU_{11}(t)}\bigg\|_F^2 - \bigg|\frac{\partial L(\mU_{11}(t),u_{-1}(t))}{\partial u_{-1}(t)}\bigg|^2\nonumber\\
    &\!\!\!\! \leq \!\!\!\!& -\alpha\big(L(\mU_{11}(t),u_{-1}(t)) - \min_{\mU_{11}\in\R^{d\times d}, \atop u_{-1}\in\R} L(\mU_{11},u_{-1})\big).
\end{eqnarray}

When $t\to \infty$, we have
\begin{eqnarray}
    \label{solution of ODE function}
    0&\!\!\!\! \leq \!\!\!\!& L(\mU_{11}(t),u_{-1}(t)) - \min_{\mU_{11}\in\R^{d\times d}, \atop u_{-1}\in\R} L(\mU_{11},u_{-1})\nonumber\\
    &\!\!\!\! \leq \!\!\!\!& e^{-\alpha t}\big(L(\mU_{11}(0),u_{-1}(0)) - \min_{\mU_{11}\in\R^{d\times d}, \atop u_{-1}\in\R} L(\mU_{11},u_{-1}) \big)\to 0,
\end{eqnarray}
which implies $\lim_{t\to \infty} L(\mU_{11}(t),u_{-1}(t)) - \min_{\mU_{11}\in\R^{d\times d}, \atop u_{-1}\in\R} L(\mU_{11},u_{-1}) = 0$.

According to {Corollary} \ref{Corollary: minimum of Different loss function}, \Cref{Lemma: relationshiop u1 U11} and \Cref{Lemma: nonnegativitity of u1}, we have $u_{-1}(t)\mU_{11}(t) = D_1 \wt\mLambda^{-1}$, $u_{-1}(t) = \|\mU_{11}(t)\|_F$ and further obtain
\begin{eqnarray}
    \label{global minimum of U11 and u1}
    \lim_{t\to \infty}\mU_{11}(t) = \sqrt{D_1\|\wt\mLambda^{-1} \|_F^{-1}}\wt\mLambda^{-1}, \ \ \lim_{t\to \infty}u_{-1}(t) = \sqrt{D_1\|\wt\mLambda^{-1} \|_F}.
\end{eqnarray}
This completes the proof.

\end{proof}

\section{Proof of \Cref{theorem recovery error using global minimum simplified}}
\label{proof of Theorem of recovery error global minimum}

Similar with \Cref{Theorem of global minimum}, the following theorem is also established under a more general setting with $\lambda, \gamma > 0$.
\begin{theorem} (Training error)
\label{theorem recovery error using global minimum}
Assuming the conditions in \Cref{Theorem of global minimum} hold, the recovery error between \eqref{output of y global minimum} and \eqref{definition of y} is
\begin{eqnarray}
    \label{recovery error final global minimum}
    \E[(\wh y_{n+1} - y_{n+1})^2] &\!\!\!\! = \!\!\!\!& D_1^2\trace\big(D_2(\mLambda\trace(\wt\mLambda^{-1}\mLambda\wt\mLambda^{-1}\mLambda ) + 2\mLambda \wt\mLambda^{-1}\mLambda \wt\mLambda^{-1}\mLambda)\nonumber\\
    &\!\!\!\! \!\!\!\!& + D_3 \mLambda \wt\mLambda^{-1}\mLambda \wt\mLambda^{-1}\mLambda  \big)+ D_4\trace(\mLambda) - 2D_1^2\trace(\mLambda \wt\mLambda^{-1}\mLambda ),
\end{eqnarray}
where $D_4 = \begin{cases} \sigma_w^2 + (n+1)\sigma_e^2  , & \gamma = 1 \\
     \gamma^{2n+2}\sigma_w^2 + \frac{1- \gamma^{2n+2}}{1-\gamma^2 }\sigma_e^2  , &  \gamma \neq 1   \end{cases}$.
\end{theorem}

\begin{proof} To establish the result, we expand $\E[(\wh y_{n+1} - y_{n+1})^2]$ as
\begin{eqnarray}
    \label{recovery error final recovery error global minimum}
    &\!\!\!\!  \!\!\!\!& \E[(\wh y_{n+1} - y_{n+1})^2]\nonumber\\
     &\!\!\!\! = \!\!\!\!& \E\bigg[\bigg(D_1 \bigg( \sum_{i=1}^n\lambda^{n+1-i} \vw_i^\top\vx_i\vx_i^\top \bigg)\wt\mLambda^{-1} \vx_{n+1} - \vw_{n+1}^\top \vx_{n+1}\bigg)^2\bigg]\nonumber\\
    &\!\!\!\! = \!\!\!\!& D_1^2\E\bigg[ \bigg( \sum_{i=1}^n\lambda^{n+1-i} \vw_i^\top\vx_i\vx_i^\top \bigg)\wt\mLambda^{-1} \vx_{n+1}\vx_{n+1}^\top \wt\mLambda^{-1} \bigg( \sum_{j=1}^n\lambda^{n+1-j} \vx_j\vx_j^\top\vw_j \bigg) \bigg]\nonumber\\
    &\!\!\!\!  \!\!\!\!& \!+\! \E[ \vw_{n+1}^\top \vx_{n+1}\vx_{n+1}^\top\vw_{n+1} ] \!-\! 2D_1\E\bigg[ \bigg( \sum_{i=1}^n\lambda^{n+1-i} \vw_i^\top\vx_i\vx_i^\top \bigg)\wt\mLambda^{-1} \vx_{n+1} \vx_{n+1}^\top\vw_{n+1} \bigg].
\end{eqnarray}

For the first term of the last equation in \eqref{recovery error final recovery error global minimum}, we have
\begin{eqnarray}
    \label{first term recovery error final recovery error global minimum}
    &\!\!\!\!  \!\!\!\!& \E\bigg[ \bigg( \sum_{i=1}^n\lambda^{n+1-i} \vw_i^\top\vx_i\vx_i^\top \bigg)\wt\mLambda^{-1} \vx_{n+1}\vx_{n+1}^\top \wt\mLambda^{-1} \bigg( \sum_{j=1}^n\lambda^{n+1-j} \vx_j\vx_j^\top\vw_j \bigg) \bigg]\nonumber\\
     &\!\!\!\! = \!\!\!\!& \E\bigg[\trace\bigg( \sum_{i=1}^n \sum_{j=1}^n \lambda^{2n+2-i-j} \vx_i\vx_i^\top \wt\mLambda^{-1}\mLambda  \wt\mLambda^{-1} \vx_j\vx_j^\top\vw_j \vw_i^\top \bigg)\bigg]\nonumber\\
     &\!\!\!\! = \!\!\!\!&D_2\trace(\mLambda\trace(\wt\mLambda^{-1}\mLambda\wt\mLambda^{-1}\mLambda ) + 2\mLambda \wt\mLambda^{-1}\mLambda \wt\mLambda^{-1}\mLambda) + D_3 \trace(\mLambda \wt\mLambda^{-1}\mLambda \wt\mLambda^{-1}\mLambda),
\end{eqnarray}
where the last line uses \eqref{statisctis of fourth-order Gaussian vector} and $\E\big[ \vx_b\vx_b^\top \mA\vx_a \vx_a^\top \big] = \begin{cases}
    \mLambda\mA\mLambda, & a \neq b \\
    2\mLambda\mA\mLambda + \trace(\mA\mLambda)\mLambda, &  a = b
\end{cases}$ in \citep[Lemma A.2]{sayed2011adaptive}.

For the second term of the last equation in \eqref{recovery error final recovery error global minimum}, we can derive
\begin{eqnarray}
    \label{second term recovery error final recovery error global minimum}
    \E[ \vw_{n+1}^\top \vx_{n+1}\vx_{n+1}^\top\vw_{n+1} ] &\!\!\!\! = \!\!\!\!&\E[ \trace(\mLambda \vw_{n+1}\vw_{n+1}^\top) ]\nonumber\\
     &\!\!\!\! = \!\!\!\!&D_4\trace(\mLambda),
\end{eqnarray}
where we define $D_4 = \begin{cases} \sigma_w^2 + (n+1)\sigma_e^2  , & \gamma = 1 \\
     \gamma^{2n+2}\sigma_w^2 + \frac{1- \gamma^{2n+2}}{1-\gamma^2 }\sigma_e^2  , &  \gamma \neq 1   \end{cases}$ which follows the result in \Cref{covariance matrix of gaussian vector}.

For the third term  of the last equation in \eqref{recovery error final recovery error global minimum}, we get
\begin{eqnarray}
    \label{third term recovery error final recovery error global minimum}
    \E\bigg[ \bigg( \sum_{i=1}^n\lambda^{n+1-i} \vw_i^\top\vx_i\vx_i^\top \bigg)\wt\mLambda^{-1} \vx_{n+1} \vx_{n+1}^\top\vw_{n+1} \bigg] &\!\!\!\! = \!\!\!\!&\E\bigg[ \trace\bigg( \sum_{i=1}^n\lambda^{n+1-i}\mLambda\wt\mLambda^{-1} \mLambda \vw_{n+1}\vw_i^\top \bigg)\bigg]\nonumber\\
     &\!\!\!\! = \!\!\!\!&D_1\trace(\mLambda \wt\mLambda^{-1}\mLambda ),
\end{eqnarray}
where the last line uses \eqref{Definition of D1}. This completes the proof.

\end{proof}

\section{Proof of \Cref{theorem recovery error using global minimum test}}
\label{proof of Theorem of recovery error global minimum test}

\begin{proof} We first expand $\E[(\wt y_{m+1} - \ol y_{m+1})^2]$ as
\begin{eqnarray}
    \label{recovery error final recovery error global minimum test}
    &\!\!\!\!  \!\!\!\!& \E[(\wt y_{m+1} - \ol y_{m+1})^2]\nonumber\\
     &\!\!\!\! = \!\!\!\!& \E\bigg[\bigg(D_1 \bigg( \sum_{i=1}^m\ol \lambda^{m+1-i} \ol \vw_i^\top\ol \vx_i\ol \vx_i^\top \bigg)\wt\mLambda^{-1} \ol \vx_{m+1} - \ol \vw_{m+1}^\top \ol \vx_{m+1}\bigg)^2\bigg]\nonumber\\
    &\!\!\!\! = \!\!\!\!& D_1^2\E\bigg[ \bigg( \sum_{i=1}^m\ol \lambda^{m+1-i} \ol \vw_i^\top\ol \vx_i\ol \vx_i^\top \bigg)\wt\mLambda^{-1} \ol \vx_{m+1}\ol \vx_{m+1}^\top \wt\mLambda^{-1} \bigg( \sum_{j=1}^m\ol \lambda^{m+1-j} \ol \vx_j\ol \vx_j^\top\ol \vw_j \bigg) \bigg]\nonumber\\
    &\!\!\!\!  \!\!\!\!& + \E[ \ol \vw_{m+1}^\top \ol \vx_{m+1}\ol \vx_{m+1}^\top\ol \vw_{m+1} ] - 2D_1\E\bigg[ \bigg( \sum_{i=1}^m\ol \lambda^{m+1-i} \ol \vw_i^\top\ol \vx_i\ol \vx_i^\top \bigg)\wt\mLambda^{-1} \ol \vx_{m+1} \ol \vx_{m+1}^\top\ol \vw_{m+1} \bigg].\nonumber\\
\end{eqnarray}

For the first term of the last equation in \eqref{recovery error final recovery error global minimum test}, we have
\begin{eqnarray}
    \label{first term recovery error final recovery error global minimum test}
    &\!\!\!\!  \!\!\!\!& \E\bigg[ \bigg( \sum_{i=1}^m\ol \lambda^{m+1-i} \ol \vw_i^\top\ol \vx_i\ol \vx_i^\top \bigg)\wt\mLambda^{-1} \ol \vx_{m+1}\ol \vx_{m+1}^\top \wt\mLambda^{-1} \bigg( \sum_{j=1}^m\ol \lambda^{m+1-j} \ol \vx_j\ol \vx_j^\top\ol \vw_j \bigg) \bigg]\nonumber\\
     &\!\!\!\! = \!\!\!\!& \E\bigg[\trace\bigg( \sum_{i=1}^m \sum_{j=1}^m \ol\lambda^{2m+2-i-j} \ol\vx_i\ol\vx_i^\top \wt\mLambda^{-1}\ol\mLambda  \wt\mLambda^{-1} \ol\vx_j\ol\vx_j^\top\ol\vw_j \ol\vw_i^\top \bigg)\bigg]\nonumber\\
     &\!\!\!\! = \!\!\!\!&\ol D_2\trace(\ol\mLambda\trace(\wt\mLambda^{-1}\ol\mLambda\wt\mLambda^{-1}\ol\mLambda ) + 2\ol\mLambda \wt\mLambda^{-1}\ol\mLambda \wt\mLambda^{-1}\ol\mLambda) + \ol D_3 \trace(\ol\mLambda \wt\mLambda^{-1}\ol\mLambda \wt\mLambda^{-1}\ol\mLambda),
\end{eqnarray}
where the last line uses \eqref{statisctis of fourth-order Gaussian vector} and $\E\big[ \ol\vx_b\ol\vx_b^\top \mA\ol\vx_a \ol\vx_a^\top \big] = \begin{cases}
    \ol\mLambda\mA\ol\mLambda, & a \neq b \\
    2\ol\mLambda\mA\ol\mLambda + \trace(\mA\ol\mLambda)\ol\mLambda, &  a = b
\end{cases}$ in \citep[Lemma A.2]{sayed2011adaptive}. By replacing $\lambda$, $\gamma$, $\sigma_w^2$, $\sigma_e^2$ and $n$ in $D_2$ and $D_3$ with $\ol\lambda$, $\ol\gamma$, $\ol\sigma_w^2$, $\ol\sigma_e^2$ and $m$, respectively, we obtain the modified quantities $\ol D_2$ and $\ol D_3$.

For the second term of the last equation in \eqref{recovery error final recovery error global minimum test}, we can derive
\begin{eqnarray}
    \label{second term recovery error final recovery error global minimum test}
    \E[ \ol\vw_{n+1}^\top \ol\vx_{n+1}\ol\vx_{n+1}^\top\ol\vw_{n+1} ] &\!\!\!\! = \!\!\!\!&\E[ \trace(\ol\mLambda \ol\vw_{n+1}\ol\vw_{n+1}^\top) ]\nonumber\\
     &\!\!\!\! = \!\!\!\!&\ol D_4\trace(\ol\mLambda)
\end{eqnarray}
where we define $\ol D_4 = \begin{cases} \ol\sigma_w^2 + (m+1)\ol\sigma_e^2  , & \ol\gamma = 1 \\
     \ol\gamma^{2m+2}\ol\sigma_w^2 + \frac{1- \ol\gamma^{2m+2}}{1-\ol\gamma^2 }\ol\sigma_e^2  , &  \ol\gamma \neq 1   \end{cases}$ which follows the result in \Cref{covariance matrix of gaussian vector}.

For the third term  of the last equation in \eqref{recovery error final recovery error global minimum test}, we get
\begin{eqnarray}
    \label{third term recovery error final recovery error global minimum test}
    \E\bigg[ \bigg( \sum_{i=1}^m\ol \lambda^{m+1-i} \ol \vw_i^\top\ol \vx_i\ol \vx_i^\top \bigg)\wt\mLambda^{-1} \ol \vx_{m+1} \ol \vx_{m+1}^\top\ol \vw_{m+1} \bigg] &\!\!\!\! = \!\!\!\!&\E\bigg[ \trace\bigg( \sum_{i=1}^m\ol \lambda^{m+1-i}\ol\mLambda\wt\mLambda^{-1} \ol\mLambda \ol\vw_{m+1}\ol\vw_i^\top \bigg)\bigg]\nonumber\\
     &\!\!\!\! = \!\!\!\!&\ol D_1\trace(\ol\mLambda \wt\mLambda^{-1}\ol\mLambda ),
\end{eqnarray}
where the last line uses \eqref{Definition of D1}. By replacing $\lambda$, $\gamma$, $\sigma_w^2$, $\sigma_e^2$ and $n$ in $D_1$ with $\ol\lambda$, $\ol\gamma$, $\ol\sigma_w^2$, $\ol\sigma_e^2$ and $m$, we obtain the modified quantity $\ol D_1$. This completes the proof.

\end{proof}


\begin{thebibliography}{57}
\providecommand{\natexlab}[1]{#1}
\providecommand{\url}[1]{\texttt{#1}}
\expandafter\ifx\csname urlstyle\endcsname\relax
  \providecommand{\doi}[1]{doi: #1}\else
  \providecommand{\doi}{doi: \begingroup \urlstyle{rm}\Url}\fi

\bibitem[Abdolee et~al.(2016)Abdolee, Vakilian, and
  Champagne]{abdolee2016tracking}
Reza Abdolee, Vida Vakilian, and Benoit Champagne.
\newblock Tracking performance and optimal adaptation step-sizes of
  diffusion-lms networks.
\newblock \emph{IEEE Transactions on Control of Network Systems}, 5\penalty0
  (1):\penalty0 67--78, 2016.

\bibitem[Achiam et~al.(2023)Achiam, Adler, Agarwal, Ahmad, Akkaya, Aleman,
  Almeida, Altenschmidt, Altman, Anadkat, et~al.]{achiam2023gpt}
Josh Achiam, Steven Adler, Sandhini Agarwal, Lama Ahmad, Ilge Akkaya,
  Florencia~Leoni Aleman, Diogo Almeida, Janko Altenschmidt, Sam Altman,
  Shyamal Anadkat, et~al.
\newblock Gpt-4 technical report.
\newblock \emph{arXiv preprint arXiv:2303.08774}, 2023.

\bibitem[Ahn et~al.(2023)Ahn, Cheng, Daneshmand, and Sra]{ahn2023transformers}
Kwangjun Ahn, Xiang Cheng, Hadi Daneshmand, and Suvrit Sra.
\newblock Transformers learn to implement preconditioned gradient descent for
  in-context learning.
\newblock \emph{Advances in Neural Information Processing Systems},
  36:\penalty0 45614--45650, 2023.

\bibitem[Aky{\"u}rek et~al.(2022)Aky{\"u}rek, Schuurmans, Andreas, Ma, and
  Zhou]{akyurek2022learning}
Ekin Aky{\"u}rek, Dale Schuurmans, Jacob Andreas, Tengyu Ma, and Denny Zhou.
\newblock What learning algorithm is in-context learning? investigations with
  linear models.
\newblock \emph{arXiv preprint arXiv:2211.15661}, 2022.

\bibitem[Aky{\"u}rek et~al.(2024)Aky{\"u}rek, Wang, Kim, and
  Andreas]{akyurek2024context}
Ekin Aky{\"u}rek, Bailin Wang, Yoon Kim, and Jacob Andreas.
\newblock In-context language learning: Architectures and algorithms.
\newblock \emph{arXiv preprint arXiv:2401.12973}, 2024.

\bibitem[Bai et~al.(2023)Bai, Chen, Wang, Xiong, and Mei]{bai2023transformers}
Yu~Bai, Fan Chen, Huan Wang, Caiming Xiong, and Song Mei.
\newblock Transformers as statisticians: Provable in-context learning with
  in-context algorithm selection.
\newblock \emph{Advances in neural information processing systems},
  36:\penalty0 57125--57211, 2023.

\bibitem[Bertsch et~al.(2024)Bertsch, Ivgi, Xiao, Alon, Berant, Gormley, and
  Neubig]{bertsch2024context}
Amanda Bertsch, Maor Ivgi, Emily Xiao, Uri Alon, Jonathan Berant, Matthew~R
  Gormley, and Graham Neubig.
\newblock In-context learning with long-context models: An in-depth
  exploration.
\newblock \emph{arXiv preprint arXiv:2405.00200}, 2024.

\bibitem[Brown et~al.(2020)Brown, Mann, Ryder, Subbiah, Kaplan, Dhariwal,
  Neelakantan, Shyam, Sastry, Askell, et~al.]{brown2020language}
Tom Brown, Benjamin Mann, Nick Ryder, Melanie Subbiah, Jared~D Kaplan, Prafulla
  Dhariwal, Arvind Neelakantan, Pranav Shyam, Girish Sastry, Amanda Askell,
  et~al.
\newblock Language models are few-shot learners.
\newblock \emph{Advances in neural information processing systems},
  33:\penalty0 1877--1901, 2020.

\bibitem[Chen et~al.(2021)Chen, Lu, Rajeswaran, Lee, Grover, Laskin, Abbeel,
  Srinivas, and Mordatch]{chen2021decision}
Lili Chen, Kevin Lu, Aravind Rajeswaran, Kimin Lee, Aditya Grover, Misha
  Laskin, Pieter Abbeel, Aravind Srinivas, and Igor Mordatch.
\newblock Decision transformer: Reinforcement learning via sequence modeling.
\newblock \emph{Advances in neural information processing systems},
  34:\penalty0 15084--15097, 2021.

\bibitem[Chen et~al.(2019)Chen, Zhao, Li, Huang, and Ou]{chen2019behavior}
Qiwei Chen, Huan Zhao, Wei Li, Pipei Huang, and Wenwu Ou.
\newblock Behavior sequence transformer for e-commerce recommendation in
  alibaba.
\newblock In \emph{Proceedings of the 1st international workshop on deep
  learning practice for high-dimensional sparse data}, pp.\  1--4, 2019.

\bibitem[Chen et~al.(2024)Chen, Sheen, Wang, and Yang]{chen2024training}
Siyu Chen, Heejune Sheen, Tianhao Wang, and Zhuoran Yang.
\newblock Training dynamics of multi-head softmax attention for in-context
  learning: Emergence, convergence, and optimality.
\newblock \emph{arXiv preprint arXiv:2402.19442}, 2024.

\bibitem[Claser \& Nascimento(2021)Claser and Nascimento]{claser2021tracking}
Raffaello Claser and Vitor~H Nascimento.
\newblock On the tracking performance of adaptive filters and their
  combinations.
\newblock \emph{IEEE Transactions on Signal Processing}, 69:\penalty0
  3104--3116, 2021.

\bibitem[Das et~al.(2015)Das, Ruiz, Chakraborty, and
  Arenas-Garc{\'\i}a]{das2015steady}
Bijit~Kumar Das, Luis A~Azpicueta Ruiz, Mrityunjoy Chakraborty, and
  Jer{\'o}nimo Arenas-Garc{\'\i}a.
\newblock On steady state tracking performance of adaptive networks.
\newblock In \emph{2015 IEEE International Conference on Digital Signal
  Processing (DSP)}, pp.\  843--847. IEEE, 2015.

\bibitem[Ding et~al.()Ding, Levinboim, Wu, Goodman, and Soricut]{dingcausallm}
Nan Ding, Tomer Levinboim, Jialin Wu, Sebastian Goodman, and Radu Soricut.
\newblock Causallm is not optimal for in-context learning.
\newblock In \emph{The Twelfth International Conference on Learning
  Representations}.

\bibitem[Dong et~al.(2022)Dong, Li, Dai, Zheng, Ma, Li, Xia, Xu, Wu, Liu,
  et~al.]{dong2022survey}
Qingxiu Dong, Lei Li, Damai Dai, Ce~Zheng, Jingyuan Ma, Rui Li, Heming Xia,
  Jingjing Xu, Zhiyong Wu, Tianyu Liu, et~al.
\newblock A survey on in-context learning.
\newblock \emph{arXiv preprint arXiv:2301.00234}, 2022.

\bibitem[Dosovitskiy et~al.(2020)Dosovitskiy, Beyer, Kolesnikov, Weissenborn,
  Zhai, Unterthiner, Dehghani, Minderer, Heigold, Gelly,
  et~al.]{dosovitskiy2020image}
Alexey Dosovitskiy, Lucas Beyer, Alexander Kolesnikov, Dirk Weissenborn,
  Xiaohua Zhai, Thomas Unterthiner, Mostafa Dehghani, Matthias Minderer,
  G~Heigold, S~Gelly, et~al.
\newblock An image is worth 16x16 words: Transformers for image recognition at
  scale.
\newblock In \emph{International Conference on Learning Representations}, 2020.

\bibitem[Fu et~al.(2024)Fu, Chen, Jia, and Sharan]{fu2024transformers}
Deqing Fu, Tian-qi Chen, Robin Jia, and Vatsal Sharan.
\newblock Transformers learn to achieve second-order convergence rates for
  in-context linear regression.
\newblock \emph{Advances in Neural Information Processing Systems},
  37:\penalty0 98675--98716, 2024.

\bibitem[Garg et~al.(2022)Garg, Tsipras, Liang, and Valiant]{garg2022can}
Shivam Garg, Dimitris Tsipras, Percy~S Liang, and Gregory Valiant.
\newblock What can transformers learn in-context? a case study of simple
  function classes.
\newblock \emph{Advances in Neural Information Processing Systems},
  35:\penalty0 30583--30598, 2022.

\bibitem[Gu \& Dao(2023)Gu and Dao]{gu2023mamba}
Albert Gu and Tri Dao.
\newblock Mamba: Linear-time sequence modeling with selective state spaces.
\newblock \emph{arXiv preprint arXiv:2312.00752}, 2023.

\bibitem[Huang et~al.(2023)Huang, Cheng, and Liang]{huang2023context}
Yu~Huang, Yuan Cheng, and Yingbin Liang.
\newblock In-context convergence of transformers.
\newblock \emph{arXiv preprint arXiv:2310.05249}, 2023.

\bibitem[Huang et~al.(2024)Huang, Cheng, and Liang]{huang2024context}
Yu~Huang, Yuan Cheng, and Yingbin Liang.
\newblock In-context convergence of transformers.
\newblock In \emph{Proceedings of the 41st International Conference on Machine
  Learning}, pp.\  19660--19722, 2024.

\bibitem[Janner et~al.(2021)Janner, Li, and Levine]{janner2021offline}
Michael Janner, Qiyang Li, and Sergey Levine.
\newblock Offline reinforcement learning as one big sequence modeling problem.
\newblock \emph{Advances in neural information processing systems},
  34:\penalty0 1273--1286, 2021.

\bibitem[Jiang et~al.(2025{\natexlab{a}})Jiang, Dong, Zhou, and
  Zhu]{jiang2025compression}
Jiachen Jiang, Yuxin Dong, Jinxin Zhou, and Zhihui Zhu.
\newblock From compression to expansion: A layerwise analysis of in-context
  learning.
\newblock \emph{arXiv preprint arXiv:2505.17322}, 2025{\natexlab{a}}.

\bibitem[Jiang et~al.(2025{\natexlab{b}})Jiang, Qin, and Zhu]{jiang2025context}
Jiachen Jiang, Zhen Qin, and Zhihui Zhu.
\newblock In-context learning for non-stationary mimo equalization.
\newblock \emph{arXiv preprint arXiv:2510.08711}, 2025{\natexlab{b}}.

\bibitem[Katsch(2023)]{katsch2023gateloop}
Tobias Katsch.
\newblock Gateloop: Fully data-controlled linear recurrence for sequence
  modeling.
\newblock \emph{arXiv preprint arXiv:2311.01927}, 2023.

\bibitem[Kim et~al.(2023)Kim, Lee, Ju, Ngo, Moon, and Shim]{kim2023transformer}
Seungnyun Kim, Anho Lee, Hyungyu Ju, Khoa~Anh Ngo, Jihoon Moon, and Byonghyo
  Shim.
\newblock Transformer-based channel parameter acquisition for terahertz
  ultra-massive mimo systems.
\newblock \emph{IEEE Transactions on Vehicular Technology}, 72\penalty0
  (11):\penalty0 15127--15132, 2023.

\bibitem[Li et~al.(2024{\natexlab{a}})Li, Zhang, Do, Yue, and Chen]{li2024long}
Tianle Li, Ge~Zhang, Quy~Duc Do, Xiang Yue, and Wenhu Chen.
\newblock Long-context llms struggle with long in-context learning.
\newblock \emph{arXiv preprint arXiv:2404.02060}, 2024{\natexlab{a}}.

\bibitem[Li et~al.(2024{\natexlab{b}})Li, Rawat, and Oymak]{li2024fine}
Yingcong Li, Ankit~S Rawat, and Samet Oymak.
\newblock Fine-grained analysis of in-context linear estimation: Data,
  architecture, and beyond.
\newblock \emph{Advances in Neural Information Processing Systems},
  37:\penalty0 138324--138364, 2024{\natexlab{b}}.

\bibitem[Li et~al.(2025)Li, Tarzanagh, Rawat, Fazel, and Oymak]{li2025gating}
Yingcong Li, Davoud~Ataee Tarzanagh, Ankit~Singh Rawat, Maryam Fazel, and Samet
  Oymak.
\newblock Gating is weighting: Understanding gated linear attention through
  in-context learning.
\newblock \emph{arXiv preprint arXiv:2504.04308}, 2025.

\bibitem[Ma et~al.(2025)Ma, Sun, Dong, Chen, and Rabitz]{ma2025tomography}
Hailan Ma, Zhenhong Sun, Daoyi Dong, Chunlin Chen, and Herschel Rabitz.
\newblock Tomography of quantum states from structured measurements via
  quantum-aware transformer.
\newblock \emph{IEEE Transactions on Cybernetics}, 2025.

\bibitem[Mahankali et~al.(2023)Mahankali, Hashimoto, and Ma]{mahankali2023one}
Arvind Mahankali, Tatsunori~B Hashimoto, and Tengyu Ma.
\newblock One step of gradient descent is provably the optimal in-context
  learner with one layer of linear self-attention.
\newblock \emph{arXiv preprint arXiv:2307.03576}, 2023.

\bibitem[Mahankali et~al.()Mahankali, Hashimoto, and Ma]{mahankalione}
Arvind~V Mahankali, Tatsunori Hashimoto, and Tengyu Ma.
\newblock One step of gradient descent is provably the optimal in-context
  learner with one layer of linear self-attention.
\newblock In \emph{The Twelfth International Conference on Learning
  Representations}.

\bibitem[Min et~al.(2021)Min, Lewis, Zettlemoyer, and
  Hajishirzi]{min2021metaicl}
Sewon Min, Mike Lewis, Luke Zettlemoyer, and Hannaneh Hajishirzi.
\newblock Metaicl: Learning to learn in context.
\newblock \emph{arXiv preprint arXiv:2110.15943}, 2021.

\bibitem[Peng et~al.(2024)Peng, Goldstein, Anthony, Albalak, Alcaide, Biderman,
  Cheah, Du, Ferdinan, Hou, et~al.]{peng2024eagle}
Bo~Peng, Daniel Goldstein, Quentin Anthony, Alon Albalak, Eric Alcaide, Stella
  Biderman, Eugene Cheah, Xingjian Du, Teddy Ferdinan, Haowen Hou, et~al.
\newblock Eagle and finch: Rwkv with matrix-valued states and dynamic
  recurrence.
\newblock \emph{arXiv preprint arXiv:2404.05892}, 2024.

\bibitem[Qin et~al.(2020)Qin, Tao, and Xia]{qin2020proportionate}
Zhen Qin, Jun Tao, and Yili Xia.
\newblock A proportionate recursive least squares algorithm and its performance
  analysis.
\newblock \emph{IEEE Transactions on Circuits and Systems II: Express Briefs},
  68\penalty0 (1):\penalty0 506--510, 2020.

\bibitem[Qin et~al.(2025)Qin, Zhou, and Zhu]{qin2025convergence}
Zhen Qin, Jinxin Zhou, and Zhihui Zhu.
\newblock On the convergence of gradient descent on learning transformers with
  residual connections.
\newblock \emph{arXiv preprint arXiv:2506.05249}, 2025.

\bibitem[Radford et~al.(2019)Radford, Wu, Child, Luan, Amodei, Sutskever,
  et~al.]{radford2019language}
Alec Radford, Jeffrey Wu, Rewon Child, David Luan, Dario Amodei, Ilya
  Sutskever, et~al.
\newblock Language models are unsupervised multitask learners.
\newblock \emph{OpenAI blog}, 1\penalty0 (8):\penalty0 9, 2019.

\bibitem[Sayed(2011)]{sayed2011adaptive}
Ali~H Sayed.
\newblock \emph{Adaptive filters}.
\newblock John Wiley \& Sons, 2011.

\bibitem[Socher et~al.(2013)Socher, Perelygin, Wu, Chuang, Manning, Ng, and
  Potts]{socher2013recursive}
Richard Socher, Alex Perelygin, Jean Wu, Jason Chuang, Christopher~D Manning,
  Andrew~Y Ng, and Christopher Potts.
\newblock Recursive deep models for semantic compositionality over a sentiment
  treebank.
\newblock In \emph{Proceedings of the 2013 conference on empirical methods in
  natural language processing}, pp.\  1631--1642, 2013.

\bibitem[Song et~al.(2024)Song, Han, Zhang, Ding, and Hong]{song2024unraveling}
Bingqing Song, Boran Han, Shuai Zhang, Jie Ding, and Mingyi Hong.
\newblock Unraveling the gradient descent dynamics of transformers.
\newblock \emph{Advances in Neural Information Processing Systems},
  37:\penalty0 92317--92351, 2024.

\bibitem[Sun et~al.(2023)Sun, Dong, Huang, Ma, Xia, Xue, Wang, and
  Wei]{sun2023retentive}
Yutao Sun, Li~Dong, Shaohan Huang, Shuming Ma, Yuqing Xia, Jilong Xue, Jianyong
  Wang, and Furu Wei.
\newblock Retentive network: A successor to transformer for large language
  models.
\newblock \emph{arXiv preprint arXiv:2307.08621}, 2023.

\bibitem[Tsai et~al.(2019)Tsai, Bai, Liang, Kolter, Morency, and
  Salakhutdinov]{tsai2019multimodal}
Yao-Hung~Hubert Tsai, Shaojie Bai, Paul~Pu Liang, J~Zico Kolter, Louis-Philippe
  Morency, and Ruslan Salakhutdinov.
\newblock Multimodal transformer for unaligned multimodal language sequences.
\newblock In \emph{Proceedings of the conference. Association for computational
  linguistics. Meeting}, volume 2019, pp.\  6558, 2019.

\bibitem[Vaswani et~al.(2017)Vaswani, Shazeer, Parmar, Uszkoreit, Jones, Gomez,
  Kaiser, and Polosukhin]{vaswani2017attention}
Ashish Vaswani, Noam Shazeer, Niki Parmar, Jakob Uszkoreit, Llion Jones,
  Aidan~N Gomez, {\L}ukasz Kaiser, and Illia Polosukhin.
\newblock Attention is all you need.
\newblock \emph{Advances in neural information processing systems}, 30, 2017.

\bibitem[Von~Oswald et~al.(2023)Von~Oswald, Niklasson, Randazzo, Sacramento,
  Mordvintsev, Zhmoginov, and Vladymyrov]{von2023transformers}
Johannes Von~Oswald, Eyvind Niklasson, Ettore Randazzo, Jo{\~a}o Sacramento,
  Alexander Mordvintsev, Andrey Zhmoginov, and Max Vladymyrov.
\newblock Transformers learn in-context by gradient descent.
\newblock In \emph{International Conference on Machine Learning}, pp.\
  35151--35174. PMLR, 2023.

\bibitem[Wang et~al.(2022)Wang, Qin, Tao, and Yang]{wang2022performance}
Yu~Wang, Zhen Qin, Jun Tao, and Le~Yang.
\newblock Performance analysis of prls-based time-varying sparse system
  identification.
\newblock In \emph{2022 IEEE 12th Sensor Array and Multichannel Signal
  Processing Workshop (SAM)}, pp.\  251--255. IEEE, 2022.

\bibitem[Wies et~al.(2023)Wies, Levine, and Shashua]{wies2023learnability}
Noam Wies, Yoav Levine, and Amnon Shashua.
\newblock The learnability of in-context learning.
\newblock \emph{Advances in Neural Information Processing Systems},
  36:\penalty0 36637--36651, 2023.

\bibitem[Williams et~al.(2018)Williams, Nangia, and Bowman]{williams2018broad}
Adina Williams, Nikita Nangia, and Samuel Bowman.
\newblock A broad-coverage challenge corpus for sentence understanding through
  inference.
\newblock In \emph{Proceedings of the 2018 conference of the North American
  chapter of the association for computational linguistics: human language
  technologies, volume 1 (long papers)}, pp.\  1112--1122, 2018.

\bibitem[Wu et~al.(2023)Wu, Liu, Chrysos, and Cevher]{wu2023convergence}
Yongtao Wu, Fanghui Liu, Grigorios Chrysos, and Volkan Cevher.
\newblock On the convergence of encoder-only shallow transformers.
\newblock \emph{Advances in Neural Information Processing Systems},
  36:\penalty0 52197--52237, 2023.

\bibitem[Yang et~al.(2023)Yang, Wang, Shen, Panda, and Kim]{yang2023gated}
Songlin Yang, Bailin Wang, Yikang Shen, Rameswar Panda, and Yoon Kim.
\newblock Gated linear attention transformers with hardware-efficient training.
\newblock \emph{arXiv preprint arXiv:2312.06635}, 2023.

\bibitem[Yang et~al.(2024)Yang, Huang, Liang, and Chi]{yang2024context}
Tong Yang, Yu~Huang, Yingbin Liang, and Yuejie Chi.
\newblock In-context learning with representations: Contextual generalization
  of trained transformers.
\newblock \emph{arXiv preprint arXiv:2408.10147}, 2024.

\bibitem[Yu et~al.(2021)Yu, de~Lamare, Yang, and Cai]{yu2021tracking}
Yi~Yu, Rodrigo~C de~Lamare, Tao Yang, and Qiangming Cai.
\newblock Tracking analyses of m-estimate based lms and nlms algorithms.
\newblock In \emph{2021 IEEE Statistical Signal Processing Workshop (SSP)},
  pp.\  1--5. IEEE, 2021.

\bibitem[Zhang et~al.(2024)Zhang, Frei, and Bartlett]{zhang2024trained}
Ruiqi Zhang, Spencer Frei, and Peter~L Bartlett.
\newblock Trained transformers learn linear models in-context.
\newblock \emph{Journal of Machine Learning Research}, 25\penalty0
  (49):\penalty0 1--55, 2024.

\bibitem[Zhang \& Wu(2012)Zhang and Wu]{zhang2012inference}
Ting Zhang and Wei~Biao Wu.
\newblock Inference of time-varying regression models.
\newblock 2012.

\bibitem[Zhang \& Wu(2015)Zhang and Wu]{zhang2015time}
Ting Zhang and Wei~Biao Wu.
\newblock Time-varying nonlinear regression models: nonparametric estimation
  and model selection.
\newblock 2015.

\bibitem[Zhang et~al.(2025)Zhang, Singh, Latham, and Saxe]{zhang2025training}
Yedi Zhang, Aaditya~K Singh, Peter~E Latham, and Andrew Saxe.
\newblock Training dynamics of in-context learning in linear attention.
\newblock \emph{arXiv preprint arXiv:2501.16265}, 2025.

\bibitem[Zhang et~al.(2023)Zhang, Zhou, and Liu]{zhang2023makes}
Yuanhan Zhang, Kaiyang Zhou, and Ziwei Liu.
\newblock What makes good examples for visual in-context learning?
\newblock \emph{Advances in Neural Information Processing Systems},
  36:\penalty0 17773--17794, 2023.

\bibitem[Zhou et~al.(2018)Zhou, Zhu, Song, Fan, Zhu, Ma, Yan, Jin, Li, and
  Gai]{zhou2018deep}
Guorui Zhou, Xiaoqiang Zhu, Chenru Song, Ying Fan, Han Zhu, Xiao Ma, Yanghui
  Yan, Junqi Jin, Han Li, and Kun Gai.
\newblock Deep interest network for click-through rate prediction.
\newblock In \emph{Proceedings of the 24th ACM SIGKDD international conference
  on knowledge discovery \& data mining}, pp.\  1059--1068, 2018.

\end{thebibliography}
\end{document}